\documentclass[11pt]{article}

\usepackage{amsmath,amsfonts,bm}

\def\eqref#1{equation~\ref{#1}}

\def\1{\bm{1}}

\DeclareMathAlphabet{\mathsfit}{\encodingdefault}{\sfdefault}{m}{sl}
\SetMathAlphabet{\mathsfit}{bold}{\encodingdefault}{\sfdefault}{bx}{n}

\usepackage[margin=1in]{geometry}
\usepackage{times,natbib}
\usepackage[utf8]{inputenc} 
\usepackage[T1]{fontenc}    
\PassOptionsToPackage{hyphens}{url}\usepackage{hyperref}       
\usepackage{xurl}
\usepackage{url}
\usepackage{graphicx}
\usepackage{amsmath,amssymb,amsfonts,amstext,amsthm,mathrsfs}
\usepackage{mathtools}
\usepackage{subcaption,wrapfig}
\usepackage{booktabs}       
\usepackage{nicefrac}       
\usepackage{microtype}      
\usepackage{enumerate}
\usepackage{cleveref}
\usepackage{dsfont}
\usepackage{enumitem}
\usepackage{thm-restate}
\usepackage{color}
\usepackage{algorithmic,algorithm}
\usepackage[algo2e]{algorithm2e}
\usepackage{bbm}
\usepackage{makecell}
\usepackage[normalem]{ulem}
\usepackage[textsize=tiny]{todonotes}
\usepackage{thm-restate}
\usepackage{authblk}

\usepackage[toc,page,header]{appendix} 
\usepackage{minitoc} 
\usepackage{braket}

\usepackage[most]{tcolorbox}
\newtcolorbox{bluebox}{
  colback=blue!3!white,
  colframe=blue!60!black,
}

\newtheorem{definition}{Definition}
\newtheorem{proposition}{Proposition}
\newtheorem{lemma}{Lemma}
\newtheorem{thm}{Theorem}
\newtheorem{coro}{Corollary}

\newtheorem{assum}{Assumption}

\allowdisplaybreaks

\title{Achieve Performatively Optimal Policy for Performative Reinforcement Learning}
\author{Ziyi Chen, Heng Huang\\ 
Department of Computer Science\\
University of Maryland\\
College Park, MD 20742, USA \\
\texttt{\{zc286,heng\}@umd.edu}}
\date{}

\begin{document}
\maketitle

\doparttoc 
\faketableofcontents 

\begin{abstract}
Performative reinforcement learning is an emerging dynamical decision making framework, which extends reinforcement learning to the common applications where the agent's policy can change the environmental dynamics. 
Existing works on performative reinforcement learning only aim at a performatively stable (PS) policy that maximizes an approximate value function. However, there is a provably positive constant gap between the PS policy and the desired performatively optimal (PO) policy that maximizes the original value function. In contrast, this work proposes a 
zeroth-order Frank-Wolfe algorithm (0-FW) algorithm with a zeroth-order approximation of the performative policy gradient in the Frank-Wolfe framework, and obtains \textbf{the first polynomial-time convergence to the desired PO} policy under the standard regularizer dominance condition. For the convergence analysis, we prove two important properties of the nonconvex value function. First, when the policy regularizer dominates the environmental shift, the value function satisfies a certain gradient dominance property, so that any stationary point (not PS) of the value function is a desired PO. Second, though the value function has unbounded gradient, we prove that all the sufficiently stationary points lie in a convex and compact policy subspace $\Pi_{\Delta}$, where the policy value has a constant lower bound $\Delta>0$ and thus the gradient becomes bounded and Lipschitz continuous. Experimental results also demonstrate that our 0-FW algorithm is more effective than the existing algorithms in finding the desired PO policy.
\end{abstract}

\section{Introduction}\label{sec:intro}
Reinforcement learning is a useful dynamic decision making framework with many successes in AI, such as AlphaGo \citep{silver2017mastering}, AlphaStar \citep{vinyals2019grandmaster}, Pluribus \citep{brown2019superhuman}, large language model alignment \citep{bai2022training} and reasoning \citep{havrilla2024teaching}. However, most reinforcement learning works ignore the effect of the deployed policy on the environmental dynamics, including transition kernel and reward function. This effect is significant in multi-agent systems, particularly the Stackelberg game, where leaders' policy change triggers the followers'  policy change, which in turn affects the environmental dynamics faced by the leader \citep{mandal2023performative}. For example, a recommender system (leader) affects the users' (followers) demographics and their interaction strategy with the system \citep{chaney2018algorithmic,mansoury2020feedback}. Autonomous vehicles (leaders) affect the strategies of the pedestrians and the other vehicles (followers)  \citep{nikolaidis2017game}.


To account for such effect of deployed policy on environmental dynamics, performative reinforcement learning has been proposed by \citep{mandal2023performative} where the transition kernel $p_{\pi}$ and reward function $r_{\pi}$ are modeled as functions of the deployed policy $\pi$. 
The ultimate goal is to find the \textit{performatively optimal (PO)} policy that maximizes the \textit{performative value function}, defined as the accumulated discounted reward when deploying a policy $\pi$ to its corresponding environment $(p_{\pi},r_{\pi})$. However, the policy-dependent environmental dynamics pose significant challenges to achieve PO. Hence, \citep{mandal2023performative} pursues a suboptimal \textit{performatively stable (PS)} policy using repeated retraining method with environmental dynamics fixed for the current policy at each policy optimization step. However, \citep{mandal2023performative} shows that PS can have a positive constant distance to PO. 

Extensions of the basic performative reinforcement learning problem \citep{mandal2023performative} have been proposed and all of them focus on the suboptimal PS policy. For example, \cite{rank2024performative} allows the environmental dynamics to gradually adjust to the currently deployed policy, and proposes a mixed delayed repeated retraining algorithm with accelerated convergence to a PS policy. \cite{mandal2024performative} extends \citep{mandal2023performative} from tabular setting to linear Markov decision processes with large number of states, and also obtains the convergence rate of the repeated retraining algorithm to a PS policy. \cite{pollatos2025corruption} obtains a PS policy that is robust to data contamination. \cite{sahitaj2025independent} obtains a performatively stable equilibrium as an extension of PS policy to performative Markov potential games with multiple competitive agents. 

In sum, all these existing performative reinforcement learning works pursue a suboptimal PS policy by repeated retraining algorithms. Therefore, we want to ask the following basic research question: 
\vspace{-3pt}
\begin{bluebox}
\vspace{-3pt}
\textit{\textbf{Q:} Is there an algorithm that converges to the desired performatively optimal (PO) policy?} 
\vspace{-3pt}
\end{bluebox}
\vspace{-3pt}

\subsection{Our Contributions}
We will answer affirmatively to the research question above in the following steps. Each step yields a novel contribution. 


\noindent$\bullet$ \quad We study an entropy regularized performative reinforcement learning problem, compatible with the basic performative reinforcement learning problem in \citep{mandal2023performative}. We prove that the objective function satisfies a certain gradient dominance condition, which implies that an approximate stationary point (not the suboptimal PS) is the desired approximate PO policy, under a standard regularizer dominance condition similar to that used by \citep{mandal2023performative,rank2024performative,mandal2024performative,pollatos2025corruption} to ensure convergence to a suboptimal PS policy. The proof adopts novel techniques such as recursion for $p_{\pi}$-related error term and frequent switch among various necessary and sufficient conditions of smoothness and strong concavity like properties for various variables (see Section \ref{sec:GradDom}). 

\noindent$\bullet$ \quad We obtain a policy lower bound as a decreasing function of a stationary measure. This bound not only implies the unbounded \textit{performative policy gradient} (a challenge to find a stationary policy and thus PO), but also inspires us to find a stationary policy in the policy subspace $\Pi_{\Delta}$ with a constant policy lower bound $\Delta>0$ where we prove the objective function to be Lipschitz continuous and Lipschitz smooth (a solution to this challenge). The lower bound $\Delta$ is obtained using a novel technique which simplifies a complicated inequality of the minimum policy value $\pi[a_{\min}(s)|s]$ in two cases (see Section \ref{sec:CompactLip}). 

\noindent$\bullet$ \quad We construct a zeroth-order estimation of the \textit{performative policy gradient} and obtains its estimation error. This is more challenging than  the existing zeroth-order estimation methods since our objective function is only well-defined on the policy space, a compact subset of a linear subspace of the Euclidean space $\mathbb{R}^{|\mathcal{S}||\mathcal{A}|}$. To solve this puzzle, we adjust a two-point estimation to the linear subspace $\mathcal{L}_0$ of policy difference, and simplify the estimation error analysis by mapping policies onto the Euclidean space $\mathbb{R}^{|\mathcal{S}|(|\mathcal{A}|-1)}$ via orthogonal transformation (see Section \ref{sec:g_hat}). 

\noindent$\bullet$ \quad We propose a zeroth-order Frank-Wolfe (0-FW) algorithm (see Algorithm \ref{alg:0ppg}) by combining the \textit{performative policy gradient} estimation above with the Frank-Wolfe algorithm. Then we obtain a polynomial computation complexity of our 0-FW algorithm to converge to a stationary policy, which is also the desired PO policy under the regularizer dominance condition above. The convergence analysis uses a policy averaging technique to show that an approximate stationary policy on $\Pi_{\Delta}$ is also approximately stationary on the whole policy space $\Pi$ (see Section \ref{sec:0PPG}).  

Finally, we briefly show that the results above, including gradient dominance, Lipschitz properties and the finite-time convergence of 0-FW algorithm to the desired PO, can be adjusted to the performative reinforcement learning problem with the quadratic regularizer used by \citep{mandal2023performative,rank2024performative,pollatos2025corruption} (see Appendix \ref{sec:QuadReg}).

\section{Preliminary: Performative Reinforcement Learning}
\subsection{Problem Formulation}
Performative reinforcement learning is characterized by a Markov decision process (MDP) $\mathcal{M}_{\pi}=(\mathcal{S},\mathcal{A},p_{\pi},r_{\pi},\rho)$ that depends on a certain policy $\pi$. Here, $\mathcal{S}$ and $\mathcal{A}$ denote the finite state and action spaces respectively. The policy $\pi\in[0,1]^{|\mathcal{S}||\mathcal{A}|}$, transition kernel $p_{\pi}\in[0,1]^{|\mathcal{S}|^2|\mathcal{A}|}$, reward $r_{\pi}\in[0,1]^{|\mathcal{S}||\mathcal{A}|}$, and initial state distribution $\rho\in[0,1]^{|\mathcal{S}|}$ are vectors that represent distributions. Specifically, the policy $\pi\in[0,1]^{|\mathcal{S}||\mathcal{A}|}$, with entries $\pi(a|s)$ for any state $s\in\mathcal{S}$ and action $a\in\mathcal{A}$, lies in the policy space $\Pi\overset{\rm def}{=}\big\{\pi\in[0,1]^{|\mathcal{S}|^2|\mathcal{A}|}: \sum_{a\in\mathcal{A}}\pi(a|s)=1, \forall s\in\mathcal{S}\big\}$, such that $\pi(\cdot|s)$ for any state $s$ can be seen as a distribution over $\mathcal{A}$.
The transition kernel 
$p_{\pi}\in[0,1]^{|\mathcal{S}|^2|\mathcal{A}|}$ dependent on policy $\pi\in\Pi$, with entries $p_{\pi}(s'|s,a)$ for any $s,s'\in\mathcal{S}$ and $a\in\mathcal{A}$, lies in the transition kernel space $\mathcal{P}\overset{\rm def}{=}\big\{p\in[0,1]^{|\mathcal{S}|^2|\mathcal{A}|}\!:\! \sum_{s'\in\mathcal{S}}p(s'|s,a)\!=\!1, \forall s\!\in\!\mathcal{S}, a\!\in\!\mathcal{A}\big\}$ such that $p_{\pi}(\cdot|s,a)$ can be seen as a state distribution on $\mathcal{S}$. $r_{\pi}\in\mathcal{R}\overset{\rm def}{=}[0,1]^{|\mathcal{S}||\mathcal{A}|}$ is the reward function with entries $r_{\pi}(s,a)\in[0,1]$ for any $s\in\mathcal{S}$ and $a\in\mathcal{A}$. $\rho\in[0,1]^{|\mathcal{S}|}$ is the initial state distribution such that $\sum_{s\in\mathcal{S}}\rho(s)=1$. Note that we consider $p_{\pi}$, $r_{\pi}$, $\rho$, $\pi$ as Euclidean vectors, so that we can conveniently define their Euclidean norm. For example, we define $\|p_{\pi}\|_q=\big[{\sum_{s,a,s'}|p_{\pi}(s'|s,a)|^q}\big]^{1/q}$ for any $q>1$ and $\|p_{\pi}\|_{\infty}=\max_{s,a,s'}|p_{\pi}(s'|s,a)|$. Such norms can be similarly defined over $r_{\pi}$, $\rho$, $\pi$ by summing or maximizing over all the entries. Specifically, denote $\|\cdot\|=\|\cdot\|_2$ by convention. 

When an agent applies its policy $\pi\in\Pi$ to MDP $\mathcal{M}_{\pi'}=(\mathcal{S},\mathcal{A},p_{\pi'},r_{\pi'},\rho)$, the initial environmental state $s_0\in\mathcal{S}$ is generated from the distribution $\rho$. Then at each time $t=0,1,2,\ldots$, the agent takes a random action $a_t\sim\pi(\cdot|s_t)$ based on the current state $s_t\in\mathcal{S}$, the environment transitions to the next state $s_{t+1}\sim p_{\pi'}(\cdot|s_t,a_t)$ and provides reward $r_t=r_{\pi'}(s_t,a_t)\in[0,1]$ to the agent. The value of applying policy $\pi$ to $\mathcal{M}_{\pi'}$ can be characterized by the following \textit{value function}:
\begin{align}
V_{\lambda,\pi'}^{\pi}\overset{\rm def}{=}\mathbb{E}_{\pi,p_{\pi'},\rho}\Big[\sum_{t=0}^{\infty}\gamma^tr_{\pi'}(s_t,a_t)\Big]-\lambda\mathcal{H}_{\pi'}(\pi).\label{eq:Vfunc_general}
\end{align}
Here, $\mathbb{E}_{\pi,p_{\pi'},\rho}$ is the expectation under policy $\pi$, transition kernel $p_{\pi'}$ and initial state distribution $\rho$. $\gamma\in(0,1)$ is the discount factor. $\mathcal{H}_{\pi'}(\pi)$ is a regularizer with coefficient $\lambda\ge0$ to ensure or accelerate algorithm convergence. Existing works use the quadratic regularizers such as $\mathcal{H}_{\pi'}(\pi)\!=\!\frac{1}{2}\|d_{\pi,p_{\pi'}}\|^2$ \citep{mandal2023performative,rank2024performative,pollatos2025corruption} and $\mathcal{H}_{\pi'}(\pi)\!=\!\frac{1}{2}\|\Phi^{\top} d_{\pi,p_{\pi'}}\|^2$ \citep{mandal2024performative} with a feature matrix $\Phi$, where the occupancy measure $d_{\pi,p}\in[0,1]^{|\mathcal{S}||\mathcal{A}|}$ for any policy $\pi$ and transition kernel $p$ is defined as the following distribution on $\mathcal{S}\times\mathcal{A}$. 
\begin{align}
d_{\pi,p}(s,a)\overset{\rm def}{=}(1-\gamma)\sum_{t=0}^{\infty} \gamma^t\mathbb{P}_{\pi,p,\rho} \{s_t=s,a_t=a\},\label{eq:occup}
\end{align}
\noindent Then the state occupancy measure defined as $d_{\pi,p}(s)\overset{\rm def}{=}\sum_a d_{\pi,p}(s,a)$ satisfies the following well-known Bellman equation for any state $s'\in\mathcal{S}$. 
\begin{align}
d_{\pi,p}(s')&\!=\!(1\!-\!\gamma)\rho(s')\!+\!\gamma\!\sum_{s,a}\!d_{\pi,p}(s)\pi(a|s)p(s'|s,a).\!\label{eq:Bellman_occup}
\end{align}
The goal of performative reinforcement learning is to find the \textit{performatively optimal (PO)} policy $\pi$ that maximizes the \textit{performative value function} $V_{\lambda,\pi}^{\pi}$ (with $\pi'=\pi$ in Eq. (\ref{eq:Vfunc_general})), as defined below.  
\begin{definition}[Ultimate Goal: PO]\label{def:PO}
For any $\epsilon\ge 0$, a policy $\pi\in\Pi$ is defined as \textit{$\epsilon$-performatively optimal} (\textit{$\epsilon$-PO}) if $\max_{\pi'\in\Pi}V_{\lambda,\pi'}^{\pi'}-V_{\lambda,\pi}^{\pi}\le \epsilon$. Specifically, we call a 0-PO policy as a PO policy. 
\end{definition}
Conventional reinforcement learning can be seen as a special case of performative reinforcement learning with fixed environmental dynamics, namely, fixed transition kernel $p_{\pi}\equiv p$ and fixed reward function $r_{\pi}\equiv r$. However, this may fail on applications with policy-dependent environmental dynamics, such as recommender system and autonomous driving as explained in Section \ref{sec:intro}. 
\subsection{Existing Repeated Retraining Methods for Performatively Stable (PS) Policy}
Achieving an $\epsilon$-PO policy (defined by Definition \ref{def:PO}) is challenging, due to the policy-dependent environmental dynamics $p_{\pi}$ and $r_{\pi}$. 
To alleviate the challenge, all the existing works \citep{mandal2023performative,rank2024performative,mandal2024performative,pollatos2025corruption,sahitaj2025independent} aim at a \textit{performatively stable (PS)} policy $\pi_{\rm PS}$ defined as follows, as an approximation to a \textit{PO policy}.
\begin{align}
\pi_{\rm PS}\in\mathop{\arg\max}_{\pi\in\Pi}V_{\lambda,\pi_{\rm PS}}^{\pi}.\label{eq:pi_PS}
\end{align}
In other words, a PS policy $\pi_{\rm PS}$ has the optimal value on the fixed environment $\mathcal{M}_{\pi_{\rm PS}}$. However, \cite{mandal2023performative} shows that a PS policy can be suboptimal. 

Nevertheless, we will briefly introduce the suboptimal repeated retraining algorithms in their works, to later partially inspire our method that converges to the global optimal PO policy. All these repeated retraining algorithms share the fundamental idea that in each iteration $t$, the next policy $\pi_{t+1}\!\!\approx\! \mathop{\arg\max}_{\pi\in\Pi}\!V_{\lambda,\pi_t}^{\pi}$ is obtained by solving the conventional reinforcement learning problem under fixed dynamics $p_{\pi_t}$ and $r_{\pi_t}$. This strategy highly relies on conventional reinforcement learning but fail to make full use of the policy-dependent dynamics, which leads to the suboptimal PS policy. Next, we will propose our significantly different strategies to achieve the desired PO policy.

\section{Entropy Regularized Performative Reinforcement Learning}\label{sec:properties}
In this section, we obtain critical properties of an entropy regularized performative reinforcement learning problem for achieving the desired PO policy.
\subsection{Negative Entropy Regularizer}\label{sec:entropy_reg}

We consider the following negative entropy regularizer of the policy $\pi$, which is widely used in reinforcement learning to encourage environment exploration and accelerate convergence \citep{Mnih2016Asynchronous,mankowitz2019robust,cen2022fast,chen2024acc}. 
\begin{align}
\mathcal{H}_{\pi'}(\pi)=\mathbb{E}_{\pi,p_{\pi'},\rho}\Big[\sum_{t=0}^{\infty}\gamma^t\log\pi(a_t|s_t)\Big].\label{eq:entropy}
\end{align}
In addition, this negative entropy regularizer can be seen as a strongly convex function of the occupancy measure $d_{\pi,p_{\pi'}}$ (proved in Appendix \ref{sec:sc_entropy}), which is critical to develop algorithms convergent to a PO (see Theorem \ref{thm:ToOpt} later) or PS policy \citep{mandal2023performative}. For optimization problem on a probability simplex variable (policy $\pi$ or occupancy measure $d$), negative entropy regularizer is more natural and yields faster theoretical convergence than the quadratic regularizers used in the existing performative reinforcment learning works \citep{mandal2023performative,rank2024performative,pollatos2025corruption} (see pages 43-45 of \citep{chen2020mirror} for explanation).

Therefore, we will mainly focus on the following entropy-regularized value function, which is obtained by substituting the negative entropy regularizer (\ref{eq:entropy}) into the general value function (\ref{eq:Vfunc_general}). 
\begin{align}
V_{\lambda,\pi'}^{\pi}\overset{\rm def}{=}\mathbb{E}_{\pi,p_{\pi'},\rho}\Big[\sum_{t=0}^{\infty}\gamma^t[r_{\pi'}(s_t,a_t)-\lambda\log\pi(a_t|s_t)]\Big].\label{eq:Vfunc}
\end{align}
Specifically, we will study the critical properties of the entropy-regularized value function (\ref{eq:Vfunc}) (Section \ref{sec:alg_all}) to develop algorithm that converges to PO (Sections \ref{sec:g_hat}-\ref{sec:0PPG}). Then we will briefly discuss about how to adjust these results to the existing quadratic regularizers (Appendix \ref{sec:QuadReg}). 

We make the following standard assumptions to study the properties of the value function (\ref{eq:Vfunc}). 
\begin{assum}[Sensitivity]\label{assum:sensitive} There exist constants $\epsilon_p,\epsilon_r>0$ such that for any $\pi,\pi'\in\Pi$, 
\begin{align}
    \|p_{\pi'}\!-\!p_{\pi}\|\!\le\!\epsilon_p\|\pi'\!-\!\pi\|,\quad\|r_{\pi'}\!-\!r_{\pi}\|\!\le\!\epsilon_r\|\pi'\!-\!\pi\|\label{eq:sensitive}
\end{align}    
\end{assum}
\begin{assum}[Smoothness]\label{assum:smooth_pr}
$p_{\pi}$ and $r_{\pi}$ are Lipschitz smooth with modulus $S_p, S_r>0$ respectively, that is, for any $\pi\in\Pi$, $s,s'\in\mathcal{S}$, $a\in\mathcal{A}$, we have
\begin{align}
\|\nabla_{\pi} p_{\pi'}(s'|s,a)-\nabla_{\pi} p_{\pi}(s'|s,a)\|\le& S_p\|\pi'-\pi\|,\label{eq:smooth_p_pi}\\
\|\nabla_{\pi} r_{\pi'}(s,a)-\nabla_{\pi} r_{\pi}(s,a)\|\le& S_r\|\pi'-\pi\|.\label{eq:smooth_r_pi}
\end{align}
\end{assum}
\begin{assum}\label{assum:dmin}
There exists a constant $D>0$ such that $\inf_{\pi\in\Pi,p\in\mathcal{P},s\in\mathcal{S}}d_{\pi,p}(s)\ge D$. 
\end{assum}
Assumptions \ref{assum:sensitive}-\ref{assum:smooth_pr} ensure that the environmental dynamics $p_{\pi}$ and $r_{\pi}$ adjust continuously and smoothly to policy $\pi$, and thus the \textit{performative value function} $V_{\lambda,\pi}^{\pi}$ is differentiable with \textit{performative policy gradient} $\nabla_{\pi} V_{\lambda,\pi}^{\pi}$. Similar versions of Assumption \ref{assum:sensitive} on environmental sensitivity have also been used for performative reinforcement learning \citep{mandal2023performative,rank2024performative,mandal2024performative,pollatos2025corruption,sahitaj2025independent}. 
Assumption \ref{assum:dmin} has been used \citep{zhang2021beyond,sahitaj2025independent} or implied by stronger assumptions \citep{wei2021last,chen2021sample,agarwal2021theory,leonardos2022global,PGrobust_ICML2023,chen2024acc,bhandari2024global} in conventional reinforcement learning (see Appendix \ref{sec:dmin} for the proof), which guarantees that each state is visited sufficiently often. 



\subsection{Gradient Dominance}\label{sec:GradDom}
For the nonconvex policy optimization problem $\max_{\pi\in\Pi}V_{\lambda,\pi}^{\pi}$ in Eq. (\ref{eq:Vfunc}) on the convex policy space $\Pi$, it is natural to consider its approximate stationary solution as defined below. 
\begin{definition}[Stationary Policy]
For any $\epsilon\ge 0$, a policy $\pi\in\Pi$ is $\epsilon$-stationary if $\max_{\Tilde{\pi}\in\Pi}\big\langle \nabla_{\pi}V_{\lambda,\pi}^{\pi},\Tilde{\pi}-\pi\big\rangle\le\epsilon$. We call a 0-stationary policy as a stationary policy. 
\end{definition}
Note that for a policy to be the desired PO, it is necessary to be  stationary, while the PS policy targeted by existing works is neither necessary nor sufficient. Furthermore, we will show that stationary policy can also be a sufficient condition of the desired PO under mild conditions. As a preliminary step, we show the important gradient dominance property of the objective function as follows. 
\begin{thm}[Gradient Dominance]\label{thm:ToOpt}
Under Assumptions \ref{assum:sensitive}-\ref{assum:dmin}, the entropy regularized value function (\ref{eq:Vfunc}) satisfies the following gradient dominance property for any $\pi_0,\pi_1\in\Pi$. 
\begin{align}
V_{\lambda,\pi_1}^{\pi_1}\le& V_{\lambda,\pi_0}^{\pi_0}+D^{-1}\max_{\pi\in\Pi}\big\langle \nabla_{\pi_0}V_{\lambda,\pi_0}^{\pi_0},\pi-\pi_0\big\rangle-\frac{\mu}{2}\|\pi_1-\pi_0\|^2,\label{eq:ToOpt}
\end{align}
where 
\begin{align}
\mu\overset{\rm def}{=}&\frac{D\lambda}{1-\gamma}-\frac{6\gamma|\mathcal{S}|(1+\lambda\log|\mathcal{A}|)}{D(1-\gamma)^3}\big[\epsilon_p\big(\sqrt{|\mathcal{A}|}+\gamma\epsilon_p\sqrt{|\mathcal{S}|}\big)+S_p(1-\gamma)\big]\nonumber\\
&-\frac{S_r(1-\gamma)+4\epsilon_r(\sqrt{|\mathcal{A}|}+\epsilon_p\sqrt{|\mathcal{S}|})}{D^2(1-\gamma)^2},\label{eq:mu_main}
\end{align}
\end{thm}
The gradient dominance property above generalizes that used in the conventional unregularized reinforcement learning (see Lemma 4 of \citep{agarwal2021theory}), which implies that stationary policy is close to a PO policy as shown in the corollary below. 
\begin{coro}\label{coro:stat2PO}
Under Assumptions \ref{assum:sensitive}-\ref{assum:dmin}, any $D\epsilon$-stationary policy is an $(\epsilon+|\mu||\mathcal{S}|)$-PO policy. Furthermore, this is 
also the desired $\epsilon$-PO policy if $\mu\ge 0$. The PO policy is unique if $\mu>0$.
\end{coro}
\noindent\textbf{Remark: } Corollary \ref{coro:stat2PO} implies that a $D\epsilon$-stationary policy is always ($\epsilon+|\mu||\mathcal{S}|$)-close to the desired PO policy with $|\mu|$ proportional to the environmental sensitivity $\mathcal{O}(\epsilon_p+\epsilon_r+S_p+S_r)$. Furthermore, since $\mu=[\mathcal{O}(1)-\mathcal{O}(\epsilon_p+S_p)]\lambda-\mathcal{O}(\epsilon_p+\epsilon_r+S_p+S_r)$ by Eq. (\ref{eq:mu_main}), when $\mathcal{O}(\epsilon_p+S_p)<\mathcal{O}(1)$ and the regularizer strength dominates the environmental shift ($\lambda\ge\frac{\mathcal{O}(\epsilon_p+\epsilon_r+S_p+S_r)}{\mathcal{O}(1)-\mathcal{O}(\epsilon_p+S_p)}$), we have $\mu\ge 0$ so that the $D\epsilon$-stationary policy is also the desired $\epsilon$-PO policy. Note that similar regularizer dominance condition has also been used to guarantee convergence to a suboptimal PS policy \citep{mandal2023performative,rank2024performative,mandal2024performative,pollatos2025corruption}. 

\noindent\textbf{Intuition and Novelty for Proving Theorem \ref{thm:ToOpt}: } 
Define the following more refined value function
\begin{align}
&J_{\lambda}(\pi,\pi',p,r)\overset{\rm def}{=}\mathbb{E}_{\pi,p}\Big[\sum_{t=0}^{\infty}\gamma^t[r(s_t,a_t)\!-\!\lambda\log\pi'(a_t|s_t)]\Big|s_0\!\sim\!\rho\Big].\label{eq:J_txt}
\end{align}
To get the intuition, we will first prove the bound (\ref{eq:ToOpt}) in the special case with fixed $p_{\pi}\equiv p$ and $r_{\pi}\equiv r$. Then we allow non-constant $p_{\pi}$ to inspect the perturbation on the bound (\ref{eq:ToOpt}), and finally see the effect of non-constant $r_{\pi}$ on the bound (\ref{eq:ToOpt}). 


(Step 1): For conventional reinforcement learning with fixed  $p_{\pi}\equiv p$ and $r_{\pi}\equiv r$, denote $d_{\alpha}=\alpha d_{\pi_1,p}+(1-\alpha)d_{\pi_0,p}$ ($\alpha\in[0,1]$). Based on the Bellman equation (\ref{eq:Bellman_occup}), $d_{\alpha}=d_{\pi_{\alpha},p}$ is the occupancy measure of the policy $\pi_{\alpha}(a|s)=\frac{d_{\alpha}(s,a)}{d_{\alpha}(s)}$. Therefore, $V_{\lambda,\pi_{\alpha}}^{\pi_{\alpha}}$ can be rewritten as $J_{\lambda}(\pi_{\alpha},\pi_{\alpha},p,r)=\sum_{s,a}d_{\alpha}(s,a)[r(s,a)-\lambda\log\pi_{\alpha}(a|s)]$, which has the following strong concavity like property by Pinsker's inequality. 
\begin{align}
&J_{\lambda}(\pi_{\alpha},\pi_{\alpha},p,r)-\alpha J_{\lambda}(\pi_1,\pi_1,p,r)-(1-\alpha)J_{\lambda}(\pi_0,\pi_0,p,r)\nonumber\\
=&\frac{1}{1-\gamma}\sum_s\big[\alpha d_1(s){\rm KL}[\pi_1(\cdot|s)\|\pi_{\alpha}(a|s)]+(1-\alpha)d_0(s){\rm KL}[\pi_0(\cdot|s)\|\pi_{\alpha}(a|s)]\big]\nonumber\\
\ge& \frac{D\lambda\alpha(1-\alpha)}{2(1-\gamma)}\|\pi_1-\pi_0\|^2.\label{eq:J_RL}
\end{align} 

(Step 2): Consider a harder case with non-constant $p_{\pi}$ and constant reward $r_{\pi}\equiv r$. Similarly, denote $d_{\alpha}=\alpha d_{\pi_1,p_{\pi_1}}+(1-\alpha)d_{\pi_0,p_{\pi_0}}$ and $\pi_{\alpha}(a|s)=\frac{d_{\alpha}(s,a)}{d_{\alpha}(s)}$. The non-constant $p_{\pi}$ brings a major challenge that $d_{\alpha}=d_{\pi_{\alpha},p_{\pi_{\alpha}}}$ required by Step 1 above no longer holds. To solve this challenge, we need to bound the error term $e_{\alpha}(s)=d_{\pi_{\alpha},p_{\alpha}}(s)-d_{\alpha}(s)$ which we prove to satisfy the following novel recursion. 
\begin{align}
e_{\alpha}(s')=\gamma\sum_{s,a}\big[e_{\alpha}(s)\pi_{\alpha}(a|s)p_{\pi_\alpha}(s'|s,a)+h_{\alpha}(s,a,s')\big],\nonumber  
\end{align}
where $h_{\alpha}(s,a,s')=d_{\alpha}(s,a)p_{\pi_\alpha}(s'|s,a)-\alpha d_1(s,a)p_{\pi_1}(s'|s,a)-(1-\alpha)d_0(s,a)p_{\pi_0}(s'|s,a)$. Since $d_{\alpha}(s,a)p_{\pi_\alpha}(s'|s,a)$ is a Lipschitz smooth function of $\alpha$, we can upper bound  $|h_{\alpha}(s,a,s')|$ and substitute this bound to the recursion above, which yields the following novel error bound.
\begin{align}
\sum_{s}|e_{\alpha}(s)|\le \frac{3\gamma|\mathcal{S}|\alpha(1-\alpha)}{D(1-\gamma)^2}\|\pi_1-\pi_0\|^2\big[\epsilon_p\big(\sqrt{|\mathcal{A}|}+\gamma\epsilon_p\sqrt{|\mathcal{S}|}\big)+S_p(1-\gamma)\big],\nonumber
\end{align}
The bound above reflects the effect of non-constant $p_{\pi}$, which perturbs the bound (\ref{eq:J_RL}) into
\begin{align}
&J_{\lambda}(\pi_{\alpha},\pi_{\alpha},p_{\alpha},r)\!-\!\alpha J_{\lambda}(\pi_1,\pi_1,p_1,r)\!-\!(1\!-\!\alpha)J_{\lambda}(\pi_0,\pi_0,p_0,r)
\ge\frac{\alpha(1\!-\!\alpha)\mu_1}{2}\|\pi_1\!-\!\pi_0\|^2,\label{eq:J_ppi}
\end{align}
where 
$\mu_1\overset{\rm def}{=}\frac{D\lambda}{1-\gamma}-\frac{6\gamma|\mathcal{S}|(1+\lambda\log|\mathcal{A}|)}{D(1-\gamma)^3}\big[\epsilon_p\big(\sqrt{|\mathcal{A}|}+\gamma\epsilon_p\sqrt{|\mathcal{S}|}\big)+S_p(1-\gamma)\big]$ equals $\mu$ in Eq. (\ref{eq:mu_main}) when $\epsilon_r=S_r=0$. 

(Step 3): Now we consider performative reinforcement learning with non-constant $p_{\pi}$ and $r_{\pi}$. The policy $\pi_{\alpha}$ and its occupancy measure $d_{\alpha}$ are the same as in Case II above. Then the function $w(\alpha)=\alpha J_{\lambda}(\pi_1,\pi_1,p_1,r_{\alpha})+(1-\alpha) J_{\lambda}(\pi_0,\pi_0,p_0,r_{\alpha})$ can be proved $\mu_2\|\pi_1-\pi_0\|^2$-Lipschitz smooth with parameter $\mu_2=\mu-\mu_1\ge 0$. Using $r=r_{\alpha}$ in Eq. (\ref{eq:J_ppi}), we obtain the following strong concavity like property with $\mu=\mu_1-\mu_2$. 
\begin{align}
&V_{\lambda,\pi_{\alpha}}^{\pi_{\alpha}}-\alpha V_{\lambda,\pi_1}^{\pi_1}-(1-\alpha)V_{\lambda,\pi_0}^{\pi_0}\nonumber\\
=&J_{\lambda}(\pi_{\alpha},\pi_{\alpha},p_{\alpha},r_{\alpha})-\alpha J_{\lambda}(\pi_1,\pi_1,p_1,r_1)-(1-\alpha)J_{\lambda}(\pi_0,\pi_0,p_0,r_0)\nonumber\\
\ge&\frac{\alpha(1\!-\!\alpha)\mu_1}{2}\|\pi_1\!-\!\pi_0\|^2+w(\alpha)\!-\!\alpha w(1)\!-\!(1\!-\!\alpha) w(0)\ge \frac{\alpha(1-\alpha)\mu}{2}\|\pi_1-\pi_0\|^2.\nonumber
\end{align}
Finally, the dominance property (\ref{eq:ToOpt}) follows from the inequality above as $\alpha\to +0$. 


\subsection{Policy Lower Bound and Lipschitz Properties}\label{sec:CompactLip}
\textbf{Policy Lower Bound:} Based on Section \ref{sec:GradDom}, we can focus on achieving an $\epsilon$-stationary policy. A major challenge is the unbounded \textit{performative policy gradient} $\nabla_{\pi} V_{\lambda,\pi}^{\pi}$ on $\Pi$. Specifically, we will show that as $\pi(a|s)\to 0$ for any state $s$ and action $a$, $\|\nabla_{\pi} V_{\lambda,\pi}^{\pi}\|\to +\infty$. To tackle this challenge, we prove the following policy lower bound. 
\begin{thm}\label{thm:pi_ge}
If Assumptions \ref{assum:sensitive} and \ref{assum:dmin} hold, and $p_{\pi}$, $r_{\pi}$ are differentiable functions of $\pi$, then there exists a constant $\pi_{\min}>0$ (see its value in Eq. (\ref{eq:pi_min}) in Appendix \ref{sec:proof_thm:pi_ge}) the following policy lower bound holds for any $\pi\in\Pi$, $s\in\mathcal{S}$, $a\in\mathcal{A}$.
\begin{align}
\pi(a|s)\ge&\pi_{\min}\exp\Big[-\frac{2|\mathcal{A}|}{\lambda}(1-\gamma)\langle\nabla_{\pi} V_{\lambda,\pi}^{\pi},\pi'-\pi\rangle\Big],\label{eq:pi_ge2}
\end{align}
Here, the policy $\pi'$ is defined as follows depending on $\pi$: 
\begin{align}
\pi'(a|s)=\left\{\begin{aligned}
&\pi[a_{\min}(s)|s], \quad~a=a_{\max}(s)\\
&\pi[a_{\max}(s)|s], \quad~a=a_{\min}(s)\\
&\pi(a|s), \quad\quad\quad~~~{\rm Otherwise}
\end{aligned}\right.,\label{eq:pi_pie}
\end{align}
where $a_{\max}(s)\in{\arg\max}_a\pi(a|s)$ and $a_{\min}(s)\in{\arg\min}_a\pi(a|s)$. 
\end{thm}
\noindent\textbf{Implications of Theorem \ref{thm:pi_ge}: } First, as $\pi(a|s)\to 0$, we have $\langle\nabla_{\pi} V_{\lambda,\pi}^{\pi},\pi'\!-\!\pi\rangle\to +\infty$, so $\|\nabla_{\pi} V_{\lambda,\pi}^{\pi}\|\to +\infty$ as aforementioned. Second, any stationary policy $\pi$ satisfies $\langle\nabla_{\pi} V_{\lambda,\pi}^{\pi},\pi'-\pi\rangle\le 0$, so $\pi(a|s)\ge\pi_{\min}$. Therefore, we can search $\epsilon$-stationary policy on the convex and compact policy subspace $\Pi_{\Delta}\overset{\rm def}{=}\{\pi\in\Pi:\pi(a|s)\ge\Delta\}$ with lower bound $\Delta\in(0,\pi_{\min}]$. 

\noindent\textbf{Intuition and Novelty for Proving Theorem \ref{thm:pi_ge}: }  
At first, consider conventional reinforcement learning with fixed environmental dynamics $p_{\pi}\equiv p$ and $r_{\pi}\equiv r$. In this case, $\nabla_{\pi} V_{\lambda,\pi}^{\pi}$ has analytical form (see Eq. (\ref{eq:pi_grad})), so by direct computation we obtain the following inequality with constant $C=1+\frac{\gamma(1+\lambda\log|\mathcal{A}|)}{1-\gamma}$ (see Eq. (\ref{eq:inprod_ge}) for detail)
\begin{align}
\langle\nabla_{\pi} J_{\lambda}(\pi,\pi,p,r),\pi'-\pi\rangle\!\ge\!\frac{1}{1\!-\!\gamma}\!\max_s\!\Big\{\!\big(\pi[a_{\max}(s)|s]\!-\!\pi[a_{\min}(s)|s]\big)\!\Big[\lambda\!\log\!\frac{\pi[a_{\max}(s)|s]}{\pi[a_{\min}(s)|s]}\!-C\Big]\!\Big\}.\nonumber
\end{align}
To obtain a lower bound of $\pi[a_{\min}(s)|s]$, we simplify the inequality above by considering two cases, $\pi[a_{\min}(s)|s]\ge \frac{1}{2}\pi[a_{\max}(s)|s]\ge \frac{1}{2|\mathcal{A}|}$ and $\pi[a_{\min}(s)|s]<\frac{1}{2}\pi[a_{\max}(s)|s]$. In the second case, we replace $\pi[a_{\max}(s)|s]$ and $\pi[a_{\max}(s)|s]-\pi[a_{\min}(s)|s]$ above with their lower bounds $\frac{1}{|\mathcal{A}|}$ and $\frac{1}{2|\mathcal{A}|}$ respectively. Then combining the two cases proves the lower bound (\ref{eq:pi_ge2}) at the special case of $\epsilon_p=\epsilon_r=0$. Then we extend from conventional reinforcement learning to performative reinforcement learning which involves a gradient perturbation with magnitude of at most $\mathcal{O}(\epsilon_p+\epsilon_r)$ (see Eq. (\ref{eq:dJ_diff}) for detail) based on the chain rule and leads to the lower bound (\ref{eq:pi_ge2}) for any $\epsilon_p,\epsilon_r\ge0$. 

\noindent\textbf{Lipschitz Properties:} Theorem \ref{thm:pi_ge} inspires us to find an $\epsilon$-stationary policy in the policy subspace $\Pi_{\Delta}$, where the \textit{performative value function} $V_{\lambda,\pi}^{\pi}$ is Lipschitz continuous and Lipschitz smooth as follows.
\begin{thm}\label{thm:V_Lip}
Under Assumptions \ref{assum:sensitive}-\ref{assum:smooth_pr}, there exist constants $L_{\lambda}, \ell_{\lambda}>0$ (see the values in Eqs. (\ref{eq:Ltau}) and (\ref{eq:ell_tau}) in Appendix \ref{sec:proof_thm:V_Lip}) such that the following Lipschitz propreties hold for any $\Delta>0$ and $\pi, \pi'\in\Pi_{\Delta}$.
\begin{align}
|V_{\lambda,\pi'}^{\pi'}-V_{\lambda,\pi}^{\pi}|\le\frac{L_{\lambda}}{\Delta}\|\pi'-\pi\|,\quad\quad\|\nabla_{\pi'}V_{\lambda,\pi'}^{\pi'}-\nabla_{\pi}V_{\lambda,\pi}^{\pi}\|\le\frac{\ell_{\lambda}}{\Delta}\|\pi'-\pi\|.\label{eq:V_Lip}
\end{align}
\end{thm}


\section{Zeroth-Order Frank-Wolfe (0-FW) Algorithm}\label{sec:alg_all}
\subsection{Performative Policy Gradient Estimation}\label{sec:g_hat}
In Section \ref{sec:properties}, we have obtained important properties of the entropy regularized \textit{performative value function} $V_{\lambda,\pi}^{\pi}$ (defined by Eq. (\ref{eq:Vfunc})), which indicates that it suffices to find an $\epsilon$-stationary policy in the subspace $\Pi_{\Delta}$ for $\Delta\in(0,\pi_{\min}]$. To achieve this goal, an accurate estimation of the \textit{performative policy gradient} $\nabla_{\pi} V_{\lambda,\pi}^{\pi}$ is important {but also challenging, since the performative policy gradient involves the unknown gradients $\nabla_{\pi}p_{\pi}(s'|s,a)$ and $\nabla_{\pi}r_{\pi}(s,a)$.} 

Despite these challenges in estimating $\nabla_{\pi} V_{\lambda,\pi}^{\pi}$, note that $V_{\lambda,\pi}^{\pi}$ for any policy $\pi$ can be evaluated by policy evaluation in conventional reinforcement learning under fixed environment $p_{\pi}$ and $r_{\pi}$ (for fixed $\pi$). Furthermore, for any $\epsilon_V>0$ and $\eta\in(0,1)$, many existing policy evaluation algorithms such as temporal difference \citep{bhandari2018finite,li2023sharp,samsonov2023finite}, can obtain $\hat{V}_{\lambda,\pi}^{\pi}\approx V_{\lambda,\pi}^{\pi}$ with small error bound $|\hat{V}_{\lambda,\pi}^{\pi}-V_{\lambda,\pi}^{\pi}|\le\epsilon_V$ with probability at least $1-\eta$. 

As a result, we will consider a zeroth-order estimation of $\nabla_{\pi} V_{\lambda,\pi}^{\pi}$ using policy evaluation. However, this has another challenge that $V_{\lambda,\pi}^{\pi}$ is only well-defined on $\pi\in\Pi$, so we cannot directly apply the existing zeroth-order estimation methods \citep{agarwal2010optimal,shamir2017optimal,malik2020derivative} which require the objective function to be well-defined on a sphere. Fortunately, for any $\pi,\pi'\in\Pi$, the policy difference $\pi'-\pi$ lies in the following linear subspace of dimensionality $|\mathcal{S}|(|\mathcal{A}|-1)$.
\begin{align}
\mathcal{L}_0\overset{\rm def}{=}\Big\{u\in\mathbb{R}^{|\mathcal{S}||\mathcal{A}|}\!:\sum_{a}u(a|s)\!=\!0, \forall s\in\mathcal{S}\Big\}.\label{eq:L0}
\end{align}
Therefore, inspired by the popular two-point zeroth-order estimations, we estimate $\nabla_{\pi} V_{\lambda,\pi}^{\pi}$ as follows.
\begin{align}
\hat{g}_{\lambda,\delta}(\pi)\!=\!\frac{|\mathcal{S}|(|\mathcal{A}|\!-\!1)}{2N\delta}\sum_{i=1}^N\!\big(\hat{V}_{\lambda,\pi+\delta u_i}^{\pi+\delta u_i}\!-\!\hat{V}_{\lambda,\pi-\delta u_i}^{\pi-\delta u_i}\big)u_i, \label{eq:0ppg}
\end{align}
where $\{u_i\}_{i=1}^N$ are i.i.d. samples uniformly from $U_1\cap \mathcal{L}_0$ with $U_1\overset{\rm def}{=}\{u\in\mathbb{R}^{|\mathcal{S}||\mathcal{A}|}\!:\|u\|\!=\! 1\}$.
Our estimation (\ref{eq:0ppg}) above is more tricky than the existing two-point zeroth-order estimations \citep{agarwal2010optimal,shamir2017optimal,malik2020derivative} where $u_i$ is uniformly distributed on $U_1$. To elaborate, we replace their $U_1$ with $U_1\cap \mathcal{L}_0$, a unit sphere on the linear subspace $\mathcal{L}_0$, and further require $\pi\in\Pi_{\Delta}$ and $\delta<\Delta$, to guarantee that $\pi+\delta u_i, \pi-\delta u_i\in \Pi$ for any $u_i\in U_1\cap\mathcal{L}_0$ and thus the gradient estimation (\ref{eq:0ppg}) is well-defined (see Appendix \ref{sec:proof_prop_grad_err} for the proof). 
Moreover, we use the following three steps to obtain $u_i$ uniformly from $U_1\cap \mathcal{L}_0$: (1) Obtain $v_i$ uniformly from $U_1$; (2) Project $v_i$ onto $\mathcal{L}_0$ as Eq. (\ref{eq:proj_vu}) below; (3) Normalize this projection by $u_i={\rm proj}_{\mathcal{L}_0}(v_i)/\|{\rm proj}_{\mathcal{L}_0}(v_i)\|$.
\begin{align}
{\rm proj}_{\mathcal{L}_0}(v_i)(a|s)&=v_i(a|s)-\frac{1}{|\mathcal{A}|}\sum_{a'}v_i(a'|s).\label{eq:proj_vu}
\end{align}
The gradient estimation (\ref{eq:0ppg}) has the following provable error bound. 
\begin{proposition}\label{prop:grad_err}
For any $\Delta>\delta>0$, $\eta\in(0,1)$ and $\pi\in\Pi_{\Delta}$, the stochastic gradient (\ref{eq:0ppg}) is well-defined (i.e., $\pi+\delta u_i$ and $\pi-\delta u_i$ therein are valid policies defined by $\Pi$) and approximates the projected performative policy gradient ${\rm proj}_{\mathcal{L}_0}(\nabla_{\pi} V_{\lambda,\pi}^{\pi})$ with the following error bound (see its full expression in Eq. (\ref{eq:conclude_gerr}) in Appendix \ref{sec:proof_prop_grad_err}), with probability at least $1-\eta$. 
\begin{align}
\|\hat{g}_{\lambda,\delta}(\pi)-{\rm proj}_{\mathcal{L}_0}(\nabla_{\pi} V_{\lambda,\pi}^{\pi})\| \le \mathcal{O}\big(\frac{\epsilon_V}{\delta}+\frac{\log(N/\eta)}{\sqrt{N}}+\delta\big). \label{eq:grad_err}
\end{align}
\end{proposition}  
\noindent\textbf{Remark: } Proposition \ref{prop:grad_err} above aims to approximate ${\rm proj}_{\mathcal{L}_0}(\nabla_{\pi} V_{\lambda,\pi}^{\pi})$ instead of $\nabla_{\pi} V_{\lambda,\pi}^{\pi}$. This is sufficient to find an $\epsilon$-stationary policy, because for any policies $\pi,\pi'$, the stationarity measure only involves $\langle \nabla_{\pi} V_{\lambda,\pi}^{\pi}, \pi'\!-\!\pi\rangle\!=\!\langle {\rm proj}_{\mathcal{L}_0}(\nabla_{\pi} V_{\lambda,\pi}^{\pi}), \pi'\!-\!\pi\rangle$ as $\pi'\!-\!\pi\in\mathcal{L}_0$. Therefore, we only care about ${\rm proj}_{\mathcal{L}_0}(\nabla_{\pi} V_{\lambda,\pi}^{\pi})$. The estimation error (\ref{eq:grad_err}) above can be arbitrarily small with sufficiently large batchsize $N$ (to reduce the variance), small $\delta$ (to reduce the bias), and policy evaluation error $\epsilon_V\ll\delta$.  

\noindent\textbf{Intuition and Novelty for Proving Proposition \ref{prop:grad_err}: } Unlike existing zeroth-order estimations on the whole Euclidean space, our estimation (\ref{eq:0ppg}) is made on the policy space $\Pi$, which lies in the linear manifold $\mathcal{L}_0+|\mathcal{A}|^{-1}\subset\mathbb{R}^{|\mathcal{S}||\mathcal{A}|}$. The key to our proof is to find an orthogonal transformation $T: \mathbb{R}^{|\mathcal{S}|(|\mathcal{A}|-1)}\to \mathcal{L}_0$, so that the goal is simplified to analyze the gradient estimation of $f_{\lambda}(x)\overset{\rm def}{=}V_{\lambda,T(x)+|\mathcal{A}|^{-1}}^{T(x)+|\mathcal{A}|^{-1}}$ on any $x\in \mathbb{R}^{|\mathcal{S}|(|\mathcal{A}|-1)}$. 
\subsection{Zeroth-Order Frank-Wolfe (0-FW) Algorithm}\label{sec:0PPG}
With the estimated gradient $\hat{g}_{\lambda,\delta}(\pi_t)$ defined by Eq. (\ref{eq:0ppg}), we consider the following Frank-Wolfe algorithm to find an $\epsilon$-stationary policy. 
\begin{align}
\Tilde{\pi}_t=&{\arg\max}_{\pi\in\Pi_{\Delta}}\langle \pi,\hat{g}_{\lambda,\delta}(\pi_t)\rangle,\label{eq:pi_wolfe}\\
\pi_{t+1}=&\pi_t+\beta(\Tilde{\pi}_t-\pi_t).\label{eq:pi_update}
\end{align}
\vspace{-0.03\textwidth}
\begin{lemma}\label{lemma:wolfe}
The step (\ref{eq:pi_wolfe}) has the analytical solution below.
\begin{align}
\Tilde{\pi}_t(a|s)=\left\{
\begin{aligned}
&\Delta; a\ne\Tilde{a}_t(s)\\
&1-\Delta(|\mathcal{A}|-1); a=\Tilde{a}_t(s)
\end{aligned}
\right.,\label{eq:pi_wolfe_sol}
\end{align}
where $\Tilde{a}_t(s)\in {\arg\max}_{a}\hat{g}_{\lambda,\delta}(\pi_t)(a|s)$. 
\end{lemma}
See the proof of Lemma \ref{lemma:wolfe} in Section \ref{sec:FW}. Then combining the \textit{performative policy gradient} estimation (see Section \ref{sec:entropy_reg}) with the Frank-Wolfe algorithm, we propose our zeroth-order Frank-Wolfe (0-FW) algorithm (see Algorithm \ref{alg:0ppg}). 
\begin{wrapfigure}{R}{0.62\textwidth}
\begin{minipage}{0.62\textwidth}
\vspace{-0.03\textwidth}
\begin{algorithm}[H] 
\caption{Zeroth-order Frank-Wolfe (0-FW) Algorithm}
    \begin{algorithmic}[1]\label{alg:0ppg}
    \STATE \textbf{Inputs:} $T$, $N$, $\Delta>\delta>0$, $\epsilon_V\ge 0$, $\beta>0$.
    \STATE \textbf{Initialize:} policy $\pi_0\in\Pi_{\Delta}$.\\
    \FOR{Iterations $t=0,1,\ldots,T-1$}
    {
        \STATE Obtain i.i.d. vectors $\{v_i\}_{i=1}^N$ uniformly from the unit sphere $U_1\!\overset{\rm def}{=}\!\{\!u\!\in\!\mathbb{R}^{|\mathcal{S}||\mathcal{A}|}\!:\|u\|\!=\!1\!\}$. 
        \STATE Obtain $\{{\rm proj}_{\mathcal{L}_0}(v_i)\}_{i=1}^N$ from Eq. (\ref{eq:proj_vu}).
        \STATE Obtain $\{u_i\}_{i=1}^N$ where $u_i={\rm proj}_{\mathcal{L}_0}(v_i)/\|{\rm proj}_{\mathcal{L}_0}(v_i)\|$.
        \STATE Obtain stochastic policy evaluation  $\hat{V}_{\lambda,\pi}^{\pi}\approx V_{\lambda,\pi}^{\pi}$ which satisfies $|\hat{V}_{\lambda,\pi}^{\pi}-V_{\lambda,\pi}^{\pi}|\le\epsilon_V$ for $\pi\in\{\pi_t\pm\delta u_i\}_{i=1}^N$. 
        \STATE Obtain stochastic performative policy gradient estimation $\hat{g}_{\lambda,\delta}(\pi_t)$ using Eq. (\ref{eq:0ppg}).
        \STATE Obtain $\Tilde{\pi}_t$ by Eq. (\ref{eq:pi_wolfe_sol}). 
        \STATE Update $\pi_{t+1}$ by Eq. (\ref{eq:pi_update}).
    }\ENDFOR
    \STATE {\bf Output:} $\pi_{\widetilde{T}}$ where $\widetilde{T}\!\in\!\mathop{\arg\min}_{0\le t\le T\!-\!1}\!\langle\hat{g}_{\lambda,\delta}(\pi_t),\!\Tilde{\pi}_t\!-\!\pi_t\rangle$.
    \end{algorithmic}
\end{algorithm}
\end{minipage}
\vspace{-0.01\textwidth}
\end{wrapfigure}
We obtain the following convergence result of Algorithm \ref{alg:0ppg} in Theorem \ref{thm:0ppg_rate}, the main theoretical result of this work, as follows. 
\begin{thm}\label{thm:0ppg_rate}
Suppose Assumptions \ref{assum:sensitive}-\ref{assum:dmin} hold. For any $\eta\in(0,1)$ and precision $0<\epsilon$$\le \min\big[24\sqrt{2|\mathcal{S}|}\frac{\ell_{\lambda}}{D}, $\\ 
$\frac{2\lambda}{5|\mathcal{A}|D^2(1-\gamma)},\frac{288L_{\lambda}|\mathcal{S}|^{1.5}|\mathcal{A}|}{D\pi_{\min}}\big]$, select the following hyperparameters for Algorithm \ref{alg:0ppg}: $\Delta=\frac{\pi_{\min}}{3}$, $\beta=\frac{D\pi_{\min}\epsilon}{36\ell_{\lambda}|\mathcal{S}|}$, $\delta=\mathcal{O}(\epsilon)$, $\epsilon_V=\mathcal{O}(\epsilon^2)$, $N=\mathcal{O}[\epsilon^{-2}\log(\eta^{-1}\epsilon^{-1})]$, and the number of iterations $T=\mathcal{O}(\epsilon^{-2})$ (see Eqs. (\ref{eq:Delta})-(\ref{eq:N}) in Appendix \ref{sec:proof_0ppg_rate} for detailed expression of these hyperparameters). Then with probability at least $1-\eta$, the output policy $\Tilde{\pi}_{\Tilde{T}}$ of Algorithm \ref{alg:0ppg} is a $D\epsilon$-stationary policy. Furthermore, if $\mu\ge 0$, $\Tilde{\pi}_{\Tilde{T}}$ is also an $\epsilon$-PO policy. The total number of policy evaluations is $2NT=\mathcal{O}[\epsilon^{-4}\log(\eta^{-1}\epsilon^{-1})]$. \end{thm}

\noindent\textbf{Comparison with Existing Works: } 
Theorem \ref{thm:0ppg_rate} indicates that our 0-FW algorithm for the first time converges to the desire PO policy with arbitrarily small precision $\epsilon$ in polynomial computation complexity, under the regularizer dominance condition that $\mu\ge 0$. In contrast, existing works only converge to a suboptimal PS policy under a similar regularizer dominance condition \citep{mandal2023performative,rank2024performative,mandal2024performative,pollatos2025corruption}. Our preferable convergence result is due to the main algorithmic difference that existing works use repeated retraining algorithms with iteration $\pi_{t+1}\!\approx\! {\arg\max}_{\pi\in\Pi}V_{\lambda,\pi}^{\pi_t}$ where the policy $\pi$ is deployed in a fixed environment $\mathcal{M}_{\pi_t}$ with $\pi\ne\pi_t$, while our 0-FW algorithm evaluates $V_{\lambda,\pi}^{\pi}$ where $\pi$ is always deployed at its corresponding environment $\mathcal{M}_{\pi}$. 

\noindent\textbf{Intuition and Novelty for Proving Theorem \ref{thm:0ppg_rate}: } Standard convergence analysis of Frank-Wolfe algorithm yields that $\max_{\Tilde{\pi}\in\Pi_{\Delta}}\langle\nabla_{\pi} V_{\lambda,\pi_{\Tilde{T}}}^{\pi_{\Tilde{T}}},\Tilde{\pi}-\pi_{\Tilde{T}}\rangle\le \frac{D\epsilon}{2}$ on $\Pi_{\Delta}$. However, it requires a trick to prove the following Proposition \ref{prop:2grad_inprods} which implies that $\pi_{\Tilde{T}}$ is $D\epsilon$-stationary on $\Pi$. 
\begin{proposition}\label{prop:2grad_inprods}
If $\Delta\le \pi_{\min}/3$ and a policy $\pi$ satisfies $\max_{\Tilde{\pi}\in\Pi_{\Delta}}\langle\nabla_{\pi} V_{\lambda,\pi}^{\pi},\Tilde{\pi}-\pi\rangle\le \frac{D\lambda}{5|\mathcal{A}|(1-\gamma)}$, then the stationary measures on $\Pi_{\Delta}$ and $\Pi$ bound each other as follows. 
\begin{align}\label{eq:2grad_inprods}
\max_{\Tilde{\pi}\in\Pi}\langle\nabla_{\pi} V_{\lambda,\pi}^{\pi},\Tilde{\pi}-\pi\rangle \le& 2\max_{\Tilde{\pi}\in\Pi_{\Delta}}\langle\nabla_{\pi} V_{\lambda,\pi}^{\pi},\Tilde{\pi}-\pi\rangle
\end{align}
\end{proposition}
To prove Proposition \ref{prop:2grad_inprods}, note that $\pi'$ defined by Eq. (\ref{eq:pi_pie}) also belongs to $\Pi_{\Delta}$, so Theorem \ref{thm:pi_ge} implies $\pi(a|s)\ge 2\Delta$. Then for any $\pi_2\in\Pi$, we have $\frac{\pi_2+\pi}{2}\in\Pi_{\Delta}$ and thus 
\begin{align}
\max_{\pi_2\in\Pi}\langle\nabla_{\pi} V_{\lambda,\pi}^{\pi},\pi_2\!-\!\pi\rangle
=2\max_{\pi_2\in\Pi}\Big\langle\nabla_{\pi} V_{\lambda,\pi}^{\pi},\frac{\pi_2+\pi}{2}\!-\!\pi\Big\rangle \le 2\max_{\Tilde{\pi}\in\Pi_{\Delta}}\langle\nabla_{\pi} V_{\lambda,\pi}^{\pi},\Tilde{\pi}-\pi\rangle.\nonumber
\end{align}

\section{Experiments}
We compare our Algorithm \ref{alg:0ppg} with the existing repeated retraining algorithm in a simulation environment. See Appendix \ref{sec:experiment} for the implementation details. 
Then for the policies $\pi_t$ obtained by each algorithm, we plot the training curves of the performative value function $V_{\lambda,\pi_t}^{\pi_t}$ ($\lambda=0.5$) and the unregularized performative value function $V_{0,\pi_t}^{\pi_t}$ in Figure \ref{fig1} in Appendix \ref{sec:experiment}, which show that our Algorithm \ref{alg:0ppg} converges better than the existing repeated retraining algorithm on both regularized and unregularized performative value functions. 

\section{Conclusion}\label{sec:conclusion}
We have studied an entropy-regularized performative reinforcement learning problem, obtained its important properties including gradient dominance, policy lower bound, Lipschitz continuity and smoothness. Based on these properties, we have proposed a zeroth-order Frank-Wolfe (0-FW) algorithm only using sample-based policy evaluation, which for the first time converges to a \textit{performatively optimal (PO)} policy with polynomial number of policy evaluations under the regularizer dominance condition. These theoretical results also holds for the quadratice regularizers used in the existing works on performative reinforcement learning (see Appendix \ref{sec:QuadReg} for discussion). 


\bibliographystyle{iclr2026_conference}
\bibliography{./GS}

\newpage
\onecolumn
\appendix

\addcontentsline{toc}{section}{Appendix} 
\part{Appendix} 
\parttoc 
\section{Related Works}
\textbf{Non-stationary Reinforcement Learning:} The performative reinforcement learning studied in this work relates to some non-stationary reinforcement learning. For example, \cite{gajane2018sliding,fei2020dynamic,cheung2020reinforcement,wei2021non,domingues2021kernel} provide theoretical results assuming that the non-stationary environment (rewards and transitions) change in a bounded amount or number, and \cite{even2004experts,dekel2013better,rosenberg2019online} study reinforcement learning with adversarial reward functions. 

\textbf{Performative Prediction:} Performative prediction proposed by \citep{perdomo2020performative} is a stochastic optimization framework where the data distribution depends on the decision policy. Compared with performative prediction, performative reinforcement learning is similar but more complex due to the policy-dependent transition dynamics. 

Various algorithms have been obtained with finite-time convergence to various solutions of performative prediction. For example, \cite{mendler2020stochastic,brown2022performative,li2022state} converge to a performatively stable solution that approximates the performatively optimal solution (the primary goal). \cite{izzo2021learn,roy2022constrained,haitong2024two} converge to a stationary point of the nonconvex performative prediction objective. 
\cite{miller2021outside,ray2022decision} converge to the performatively optimal solution (the primary goal), which relies on the strong assumptions that the loss function is strongly convex with degree dominating the distribution shift and that the data distribution satisfies mixture dominance condition or belongs to a location-scale family, such that the objective function becomes convex as proved by \citep{miller2021outside}. In contrast, we have proved an analogous result that the objective of performative reinforcement learning (harder than performative prediction) is gradient dominant (see our Theorem \ref{thm:ToOpt}) without these strong assumptions. In particular, our condition of regularizer dominating the environmental shift is analogous to their condition of strong convexity dominating the distribution shift, but our value function still remains nonconvex which is more challenging than their strongly convex losses. 

A survey of performative prediction can be seen in \citep{hardt2023performative}. 

\section{Experimental Details and Results}\label{sec:experiment}
We compare our Algorithm \ref{alg:0ppg} with the existing repeated retraining algorithm in a simulation environment with 5 states, 4 actions, discount factor $\gamma=0.95$, entropy regularizer coefficient $\lambda=0.5$, as well as transition kernel $p_{\pi}(s'|s,a)=\frac{\pi(a|s)+\pi(a|s')+1}{\sum_{s''}[\pi(a|s)+\pi(a|s'')+1]}$ and reward $r_{\pi}(s,a)=\pi(a|s)$ that depend on the policy $\pi$. We implement our Algorithm 1 for 401 iterations with $N=1000$, $\beta=0.01$, $\Delta=10^{-3}$, $\delta=10^{-4}$ and the performative value functions are evaluated by value iteration. The repeated retraining algorithm obtains the next policy $\pi_{t+1}$ by applying the natural policy gradient algorithm \citep{cen2022fast} with 401 iterations and stepsize 0.01 to the entropy-regularized reinforcement learning with transition kernel $p_{\pi_t}$ and reward $r_{\pi_t}$. Both algorithms start from the uniform policy (i.e. $\pi_0(a|s)\equiv 1/4$). The experiment is implemented on Python 3.9, using Apple M1 Pro with 8 cores and 16 GB memory, which costs about 110 minutes in total. 
Then for the policies $\{\pi_t\}_{t=0}^{400}$ obtained by each algorithm, we plot the training curves of the performative value function $V_{\lambda,\pi_t}^{\pi_t}$ (defined by Eq. (\ref{eq:Vfunc}) with $\lambda=0.5$) and the unregularized performative value function $V_{0,\pi_t}^{\pi_t}$ (defined by Eq. (\ref{eq:Vfunc}) with $\lambda=0$) on the left and right side of Figure \ref{fig1} respectively, which show that the existing repeated retraining algorithm stucks at the initial uniform policy $\pi_0$ since $\pi_0$ is a performatively stable (PS) policy, while our Algorithm \ref{alg:0ppg} converges well on both regularized and unregularized performative value functions in a similar pattern. 

\begin{figure*}[t]
\begin{minipage}{.5\textwidth}
    \centering
    \includegraphics[width=\textwidth]{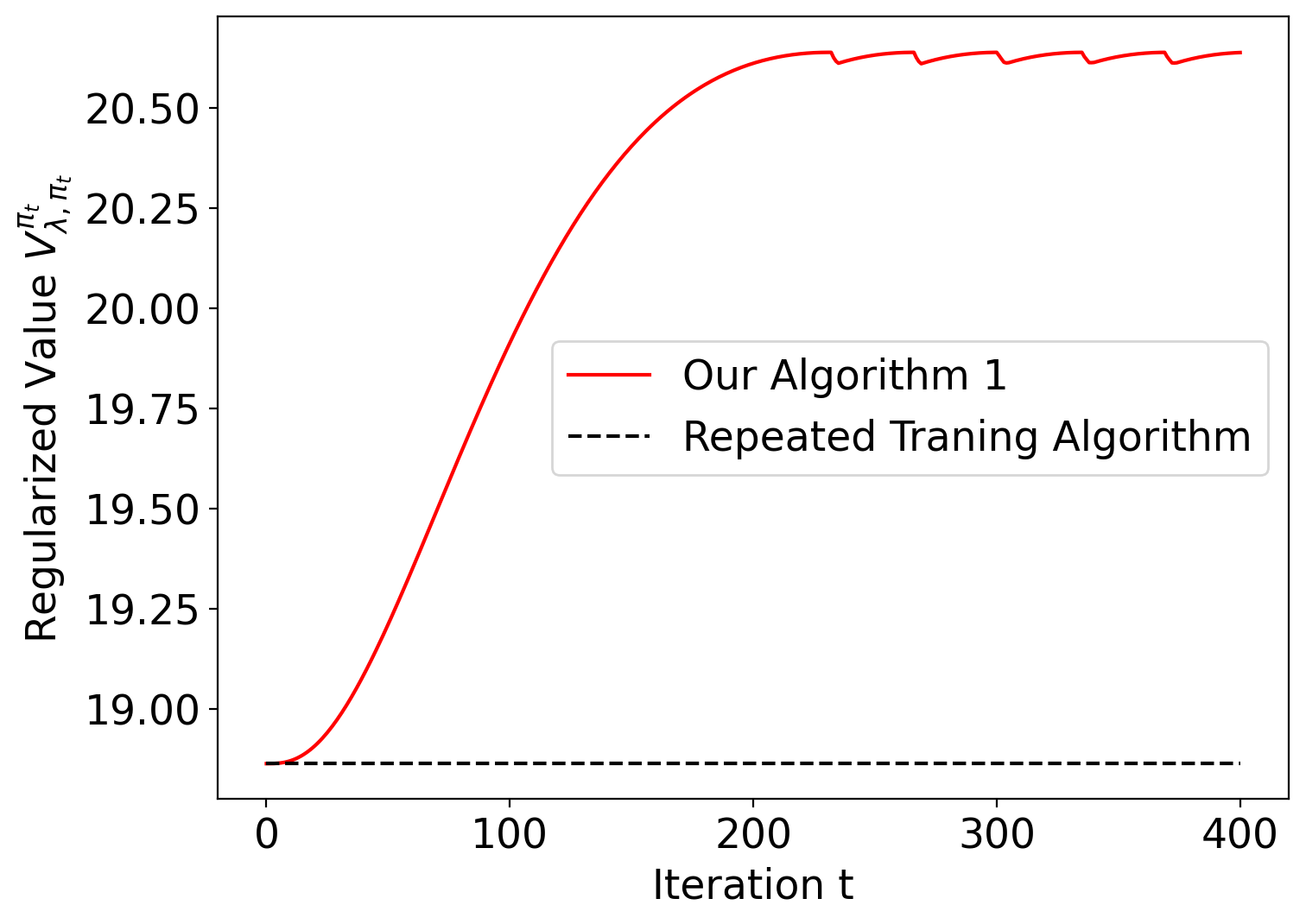}  
\end{minipage} 
\begin{minipage}{.47\textwidth}
    \centering
    \includegraphics[width=\textwidth]{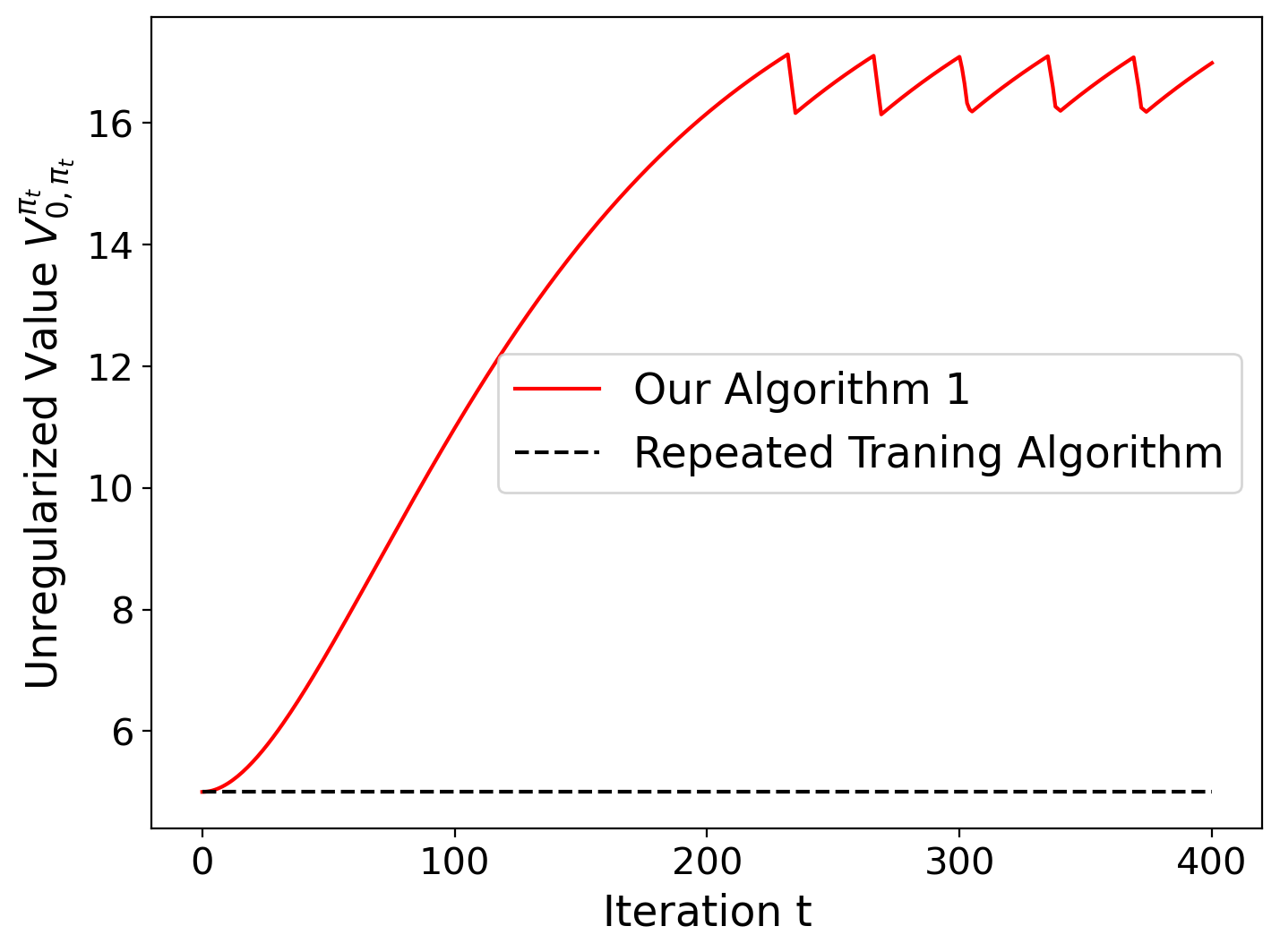}
\end{minipage} 
\caption{Experimental Results.}
\label{fig1}
\vspace{-10pt}
\end{figure*}



\section{Supporting Lemmas}
\subsection{Frank-Wolfe Step}\label{sec:FW}
We repeat Lemma \ref{lemma:wolfe} as follows.
\begin{lemma}\label{lemma:wolfe2}
The step (\ref{eq:pi_wolfe}) has the following analytical solution.
\begin{align}
\Tilde{\pi}_t(a|s)=\left\{
\begin{aligned}
&\Delta; a\ne\Tilde{a}_t(s)\\
&1-\Delta(|\mathcal{A}|-1); a=\Tilde{a}_t(s)
\end{aligned}
\right.,\label{eq:pi_wolfe_sol2}
\end{align}
where $\Tilde{a}_t(s)\in {\arg\max}_{a}\hat{g}_{\lambda,\delta}(\pi_t)(a|s)$. 
\end{lemma}
\begin{proof}
For $\Tilde{\pi}_t$ defined by Eq. (\ref{eq:pi_wolfe_sol2}) and for any $\pi\in\Pi_{\Delta}$, we have
\begin{align}
&\langle \Tilde{\pi}_t-\pi,\hat{g}_{\lambda,\delta}(\pi_t)\rangle\nonumber\\
=&\sum_{s,a}\hat{g}_{\lambda,\delta}(\pi_t)(a|s)[\Tilde{\pi}_t(a|s)-\pi(a|s)]\nonumber\\
=&\sum_s\Big\{\hat{g}_{\lambda,\delta}(\pi_t)[\Tilde{a}_t(s)|s]\big[1-\Delta(|\mathcal{A}|-1)-\pi[\Tilde{a}_t(s)|s]\big]-\sum_{a\ne\Tilde{a}_t(s)}\hat{g}_{\lambda,\delta}(\pi_t)(a|s)[\pi(a|s)-\Delta]\Big\}\nonumber\\
\overset{(a)}{\ge}&\sum_s\Big\{\hat{g}_{\lambda,\delta}(\pi_t)[\Tilde{a}_t(s)|s]\big[1-\Delta(|\mathcal{A}|-1)-\pi[\Tilde{a}_t(s)|s]\big]\nonumber\\
&-\sum_{a\ne\Tilde{a}_t(s)}\hat{g}_{\lambda,\delta}(\pi_t)[\Tilde{a}_t(s)|s][\pi(a|s)-\Delta]\Big\}\nonumber\\
=&\sum_s\Big\{\hat{g}_{\lambda,\delta}(\pi_t)[\Tilde{a}_t(s)|s]\big[1-\Delta(|\mathcal{A}|-1)-\pi[\Tilde{a}_t(s)|s]\big]\nonumber\\
&-\hat{g}_{\lambda,\delta}(\pi_t)[\Tilde{a}_t(s)|s]\big[1-\pi[\Tilde{a}_t(s)|s]-\Delta(|\mathcal{A}|-1)\big]\Big\}\nonumber\\
=&0,\nonumber
\end{align}
where (a) uses $\pi(a|s)-\Delta\ge 0$ and $\hat{g}_{\lambda,\delta}(\pi_t)(a|s)\le \hat{g}_{\lambda,\delta}(\pi_t)[\Tilde{a}_t(s)|s]$. Therefore, Eq. (\ref{eq:pi_wolfe}) holds, that is, $\Tilde{\pi}_t={\arg\max}_{\pi\in\Pi_{\Delta}}\langle \pi,\hat{g}_{\lambda,\delta}(\pi_t)\rangle$. 
\end{proof}

\subsection{Lipschitz Property of Occupany Measure}
\begin{lemma}\label{lemma:occup_Lip}
The occupancy measure $d_{\pi,p}$ defined by Eq. (\ref{eq:occup}) has the following Lipschitz properties for any $\pi,\pi'\in\Pi$, $p,p'\in\mathcal{P}$ and $\Tilde{s}\in\mathcal{S}$. 
\begin{align}
    \sum_s|d_{\pi',p}(s)-d_{\pi,p}(s)|&\le \frac{\gamma}{1-\gamma}\max_{s}\|\pi'(\cdot|s)-\pi(\cdot|s)\|_1 \le \frac{\gamma\sqrt{|\mathcal{A}|}}{1-\gamma}\|\pi'-\pi\| \label{eq:visit_dpi}\\
    \sum_s|d_{\pi,p'}(s)-d_{\pi,p}(s)|&\le \frac{\gamma}{1-\gamma}\max_{s,a} \|p'(\cdot|s,a)-p(\cdot|s,a)\|_1\le \frac{\gamma\sqrt{|\mathcal{S}|}}{1-\gamma}\|p'-p\|\label{eq:visit_dp}\\
    \sum_{s,a}|d_{\pi',p'}(s,a)-d_{\pi,p}(s,a)|&\le \frac{1}{1-\gamma}\max_{s}\|\pi'(\cdot|s)-\pi(\cdot|s)\|_1\!+\!\frac{\gamma}{1\!-\!\gamma}\!\max_{s,a}\!\|p'(\cdot|s,a)\!-\!p(\cdot|s,a)\|_1\nonumber\\
    &\le\frac{\sqrt{|\mathcal{A}|}}{1-\gamma}\|\pi'-\pi\|
    +\frac{\gamma\sqrt{|\mathcal{S}|}}{1-\gamma}\|p'-p\|\label{eq:visit_dsa}
\end{align}
\end{lemma}
\begin{proof}
The first $\le$ of Eqs. (\ref{eq:visit_dpi}) and (\ref{eq:visit_dp}) follows from Lemma 5 of \citep{chen2024acc}. The second $\le$ of Eqs. (\ref{eq:visit_dpi}) and (\ref{eq:visit_dp}) uses $\|x\|_1\le \sqrt{d}\|x\|$ for any $x\in\mathbb{R}^d$. 

Eq. (\ref{eq:visit_dsa}) can be proved as follows.
\begin{align}
&\sum_{s,a}|d_{\pi',p'}(s,a)-d_{\pi,p}(s,a)|\nonumber\\
=&\sum_{s,a}|d_{\pi',p'}(s)\pi'(a|s)-d_{\pi,p}(s)\pi(a|s)|\nonumber\\
\le&\sum_{s,a}d_{\pi',p'}(s)|\pi'(a|s)-\pi(a|s)|+\pi(a|s)|d_{\pi',p'}(s)-d_{\pi,p}(s)|\nonumber\\
\le&\sum_s [d_{\pi',p'}(s)\max_{s'}\|\pi'(\cdot|s')-\pi(\cdot|s')\|_1] +\sum_s|d_{\pi',p'}(s)-d_{\pi,p}(s)|\nonumber\\
\overset{(a)}{\le}&\max_{s'}\|\pi'(\cdot|s')-\pi(\cdot|s')\|_1\!+\!\frac{\gamma}{1-\gamma}\max_{s}\|\pi'(\cdot|s)\!-\!\pi(\cdot|s)\|_1\!+\!\frac{\gamma}{1\!-\!\gamma}\max_{s,a} \|p'(\cdot|s,a)-p(\cdot|s,a)\|_1\nonumber\\
\le&\frac{1}{1-\gamma}\max_{s}\|\pi'(\cdot|s)-\pi(\cdot|s)\|_1+\frac{\gamma}{1-\gamma}\max_{s,a} \|p'(\cdot|s,a)-p(\cdot|s,a)\|_1\nonumber\\
\le&\frac{\sqrt{|\mathcal{A}|}}{1-\gamma}\|\pi'-\pi\|+\frac{\gamma\sqrt{|\mathcal{S}|}}{1-\gamma}\|p'-p\|,\nonumber
\end{align}
where (a) uses Eqs. (\ref{eq:visit_dpi}) and (\ref{eq:visit_dp}). 
\end{proof}

\subsection{Various Value Functions}
Define the following value functions.
\begin{align}
J_{\lambda}(\pi,\pi',p,r)\overset{\rm def}{=}&\mathbb{E}_{\pi,p}\Big[\sum_{t=0}^{\infty}\gamma^t[r(s_t,a_t)-\lambda\log\pi'(a_t|s_t)]\Big|s_0\sim\rho\Big]\nonumber\\
=&\frac{1}{1-\gamma}\sum_{s,a}d_{\pi,p}(s,a)[r(s,a)-\lambda\log\pi'(a|s)],\label{eq:Jtau}\\
V_{\lambda}(\pi,\pi',p,r;s)\overset{\rm def}{=}&\mathbb{E}_{\pi,p}\Big[\sum_{t=0}^{\infty}\gamma^t[r(s_t,a_t)-\lambda\log\pi'(a_t|s_t)]\Big|s_0=s\Big],\label{eq:Vtau}\\
Q_{\lambda}(\pi,\pi',p,r;s,a)\overset{\rm def}{=}&\mathbb{E}_{\pi,p}\Big[\sum_{t=0}^{\infty}\gamma^t[r(s_t,a_t)-\lambda\log\pi'(a_t|s_t)]\Big|s_0=s,a_0=a\Big]\nonumber\\
=&r(s,a)-\lambda\log\pi'(a|s)+\gamma\sum_{s'}p(s'|s,a)V_{\lambda}(\pi,\pi',p,r;s').\label{eq:Qtau}
\end{align}
Note that the value function (\ref{eq:Vfunc}) of interest can be rewritten into the above functions as follows.
\begin{align}
V_{\lambda,\pi'}^{\pi}=&J_{\lambda}(\pi,\pi,p_{\pi'},r_{\pi'})\nonumber\\
=&\sum_s \rho(s)V_{\lambda}(\pi,\pi,p_{\pi'},r_{\pi'};s)\nonumber\\
=&\sum_{s,a}\rho(s)\pi(a|s)Q_{\lambda}(\pi,\pi,p_{\pi'},r_{\pi'};s,a).\nonumber
\end{align}
Hence, we will investigate the properties of the value functions (\ref{eq:Jtau})-(\ref{eq:Qtau}) as follows. 
\begin{lemma}\label{lemma:Jrange}
For any $\pi\in\Pi$, $p\in\mathcal{P}$, $r\in\mathcal{R}$, we have $V_{\lambda,\pi}^{\pi},J_{\lambda}(\pi,\pi,p,r),V_{\lambda}(\pi,\pi,p,r;s),Q_{\lambda}(\pi,\pi,p,r;s,a)\in \Big[0,\frac{1+\lambda\log|\mathcal{A}|}{1-\gamma}\Big]$. 
\end{lemma}
\begin{proof}
We will prove the range of $J_{\lambda}(\pi,\pi,p,r)$ as follows using $r(s,a)\in[0,1]$. The proof for the other value functions follow the same way.
\begin{align}
0\le J_{\lambda}(\pi,\pi,p,r)=&\mathbb{E}_{\pi,p,\rho}\Big[\sum_{t=0}^{\infty}\gamma^t[r(s_t,a_t)-\lambda\log\pi(a_t|s_t)]\Big]\nonumber\\
\le& \sum_{t=0}^{\infty}\gamma^t+\lambda\mathbb{E}_{\pi,p,\rho}\Big[\sum_{t=0}^{\infty}\gamma^t\sum_a[-\pi(a|s_t)\log\pi(a|s_t)]\Big]\nonumber\\
\le&\frac{1}{1-\gamma}+\lambda\sum_{t=0}^{\infty}\gamma^t\log|\mathcal{A}|\nonumber\\
\le&\frac{1+\lambda\log|\mathcal{A}|}{1-\gamma}.\nonumber
\end{align}
\end{proof}
\begin{lemma}\label{lemma:dp}
The gradients of $J_{\lambda}(\pi,\pi',p,r)$ defined by Eq. (\ref{eq:Jtau}) have the following expressions. 
\begin{align}
\frac{\partial J_{\lambda}(\pi,\pi',p,r)}{\partial\pi(a|s)}=&\frac{d_{\pi,p}(s)Q_{\lambda}(\pi,\pi',p,r;s,a)}{1-\gamma},\label{eq:dJ1}\\
\frac{\partial J_{\lambda}(\pi,\pi',p,r)}{\partial\pi'(a|s)}=&-\frac{\lambda d_{\pi,p}(s,a)}{(1-\gamma)\pi'(a|s)},\label{eq:dJ2}\\
\frac{\partial J_{\lambda}(\pi,\pi',p,r)}{\partial p(s'|s,a)}=&\frac{d_{\pi,p}(s,a)}{1-\gamma}\big[r(s,a)-\lambda\log\pi'(a|s)+\gamma V_{\lambda}(\pi,\pi',p,r;s')\big],\label{eq:dJ3}\\
\frac{\partial J_{\lambda}(\pi,\pi',p,r)}{\partial r(s,a)}=&\frac{d_{\pi,p}(s,a)}{1-\gamma},\label{eq:dJ4}\\
\frac{\partial J_{\lambda}(\pi,\pi,p,r)}{\partial\pi(a|s)}=&\frac{d_{\pi,p}(s)[Q_{\lambda}(\pi,\pi,p,r;s,a)-\lambda]}{1-\gamma}.\label{eq:dJ5}
\end{align}
\end{lemma}
\begin{proof}
Eq. (\ref{eq:dJ1}) follows from the policy gradient expression in Eq. (7) of \citep{agarwal2021theory}, with reward function $r(s,a)$ replaced by $r(s,a)-\lambda\log\pi'(a|s)$. 

Eq. (\ref{eq:dJ3}) can be proved as follows. 
\begin{align}
p(s'|s,a)\overset{(a)}{=}&\frac{d_{\pi,p}(s)\pi(a|s)}{1-\gamma}\big[r(s,a)-\lambda\log\pi(a|s)+\gamma V_{\lambda}(\pi,\pi',p,r;s')\big]\nonumber\\
=&\frac{d_{\pi,p}(s,a)}{1-\gamma}\big[r(s,a)-\lambda\log\pi(a|s)+\gamma V_{\lambda}(\pi,\pi',p,r;s')\big],\nonumber
\end{align}
where (a) uses Eq. (9) in \citep{chen2024acc}.

Eqs. (\ref{eq:dJ2}) and (\ref{eq:dJ4}) can be proved by taking derivatives of Eq. (\ref{eq:Jtau}). 

Based on the chain rule, Eq. (\ref{eq:dJ5}) can be proved as follows by adding Eqs. (\ref{eq:dJ1}) and (\ref{eq:dJ2}) with $\pi'=\pi$. 
\begin{align}
\frac{\partial J_{\lambda}(\pi,\pi,p,r)}{\partial\pi(a|s)}=&\Big[\frac{\partial J_{\lambda}(\pi,\pi',p,r)}{\partial\pi(a|s)}+\frac{\partial J_{\lambda}(\pi,\pi',p,r)}{\partial\pi'(a|s)}\Big]\Big|_{\pi'=\pi}\nonumber\\
=&\frac{d_{\pi,p}(s)Q_{\lambda}(\pi,\pi,p,r;s,a)}{1-\gamma}-\frac{\lambda d_{\pi,p}(s,a)}{(1-\gamma)\pi(a|s)}\nonumber\\
=&\frac{d_{\pi,p}(s)[Q_{\lambda}(\pi,\pi,p,r;s,a)-\lambda]}{1-\gamma},\nonumber
\end{align}
where the final $=$ uses $d_{\pi,p}(s,a)=d_{\pi,p}(s)\pi(a|s)$. 
\end{proof}

\begin{lemma}\label{lemma:J_lip}
The function $J_{\lambda}$ defined by Eq. (\ref{eq:Jtau}) has the following Lipschitz properties for any $\pi,\pi'\in\Pi$, $p,p'\in\mathcal{P}$ and $r,r'\in\mathcal{R}$. 
\begin{align}
    |J_{\lambda}(\pi',\pi',p,r)-J_{\lambda}(\pi,\pi,p,r)| &\le L_{\pi}\max_{s}\|\log\pi'(\cdot|s)-\log\pi(\cdot|s)\| \label{eq:Lpi}\\
    |J_{\lambda}(\pi,\pi,p',r)-J_{\lambda}(\pi,\pi,p,r)| &\le L_p\|p'-p\|\label{eq:Lp}\\
    |J_{\lambda}(\pi,\pi,p,r')-J_{\lambda}(\pi,\pi,p,r)|&\le \frac{\|r'-r\|_{\infty}}{1-\gamma}\le \frac{\|r'-r\|}{1-\gamma}\label{eq:Lr}\\
    \|\nabla_p J_{\lambda}(\pi',\pi',p,r)-\nabla_p J_{\lambda}(\pi,\pi,p,r)\| &\le \ell_{\pi}\max_{s}\|\log\pi'(\cdot|s)-\log\pi(\cdot|s)\|\label{eq:lpi}\\
    \|\nabla_p J_{\lambda}(\pi,\pi,p',r)-\nabla_p J_{\lambda}(\pi,\pi,p,r)\| &\le \ell_p\|p'-p\|\label{eq:lp}
\end{align}
\begin{align}
    &\|\nabla_p J_{\lambda}(\pi',\pi',p',r')-\nabla_p J_{\lambda}(\pi,\pi,p,r)\| \nonumber\\
    \le& \ell_{\pi}\max_{s}\!\|\!\log\pi'(\cdot|s)\!-\!\log\pi(\cdot|s)\|\!+\!\ell_{p}\|p'\!-\!p\|\!+\!\frac{\sqrt{|\mathcal{S}|}}{(1-\gamma)^2}\|r'\!-\!r\|_{\infty}\label{eq:lp_all}\\
    &\|\nabla_r J_{\lambda}(\pi',\pi',p',r')-\nabla_r J_{\lambda}(\pi,\pi,p,r)\| \nonumber\\
    \le&\frac{\max_{s}\|\pi'(\cdot|s)-\pi(\cdot|s)\|_1+\gamma\max_{s,a} \|p'(\cdot|s,a)-p(\cdot|s,a)\|_1}{(1-\gamma)^2}\label{eq:lr_all}\\
    &\|\nabla_{\pi} J_{\lambda}(\pi',\pi',p',r')-\nabla_{\pi} J_{\lambda}(\pi,\pi,p,r)\|\nonumber\\
    \le&\Big(\frac{|\mathcal{A}|(1+2\lambda\log|\mathcal{A}|)}{(1-\gamma)^2}+\gamma L_{\pi}\Big)\max_{s}\|\log\pi'(\cdot|s)-\log\pi(\cdot|s)\|\nonumber\\
    &+\gamma\sqrt{|\mathcal{A}|}\Big[\frac{2\sqrt{|\mathcal{S}|}(1+\lambda\log|\mathcal{A}|)}{(1-\gamma)^2}+ L_p\Big]\|p'-p\|+\frac{\sqrt{|\mathcal{A}|}\|r'-r\|_{\infty}}{1-\gamma},\label{eq:Lip_pi}
\end{align}
where $L_{\pi}:=\frac{\sqrt{|\mathcal{A}|}(2-\gamma+\gamma\lambda\log|\mathcal{A}|)}{(1-\gamma)^2}$, $L_p:=\frac{\sqrt{|\mathcal{S}|}(1+\lambda\log|\mathcal{A}|)}{(1-\gamma)^2}$, $\ell_{\pi}:=\frac{\sqrt{|\mathcal{S}||\mathcal{A}|}(2+3\gamma\lambda\log|\mathcal{A}|)}{(1-\gamma)^3}$ and $\ell_p:=\frac{2\gamma|\mathcal{S}|(1+\lambda\log|\mathcal{A}|)}{(1-\gamma)^3}$. 
\end{lemma}
\begin{proof}
Eqs. (\ref{eq:Lpi}), (\ref{eq:Lp}), (\ref{eq:lpi}) and (\ref{eq:lp}) directly follow from Lemma 6 of \citep{chen2024acc}. Eq. (\ref{eq:Lr}) can be proved as follows.
\begin{align}
|J_{\lambda}(\pi,p,r')-J_{\lambda}(\pi,p,r)|=&\Big|\frac{1}{1-\gamma}\sum_{s,a}d_{\pi,p}(s,a)[r'(s,a)-r(s,a)]\Big|\nonumber\\
\le&\frac{1}{1-\gamma}\sum_{s,a}d_{\pi,p}(s,a)|r'(s,a)-r(s,a)|\nonumber\\
=&\frac{1}{1-\gamma}\sum_{s,a}d_{\pi,p}(s,a)\|r'-r\|_{\infty} \nonumber\\
=& \frac{1}{1-\gamma}\|r'-r\|_{\infty}\le \frac{1}{1-\gamma}\|r'-r\|.\nonumber
\end{align}
To prove Eq. (\ref{eq:lp_all}), note that
\begin{align}
&\Big|\frac{\partial J_{\lambda}(\pi,\pi,p,r')}{\partial p(s'|s,a)}-\frac{\partial J_{\lambda}(\pi,\pi,p,r)}{\partial p(s'|s,a)}\Big|\nonumber\\
\overset{(a)}{=}&\frac{d_{\pi,p}(s,a)}{1-\gamma}\big|r'(s,a)-r(s,a)+\gamma [V_{\lambda}(\pi,\pi',p,r';s')-V_{\lambda}(\pi,\pi',p,r;s')]\big|\nonumber\\
\overset{(b)}{\le}& \frac{d_{\pi,p}(s,a)}{1-\gamma}\Big[\|r'-r\|_{\infty}+\gamma\sum_{t=0}^{\infty}\gamma^t\|r'-r\|_{\infty}\Big]\nonumber\\
\le& \frac{d_{\pi,p}(s,a)}{(1-\gamma)^2}\|r'-r\|_{\infty}\label{eq:dJp_r}
\end{align}
where (a) uses Eq. (\ref{eq:dJ3}) and (b) uses Eq. (\ref{eq:Vtau}). Therefore, we can prove Eq. (\ref{eq:lp_all}) as follows.
\begin{align}
&\|\nabla_p J_{\lambda}(\pi',\pi',p',r')-\nabla_p J_{\lambda}(\pi,\pi,p,r)\|\nonumber\\
\le&\|\nabla_p J_{\lambda}(\pi',\pi',p',r')-\nabla_p J_{\lambda}(\pi,\pi,p',r')\|+\|\nabla_p J_{\lambda}(\pi,\pi,p',r')-\nabla_p J_{\lambda}(\pi,\pi,p,r')\|\nonumber\\
&+\|\nabla_p J_{\lambda}(\pi,\pi,p,r')-\nabla_p J_{\lambda}(\pi,\pi,p,r)\|\nonumber\\
\overset{(a)}{\le}&\ell_{\pi}\max_{s}\|\log\pi'(\cdot|s)-\log\pi(\cdot|s)\|\!+\!\ell_p\|p'-p\|\!+\!\sqrt{\sum_{s,a,s'}\Big|\frac{\partial J_{\lambda}(\pi,\pi,p,r')}{\partial p(s'|s,a)}\!-\!\frac{\partial J_{\lambda}(\pi,\pi,p,r)}{\partial p(s'|s,a)}\Big|^2}\nonumber\\
\overset{(b)}{\le}&\ell_{\pi}\max_{s}\|\log\pi'(\cdot|s)-\log\pi(\cdot|s)\|+\ell_p\|p'-p\|+\sqrt{\frac{\|r'-r\|_{\infty}^2}{(1-\gamma)^4}\sum_{s,a,s'}d_{\pi,p}^2(s,a)}\nonumber\\
\le&\ell_{\pi}\max_{s}\|\log\pi'(\cdot|s)-\log\pi(\cdot|s)\|+\ell_p\|p'-p\|+\frac{\sqrt{|\mathcal{S}|}}{(1-\gamma)^2}\|r'-r\|_{\infty},\nonumber
\end{align} 
where (a) uses Eqs. (\ref{eq:lpi}) and (\ref{eq:lp}) and (b) uses Eq. (\ref{eq:dJp_r}). 

Then, we prove Eq. (\ref{eq:lr_all}) as follows.
\begin{align}
&\|\nabla_r J_{\lambda}(\pi',\pi',p',r')-\nabla_r J_{\lambda}(\pi,\pi,p,r)\|\nonumber\\
\overset{(a)}{=}&\frac{\|d_{\pi',p'}-d_{\pi,p}\|}{1-\gamma}\nonumber\\
\le&\frac{\|d_{\pi',p'}-d_{\pi,p}\|_1}{1-\gamma}\nonumber\\
\overset{(b)}{\le}&\frac{1}{(1-\gamma)^2}\max_{s}\|\pi'(\cdot|s)-\pi(\cdot|s)\|_1+\frac{\gamma}{(1-\gamma)^2}\max_{s,a} \|p'(\cdot|s,a)-p(\cdot|s,a)\|_1,\nonumber
\end{align}
where (a) uses Eq. (\ref{eq:dJ4}), (b) uses Eq. (\ref{eq:visit_dsa}). 

To prove Eq. (\ref{eq:Lip_pi}), we will first prove the following auxiliary bounds. 
\begin{align}
Q_{\lambda}(\pi,\pi,p,r;s,a)\!-\!\lambda\!\overset{(a)}{\in}\!\Big[\!-\!\lambda,\frac{1\!+\!\lambda\log|\mathcal{A}|}{1-\gamma}\!-\!\lambda\Big]\!\Rightarrow\!\big|Q_{\lambda}(\pi,\pi,p,r;s,a)\!-\!\lambda\big|\!\le\!\frac{1\!+\!\lambda\log|\mathcal{A}|}{1-\gamma},\label{eq:QMtau_range}
\end{align}    
where (a) uses Lemma \ref{lemma:Jrange}. 
\begin{align}
&|V_{\lambda}(\pi',\pi',p',r';s)-V_{\lambda}(\pi,\pi,p,r;s)|\nonumber\\
\le&|V_{\lambda}(\pi',\pi',p',r';s)\!-\!V_{\lambda}(\pi,\pi,p',r';s)|\!+\!|V_{\lambda}(\pi,\pi,p',r';s)\!-\!V_{\lambda}(\pi,\pi,p,r';s)|\nonumber\\
&+\!|V_{\lambda}(\pi,\pi,p,r';s)\!-\!V_{\lambda}(\pi,\pi,p,r;s)|\nonumber\\
\overset{(a)}{\le}&L_{\pi}\max_{s}\|\log\pi'(\cdot|s)-\log\pi(\cdot|s)\|+L_p\|p'-p\|+\frac{\|r'-r\|_{\infty}}{1-\gamma},\label{eq:Vdiff}
\end{align} 
where (a) applies Eqs. (\ref{eq:Lpi})-(\ref{eq:Lr}) to the case where the initial state distribution $\rho$ is probability 1 at $s$ (so $J_{\lambda}(\pi,\pi,p,r)$ becomes $V_{\lambda}(\pi,\pi,p,r;s)$). 
\begin{align}
&|Q_{\lambda}(\pi,\pi,p,r';s,a)-Q_{\lambda}(\pi,\pi,p,r;s,a)|\nonumber\\
\overset{(a)}{=}&\Big|\mathbb{E}_{\pi,p}\Big[\sum_{t=0}^{\infty}\gamma^t[r'(s_t,a_t)-r(s_t,a_t)]\Big|s_0=s,a_0=a\Big]\Big|\nonumber\\
\le&\mathbb{E}_{\pi,p}\Big[\sum_{t=0}^{\infty}\gamma^t[r'(s_t,a_t)-r(s_t,a_t)|\Big|s_0=s,a_0=a\Big]\nonumber\\
\le&\mathbb{E}_{\pi,p}\Big[\sum_{t=0}^{\infty}\gamma^t\|r'-r\|_{\infty}\Big|s_0=s,a_0=a\Big]\nonumber\\
\le&\frac{\|r'-r\|_{\infty}}{1-\gamma},\label{eq:Qr_diff}
\end{align}
where (a) uses Eq. (\ref{eq:Qtau}). 
\begin{align}
&|Q_{\lambda}(\pi',\pi',p',r;s,a)-Q_{\lambda}(\pi,\pi,p,r;s,a)|\nonumber\\
\overset{(a)}{\le}&\lambda|\log\pi'(a|s)-\log\pi(a|s)|+\gamma\Big|\sum_{s'}[p'(s'|s,a)V_{\lambda}(\pi',\pi',p',r;s)-p(s'|s,a)V_{\lambda}(\pi,\pi,p,r;s)]\Big|\nonumber\\
\le&\lambda|\log\pi'(a|s)-\log\pi(a|s)|+\gamma\sum_{s'}p'(s'|s,a)|V_{\lambda}(\pi',\pi',p',r;s)-V_{\lambda}(\pi,\pi,p,r;s)|\nonumber\\
&+\gamma\sum_{s'}|p'(s'|s,a)-p(s'|s,a)||V_{\lambda}(\pi,\pi,p,r;s)|\nonumber\\
\overset{(b)}{\le}&\lambda|\log\pi'(a|s)-\log\pi(a|s)|+\gamma L_{\pi}\max_{s'}\|\log\pi'(\cdot|s')-\log\pi(\cdot|s')\|+\gamma L_p\|p'-p\|\nonumber\\
&+\frac{\gamma(1+\lambda\log|\mathcal{A}|)}{1-\gamma}\|p'(\cdot|s,a)-p(\cdot|s,a)\|_1,\label{eq:Q_pip_diff}
\end{align}
where (a) uses Eq. (\ref{eq:Qtau}), and (b) uses Eq. (\ref{eq:Vdiff}) and Lemma \ref{lemma:Jrange}. 

Note that 
\begin{align}
&(1-\gamma)\Big|\frac{\partial J_{\lambda}(\pi',\pi',p',r')}{\partial\pi'(a|s)}-\frac{\partial J_{\lambda}(\pi,\pi,p,r)}{\partial\pi(a|s)}\Big|\nonumber\\
\overset{(a)}{=}&\big|d_{\pi',p'}(s)[Q_{\lambda}(\pi',\pi',p',r';s,a)-\lambda]-d_{\pi,p}(s)[Q_{\lambda}(\pi,\pi,p,r;s,a)-\lambda]\big|\nonumber\\
\le&\big|[d_{\pi',p'}(s)-d_{\pi,p}(s)][Q_{\lambda}(\pi',\pi',p',r';s,a)-\lambda]\nonumber\\
&+d_{\pi,p}(s)[Q_{\lambda}(\pi',\pi',p',r';s,a)-Q_{\lambda}(\pi',\pi',p',r;s,a)]\nonumber\\
&+d_{\pi,p}(s)[Q_{\lambda}(\pi',\pi',p',r;s,a)-Q_{\lambda}(\pi,\pi,p,r;s,a)]\big|\nonumber\\
\le&\big|d_{\pi',p'}(s)-d_{\pi,p}(s)\big|\cdot\big|Q_{\lambda}(\pi',\pi',p',r';s,a)-\lambda\big|\nonumber\\
&+d_{\pi,p}(s)\big|Q_{\lambda}(\pi',\pi',p',r';s,a)-Q_{\lambda}(\pi',\pi',p',r;s,a)\big|\nonumber\\
&+d_{\pi,p}(s)\big|Q_{\lambda}(\pi',\pi',p',r;s,a)-Q_{\lambda}(\pi,\pi,p,r;s,a)\big|\nonumber\\
\overset{(b)}{\le}&\frac{1+\lambda\log|\mathcal{A}|}{1-\gamma}\big|d_{\pi',p'}(s)-d_{\pi,p}(s)\big|+
\frac{d_{\pi,p}(s)\|r'-r\|_{\infty}}{1-\gamma}\nonumber\\
&+d_{\pi,p}(s)\Big[\lambda|\log\pi'(a|s)-\log\pi(a|s)|+\gamma L_{\pi}\max_{s'}\|\log\pi'(\cdot|s')-\log\pi(\cdot|s')\|\nonumber\\
&+\gamma L_p\|p'-p\|+\frac{\gamma(1+\lambda\log|\mathcal{A}|)}{1-\gamma}\|p'(\cdot|s,a)-p(\cdot|s,a)\|_1\Big],\nonumber
\end{align}
where (a) uses Eq. (\ref{eq:dJ5}), (b) uses Eqs. (\ref{eq:QMtau_range}), (\ref{eq:Qr_diff}) and (\ref{eq:Q_pip_diff}). Applying triangular inequality to the bound above, we can prove Eq. (\ref{eq:Lip_pi}) as follows. 
\begin{align}
&(1-\gamma)\big\|\nabla_{\pi'}J_{\lambda}(\pi',\pi',p',r')-\nabla_{\pi}J_{\lambda}(\pi,\pi,p,r)\big\|\nonumber\\
\le&\frac{1+\lambda\log|\mathcal{A}|}{1-\gamma}\sqrt{\sum_{s,a}\big|d_{\pi',p'}(s)-d_{\pi,p}(s)\big|^2}+
\frac{\|r'-r\|_{\infty}}{1-\gamma}\sqrt{\sum_{s,a}d_{\pi,p}(s)^2}\nonumber\\
&+\lambda\sqrt{\sum_{s,a}d_{\pi,p}(s)^2|\log\pi'(a|s)-\log\pi(a|s)|^2}\nonumber\\
&+\big[\gamma L_{\pi}\max_{s'}\|\log\pi'(\cdot|s')-\log\pi(\cdot|s')\|+\gamma L_p\|p'-p\|\big]\sqrt{\sum_{s,a}d_{\pi,p}(s)^2}\nonumber\\
&+\frac{\gamma(1+\lambda\log|\mathcal{A}|)}{1-\gamma}\sqrt{\sum_{s,a}d_{\pi,p}(s)^2\|p'(\cdot|s,a)-p(\cdot|s,a)\|_1^2}\nonumber\\
\le&\frac{\sqrt{|\mathcal{A}|}(1+\lambda\log|\mathcal{A}|)}{1-\gamma}\sum_s|d_{\pi',p'}(s)-d_{\pi,p}(s)| +\frac{\sqrt{|\mathcal{A}|}\|r'-r\|_{\infty}}{1-\gamma} \nonumber\\
&+\lambda\sqrt{\sum_s d_{\pi,p}(s)\|\log\pi'(\cdot|s)-\log\pi(\cdot|s)\|^2}\nonumber\\
&+\big[\gamma L_{\pi}\max_{s'}\|\log\pi'(\cdot|s')-\log\pi(\cdot|s')\|+\gamma L_p\|p'-p\|\big]\sqrt{|\mathcal{A}|}\nonumber\\
&+\frac{\gamma(1+\lambda\log|\mathcal{A}|)}{1-\gamma}\sqrt{|\mathcal{S}|\sum_{s,a}\|p'(\cdot|s,a)-p(\cdot|s,a)\|^2}\nonumber\\
\overset{(a)}{\le}&\frac{\gamma\sqrt{|\mathcal{A}|}(1+\lambda\log|\mathcal{A}|)}{(1-\gamma)^2}\big[\max_s\|\pi'(\cdot|s)-\pi(\cdot|s)\|_1+\max_{s,a} \|p'(\cdot|s,a)-p(\cdot|s,a)\|_1\big] \nonumber\\
&+\frac{\sqrt{|\mathcal{A}|}\|r'-r\|_{\infty}}{1-\gamma}+\lambda\max_{s'}\|\log\pi'(\cdot|s')-\log\pi(\cdot|s')\|\nonumber\\
&+\sqrt{|\mathcal{A}|}\big[\gamma L_{\pi}\max_{s'}\|\log\pi'(\cdot|s')-\log\pi(\cdot|s')\|+\gamma L_p\|p'-p\|\big]\nonumber\\
&+\frac{\gamma\sqrt{|\mathcal{S}|}(1+\lambda\log|\mathcal{A}|)}{1-\gamma}\|p'-p\|\nonumber\\
\overset{(b)}{\le}&\Big[\frac{|\mathcal{A}|(\gamma+2\lambda\log|\mathcal{A}|)}{(1-\gamma)^2}+\gamma L_{\pi}\Big]\max_{s'}\|\log\pi'(\cdot|s')-\log\pi(\cdot|s')\|\nonumber\\
&+\gamma\sqrt{|\mathcal{A}|}\Big[\frac{2\sqrt{|\mathcal{S}|}(1+\lambda\log|\mathcal{A}|)}{(1-\gamma)^2}+ L_p\Big]\|p'-p\|+\frac{\sqrt{|\mathcal{A}|}\|r'-r\|_{\infty}}{1-\gamma},\nonumber
\end{align}
where (a) uses Lemma \ref{lemma:occup_Lip}, (b) uses $\|\pi'(\cdot|s)-\pi(\cdot|s)\|_1\le\|\log\pi'(\cdot|s)-\log\pi(\cdot|s)\|_1$, \\
$\|p'(\cdot|s,a)-p(\cdot|s,a)\|_1\le \sqrt{|\mathcal{S}|}\|p'(\cdot|s,a)-p(\cdot|s,a)\|\le \sqrt{|\mathcal{S}|}\|p'-p\|$, $\frac{\gamma\sqrt{|\mathcal{S}|}(1+\lambda\log|\mathcal{A}|)}{1-\gamma}\le \frac{\sqrt{|\mathcal{S}||\mathcal{A}|}(1+\lambda\log|\mathcal{A}|)}{(1-\gamma)^2}$ and $\lambda\le\frac{\lambda|\mathcal{A}|\log|\mathcal{A}|}{(1-\gamma)^2}$. 
\end{proof}

\subsection{Zeroth-order Gradient Estimation Error}
We import Theorem 1.6.2 of \citep{tropp2015introduction} as follows. 
\begin{lemma}[Matrix Bernstein Inequality]\label{lemma:MatrixBernstein}
Suppose complex-valued matrices $S_1,\ldots,S_N\in\mathbb{C}^{d_1\times d_2}$ are independently distributed with $\mathbb{E}S_k=0$ and $\|S_k\|\le C$ for each $k=1,\ldots,N$. Denote the sum $Z_N=\sum_{k=1}^NS_k$ its variance statistic as follows
\begin{align}
v(Z_N)=\max\Big[\Big\|\sum_{k=1}^N\mathbb{E}(S_kS_k^*)\Big\|, \Big\|\sum_{k=1}^N\mathbb{E}(S_k^*S_k)\Big\|\Big], \label{eq:vZ}
\end{align}
where $S_k^*$ denotes the conjugate transpose of $S_k$. Then for any $\epsilon\ge 0$, we have
\begin{align}
\mathbb{P}\{\|Z_N\|\ge\epsilon\}\le (d_1+d_2)\exp\Big[\frac{-\epsilon^2/2}{v(Z_N)+C\epsilon/3}\Big].\label{eq:PZbig}
\end{align}
\end{lemma}
Applying the above lemma to vectors, we obtain the following vector Bernstein inequality. 
\begin{lemma}[Vector Bernstein Inequality]\label{lemma:VecBernstein}
Suppose independently distributed vectors $x_1,\ldots,x_N\in\mathbb{C}^{d}$ satisfies $\|x_k\|\le c$ for each $k=1,\ldots,N$. Then for any $\eta\in(0,1)$, with probability at least $1-\eta$, we have
\begin{align}
\Big\|\frac{1}{N}\sum_{k=1}^N (x_k-\mathbb{E}x_k)\Big\|<\frac{4c}{3N}\log\Big(\frac{d+1}{\eta}\Big)+2c\sqrt{\frac{2}{N}\log\Big(\frac{d+1}{\eta}\Big)}.\label{eq:VecBernstein}
\end{align}
\end{lemma}
\begin{proof}
Note that $S_k=x_k-\mathbb{E}x_k$ satisfies the conditions of Lemma \ref{lemma:MatrixBernstein} with $d_1=d$, $d_2=1$ and $C$ replaced by $2c$. In addition, $v(Z_N)$ defined by Eq. (\ref{eq:vZ}) satisfies $v(Z_N)\le 4Nc^2$ since
\begin{align}
&\max[\|S_kS_k^*\|,\|S_k^*S_k\|^2]\le \|S_k^*\|^2\|S_k\|^2 \le 4c^2.\nonumber
\end{align}
For any $\eta\in(0,1)$, let 
\begin{align}
\epsilon=\frac{4c}{3}\log\Big(\frac{d+1}{\eta}\Big)+c\sqrt{2N\log\Big(\frac{d+1}{\eta}\Big)}.\nonumber
\end{align}
Therefore, Lemma \ref{lemma:MatrixBernstein} implies that
\begin{align}
\mathbb{P}\Big\{\frac{1}{N}\Big\|\sum_{k=1}^N (x_k-\mathbb{E}x_k)\Big\|\ge\frac{\epsilon}{N}\Big\}\le (d+1)\exp\Big[\frac{-\epsilon^2/2}{4Nc^2+2c\epsilon/3}\Big]\le \eta,\nonumber
\end{align}
which implies that with probability at least $1-\eta$, we have
\begin{align}
\frac{1}{N}\Big\|\sum_{k=1}^N (x_k-\mathbb{E}x_k)\Big\|<\frac{\epsilon}{N}=\frac{4c}{3N}\log\Big(\frac{d+1}{\eta}\Big)+2c\sqrt{\frac{2}{N}\log\Big(\frac{d+1}{\eta}\Big)}.\nonumber
\end{align}
\end{proof}

For any function $f:\mathbb{R}^d\to\mathbb{R}$, obtain the following zeroth-order stochastic estimator of the gradient $\nabla f$. 
\begin{align}
g_{\delta}(x)=\frac{d}{2N\delta}\sum_{i=1}^N [f(x+\delta u_i)-f(x-\delta u_i)]u_i\approx \nabla f(x)\label{eq:gdelta}
\end{align}
where $\delta>0$ and $\{u_i\}_{i=1}^N$ are i.i.d. samples of the uniform distribution on the sphere $\mathbb{S}_d=\{u\in\mathbb{R}^d:\|u\|=1\}$.
\begin{lemma}\label{lemma:gerr}
Suppose $f:\mathbb{R}^d\to\mathbb{R}$ is an $L_f$-Lipschitz continuous and $\ell_f$-smooth function. Then for any $\eta\in(0,1)$, with probability at least $1-\eta$, the gradient estimator $g_{\delta}$ defined by Eq. (\ref{eq:gdelta}) has the following error bound. 
\begin{align}
\|g_{\delta}(x)-\nabla f(x)\|\le&\frac{4L_fd}{3N}\log\Big(\frac{d+1}{\eta}\Big)+2L_fd\sqrt{\frac{2}{N}\log\Big(\frac{d+1}{\eta}\Big)}+\delta\ell_f.\label{eq:gerr}
\end{align}
\end{lemma}
\begin{proof}
Note that $g_{\delta,i}(x)\overset{\rm def}{=}\frac{d}{2\delta}[f(x+\delta u_i)-f(x-\delta u_i)]u_i$ has the following norm bound
\begin{align}
\|g_{\delta,i}(x)\|\le \frac{d}{2\delta} \big|f(x+\delta u_i)-f(x-\delta u_i)\big|\cdot\|u_i\|\le \frac{d}{2\delta}\cdot L_f\|2\delta u_i\|=L_fd. \label{eq:gi_bound}
\end{align}

Define the following smoothed approximation of $f$ as follows.
\begin{align}
f_{\delta}(x)\overset{\rm def}{=}\mathbb{E}_{v\sim{\rm Unif}(\mathbb{B}_d)}[f(x+\delta v)],\label{eq:fdelta}
\end{align}
where ${\rm Unif}(\mathbb{B}_d)$ denotes the uniform distribution on the ball $\mathbb{B}_d\overset{\rm def}{=}\{u\in\mathbb{R}^d:\|u\|\le 1\}$. Then based on Lemma 1 of \citep{flaxman2005online}, we have
\begin{align}
\mathbb{E}[g_{\delta,i}(x)]=\nabla f_{\delta}(x)=\mathbb{E}_{v\sim{\rm Unif}(\mathbb{B}_d)}[\nabla f(x+\delta v)].\label{eq:g_unbiased}
\end{align}

Therefore, applying Lemma \ref{lemma:VecBernstein} to $g_{\delta,i}(x)$, the following bound holds with probability at least $1-\eta$.
\begin{align}
\frac{1}{N}\Big\|\sum_{i=1}^N[g_{\delta,i}(x)-\nabla f_{\delta}(x)]\Big\|<\frac{4L_fd}{3N}\log\Big(\frac{d+1}{\eta}\Big)+2L_fd\sqrt{\frac{2}{N}\log\Big(\frac{d+1}{\eta}\Big)}.\label{eq:gstd} 
\end{align}
Note that 
\begin{align}
\|\nabla f_{\delta}(x)-\nabla f(x)\|=\big\|\mathbb{E}_{v\sim{\rm Unif}(\mathbb{B}_d)}[\nabla f(x+\delta v)-\nabla f(x)]\big\|\le\delta\ell_f.\label{eq:gbias}
\end{align}
As a result, we can prove the conclusion as follows by using Eqs. (\ref{eq:gstd}) and (\ref{eq:gbias}) above.
\begin{align}
\|g_{\delta}(x)-\nabla f(x)\|=&\Big\|\Big[\frac{1}{N}\sum_{i=1}^Ng_{\delta,i}(x)\Big]-\nabla f(x)\Big\|\nonumber\\
\le&\Big\|\Big[\frac{1}{N}\sum_{i=1}^Ng_{\delta,i}(x)\Big]-\nabla f_{\delta}(x)\Big\|+\|\nabla f_{\delta}(x)-\nabla f(x)\|\nonumber\\
<&\frac{4L_fd}{3N}\log\Big(\frac{d+1}{\eta}\Big)+2L_fd\sqrt{\frac{2}{N}\log\Big(\frac{d+1}{\eta}\Big)}+\delta\ell_f.\nonumber
\end{align}

\end{proof}

\subsection{Orthogonal Transformation}
\begin{lemma}\label{lemma:orthoT}
There exists an orthogonal transformation $\mathcal{T}$ from the space $\mathbb{R}^{d-1}$ to $\mathcal{Z}_{d}=\{z=[z_1,\ldots,z_d]\in\mathbb{R}^d:\sum_i z_i=0\}$, that is, $\mathcal{T}$ is invertible and satisfies the following properties for any $x,y\in \mathcal{Z}_{d}$ and $\alpha,\beta\in\mathbb{R}$. 
\begin{align}
\mathcal{T}(\alpha x+\beta y)=& \alpha \mathcal{T}(x)+\beta \mathcal{T}(y),\label{eq:T_linear}\\
\langle \mathcal{T}(x),\mathcal{T}(y)\rangle=& \langle x,y\rangle.\label{eq:T_ortho}
\end{align}
\end{lemma}
\begin{proof}
It can be verified that $\mathbb{R}^{d}$ admits the following orthonormal basis with $\langle e_i,e_j\rangle=0$ for any $i\ne j$ and $\|e_i\|=1$.
\begin{align}
e_k=&\frac{1}{\sqrt{k(k+1)}}[\underbrace{1,1,\ldots,1}_{k~1's},-k,\underbrace{0,0,\ldots,0}_{(d-k-1)~0's}]\in\mathbb{R}^d; k=1,2,\ldots,d-1.\nonumber\\
e_d=&\frac{1}{\sqrt{d}}[\underbrace{1,1,\ldots,1}_{d~1's}]\in\mathbb{R}^d.\nonumber
\end{align} 
Define the transformation $\mathcal{T}$ at $x=[x_1,x_2,\ldots,x_{d-1}]\in\mathbb{R}^{d-1}$ as follows. 
\begin{align}
\mathcal{T}(x)=\sum_{i=1}^{d-1} x_ie_i.\label{eq:Tx}
\end{align}
Since $\mathcal{Z}_d$ is a linear subspace of $\mathbb{R}^d$ orthogonal to $e_d$, $\mathcal{Z}_d$ admits the orthonormal basis $\{e_i\}_{i=1}^{d-1}$. Hence, $\mathcal{T}(x)\in\mathcal{Z}_d$. Conversely, for any $y\in\mathcal{Z}_d$, there exists unique $x\in\mathbb{R}^{d-1}$ such that $y=\sum_{i=1}^{d-1} x_ie_i$. Hence, $\mathcal{T}:\mathbb{R}^{d-1}\to\mathcal{Z}^d$ is invertible. 

For any $x=[x_1,\ldots,x_{d-1}], y=[y_1,\ldots,y_{d-1}]\in\mathbb{R}^{d-1}$ and $\alpha,\beta\in\mathbb{R}$, we can prove Eqs. (\ref{eq:T_linear}) and (\ref{eq:T_ortho}) respectively as follows.
\begin{align}
\mathcal{T}(\alpha x+\beta y)=&\sum_{i=1}^{d-1} (\alpha x_i+\beta y_i)e_i\nonumber\\
=&\alpha \sum_{i=1}^{d-1} x_ie_i+\beta \sum_{i=1}^{d-1} y_ie_i\nonumber\\
=&\alpha \mathcal{T}(x)+\beta \mathcal{T}(y).\nonumber
\end{align}
\begin{align}
\langle \mathcal{T}(x),\mathcal{T}(y)\rangle=&\Big\langle\sum_{i=1}^{d-1} x_ie_i,\sum_{j=1}^{d-1} y_je_j\Big\rangle\nonumber\\
=&\sum_{i=1}^{d-1}\sum_{j=1}^{d-1}x_iy_j\langle e_i,e_j\rangle\nonumber\\
=&\sum_{i=1}^{d-1}x_iy_i=\langle x,y\rangle.\nonumber
\end{align}
\end{proof}

\subsection{Basic Inequalities}
\begin{lemma}\label{lemma:logx_by_x}
For any $\epsilon\in(0,0.5]$ and $x\ge 4\epsilon^{-1}\log(\epsilon^{-1})$, the following inequality holds.
\begin{align}
0<\frac{\log x}{x}\le \epsilon\label{eq:logx_by_x}
\end{align}
Specifically, any $x\ge 3$ satisfies $\frac{\log x}{x}\le \frac{1}{2}$. 
\end{lemma}
\begin{proof}
As $\epsilon^{-1}\ge 2$, we have $x\ge 4\epsilon^{-1}\log(\epsilon^{-1})\ge (4)(2)\log(2)>5.54$, so $\log x>\log 5.54>1.71$, which proves the first $<$ of Eq. (\ref{eq:logx_by_x}). 

Note that the function $f(x)=\frac{\log x}{x}$ has the following derivative
\begin{align}
f'(x)=\frac{1-\log x}{x^2}<0,\nonumber
\end{align}
where $<$ uses $\log x>1.71$. Hence, $f$ is monotonic decreasing in $x\ge 4\epsilon^{-1}\log(\epsilon^{-1})>5.54$, Therefore, we prove the second $\le$ of Eq. (\ref{eq:logx_by_x}) as follows.
\begin{align}
\frac{\log x}{x\epsilon}\le&\frac{\log [4\epsilon^{-1}\log(\epsilon^{-1})]}{\epsilon[4\epsilon^{-1}\log(\epsilon^{-1})]}\nonumber\\
=&\frac{\log 4+\log(\epsilon^{-1})+\log[\log(\epsilon^{-1})]}{4\log(\epsilon^{-1})}\nonumber\\
\overset{(a)}{\le}&\frac{\log 4}{4\log (2)}+\frac{\log(\epsilon^{-1})+\log(\epsilon^{-1})}{4\log(\epsilon^{-1})}=1,
\end{align}
where (a) uses $\epsilon^{-1}\ge 2$ and $\log u\le u$ for $u=\log(\epsilon^{-1})$.

When $x\ge 3$, $f'(x)=\frac{1-\log x}{x^2}<0$, so $f(x)\le f(3)=\frac{\log 3}{3}<\frac{1}{2}$. 
\end{proof}

\begin{lemma}\label{lemma:pi_diameter}
For any $\pi,\pi'\in\Pi$, we have $\|\pi'-\pi\|\le\sqrt{2|\mathcal{S}|}$.
\end{lemma}
\begin{proof}
\begin{align}
\|\pi'-\pi\|^2=\sum_{s,a}|\pi'(a|s)-\pi(a|s)|^2\le \sum_{s,a}[\pi'^2(a|s)+\pi^2(a|s)]\le \sum_{s,a}[\pi'(a|s)+\pi(a|s)]=2|\mathcal{S}|.\nonumber
\end{align}
\end{proof}

\section{Negative Entropy Regularizer as a Strongly Convex Function of Occupancy Measure}\label{sec:sc_entropy}
The negative entropy regularizer (\ref{eq:entropy}) can be rewritten as follows
\begin{align}
\mathcal{H}_{\pi'}(\pi)=\mathbb{E}_{\pi,p_{\pi'},\rho}\Big[\sum_{t=0}^{\infty}\gamma^t\log\pi(a_t|s_t)\Big]=\frac{1}{1-\gamma}\sum_{s,a}d_{\pi,p_{\pi'}}(s,a)\log\frac{d_{\pi,p_{\pi'}}(s,a)}{d_{\pi,p_{\pi'}}(s)},\label{eq:entropy2}
\end{align}
where $d_{\pi,p_{\pi'}}(s)=\sum_{a'}d_{\pi,p_{\pi'}}(s,a')$. Hence, it suffices to prove that the following function of occupancy measure $d$ is strongly convex.
\begin{align}
H(d)=\sum_{s,a}d(s,a)\log\frac{d(s,a)}{d(s)},\label{eq:Hd}
\end{align}
where $d(s)=\sum_{a'}d(s,a')$. For any $\alpha\in[0,1]$ and occupancy measures $d_1, d_0$, denote $d_{\alpha}=\alpha d_1+(1-\alpha)d_0$ and the corresponding policy as $\pi_{\alpha}(a|s)=\frac{d_{\alpha}(s,a)}{d_{\alpha}(s)}$. Then we have 
\begin{align}
&\alpha H(d_1)+(1-\alpha)H(d_0)-H(d_{\alpha})\nonumber\\
=&\sum_{s,a}\Big[\alpha d_1(s,a)\log\pi_1(a|s)+(1-\alpha)d_0(s,a)\log\pi_0(a|s)\nonumber\\
&-[\alpha d_1(s,a)+(1-\alpha)d_0(s,a)]\log\pi_{\alpha}(a|s)\Big]\nonumber\\
=&\sum_{s,a}\Big[\alpha d_1(s,a)\log\frac{\pi_1(a|s)}{\pi_{\alpha}(a|s)}+(1-\alpha)d_0(s,a)\log\frac{\pi_0(a|s)}{\pi_{\alpha}(a|s)}\Big]\nonumber\\
=&\sum_{s,a}\Big[\alpha d_1(s)\pi_1(a|s)\log\frac{\pi_1(a|s)}{\pi_{\alpha}(a|s)}+(1-\alpha)d_0(s)\pi_0(a|s)\log\frac{\pi_0(a|s)}{\pi_{\alpha}(a|s)}\Big]\nonumber\\
=&\sum_s\Big[\alpha d_1(s){\rm KL}[\pi_1(\cdot|s)\|\pi_{\alpha}(a|s)]+(1-\alpha)d_0(s){\rm KL}[\pi_0(\cdot|s)\|\pi_{\alpha}(a|s)]\Big]\nonumber\\
\overset{(a)}{\ge}&\frac{1}{2}\sum_s\Big[\alpha d_1(s)\|\pi_1(\cdot|s)-\pi_{\alpha}(\cdot|s)\|_1^2+(1-\alpha)d_0(s)\|\pi_0(\cdot|s)-\pi_{\alpha}(\cdot|s)\|_1^2\Big]\nonumber\\
\overset{(b)}{\ge}&\frac{D}{2}\sum_s\Big[\alpha \|\pi_1(\cdot|s)-\pi_{\alpha}(\cdot|s)\|_1^2+(1-\alpha)\|\pi_0(\cdot|s)-\pi_{\alpha}(\cdot|s)\|_1^2\Big]\nonumber\\
\ge&\frac{D}{2}\Big[\alpha \max_s\|\pi_1(\cdot|s)-\pi_{\alpha}(\cdot|s)\|_1^2+(1-\alpha)\max_s\|\pi_0(\cdot|s)-\pi_{\alpha}(\cdot|s)\|_1^2\Big]\nonumber\\
\overset{(c)}{\ge}&\frac{D(1-\gamma)}{2}\Big[\alpha \|d_1-d_{\alpha}\|_1^2+(1-\alpha)\max_s\|d_0-d_{\alpha}\|_1^2\Big]\nonumber\\
=&\frac{D(1-\gamma)}{2}\Big[\alpha(1-\alpha)^2 \|d_1-d_0\|_1^2+(1-\alpha)\alpha^2\|d_1-d_0\|_1^2\Big]\nonumber\\
=&\frac{\alpha(1-\alpha)}{2}\cdot D(1-\gamma)\|d_1-d_0\|_1^2.
\end{align}
where (a) uses Pinsker's inequality, (b) uses Assumption \ref{assum:dmin}, (c) uses Eq. (\ref{eq:visit_dsa}) with $p'=p$. The inequality above implies that $H(d)$ is $D(1-\gamma)$-strongly convex, so the negative entropy regularizer (\ref{eq:entropy2}) can be seen as a $D$-strongly convex function of the occupancy measure $d_{\pi,p_{\pi'}}$. 

\section{Existing Assumptions That Implies Assumption \ref{assum:dmin}}\label{sec:dmin} 
The following assumptions have been used in the reinforcement learning literature. We will show that each of these assumptions implies Assumption \ref{assum:dmin}. 
\begin{assum}\citep{bhandari2024global}\label{assum:rho_pos} $\rho(s)>0$ for any $s\in\mathcal{S}$.
\end{assum}
\begin{assum}\label{assum:d_rho_ratio}\citep{agarwal2021theory,leonardos2022global,PGrobust_ICML2023,chen2024acc} $D_{\rho}:=\sup_{\pi\in\Pi,p\in\mathcal{P}}\|d_{\pi,p}/\rho\|_{\infty}<\infty$.
\end{assum}
\begin{assum}\label{assum:mu}\citep{wei2021last,chen2021sample}
There exists a constant $\mu_{\min}>0$ and mixing time $t_{\text{mix}}\in\mathbb{N}$ such that under any policy $\pi\in\Pi$ and transition kernel $p\in\mathcal{P}$, the stationary state distribution $\mu_{\pi,p}(s)$ has uniform lower bound $\min_{s\in\mathcal{S}}\mu_{\pi,p}(s)\ge \mu_{\min}$, and 
\begin{align}
{d}_{\mathrm{TV}}\big[\mathbb{P}_{\pi,p,\rho}(s_{t_{\text{mix}}}=\cdot), \mu_{\pi,p}\big]\le \frac{1}{4}, \label{eq:dTV}
\end{align}
where $\mathbb{P}_{\pi,p,\rho}(s_{t_{\text{mix}}}=\cdot)$ denotes the state distribution at time $t_{\text{mix}}$, under the policy $\pi$, transition kernel $p$ and initial state distribution $\rho$, and ${d}_{\mathrm{TV}}$ denotes the total variation distance between two probability distributions. 
\end{assum}
\textbf{Proof of Assumption \ref{assum:rho_pos}$\Rightarrow$Assumption \ref{assum:dmin}:} For any policy $\pi\in\Pi$, transition kernel $p\in\mathcal{P}$ and state $s\in\mathcal{S}$, we have
\begin{align}
d_{\pi,p}(s)=&\sum_a d_{\pi,p}(s,a)\nonumber\\
\overset{(a)}{=}&\sum_a (1-\gamma)\sum_{t=0}^{\infty} \gamma^t\mathbb{P}_{\pi,p,\rho} \{s_t=s,a_t=a\}\nonumber\\
=&(1-\gamma)\sum_{t=0}^{\infty} \gamma^t\mathbb{P}_{\pi,p,\rho} \{s_t=s\}\nonumber\\
\ge&(1-\gamma)\mathbb{P}_{\pi,p,\rho} \{s_0=s\}\nonumber\\
=&(1-\gamma)\rho(s)\nonumber\\
\ge&(1-\gamma)\min_{s\in\mathcal{S}}\rho(s).\nonumber
\end{align}
As $\mathcal{S}$ is a finite state space, $\rho(s)>0, \forall s\in\mathcal{S}$ implies that $\min_{s\in\mathcal{S}}\rho(s)>0$. Hence, Assumption \ref{assum:dmin} holds with $D=(1-\gamma)\min_{s\in\mathcal{S}}\rho(s)>0$.

\textbf{Proof of Assumption \ref{assum:d_rho_ratio}$\Rightarrow$Assumption \ref{assum:dmin}:} If $\rho(s)=0$ for a state $s$, then Assumption \ref{assum:d_rho_ratio} implies that $d_{\pi,p}(s)=(1-\gamma)\sum_{t=0}^{\infty} \gamma^t\mathbb{P}_{\pi,p,\rho} \{s_t=s\}=0$ for any $\pi\in\Pi$ and $p\in\mathcal{P}$, which means the state $s$ will never be visited. Therefore, we can exclude all such states $s$ from $\mathcal{S}$ such that Assumption \ref{assum:rho_pos} holds, which implies Assumption \ref{assum:dmin} as proved above. 

\textbf{Proof of Assumption \ref{assum:mu}$\Rightarrow$Assumption \ref{assum:dmin}:} Eq. (\ref{eq:dTV}) implies that for any $n\in\mathbb{N}_+$, we have
\begin{align}
{d}_{\mathrm{TV}}\big[\mathbb{P}_{\pi,p,\rho}(s_{nt_{\text{mix}}}=\cdot), \mu_{\pi,p}\big]=\frac{1}{2}\sum_s |\mathbb{P}_{\pi,p,\rho}\{s_{nt_{\text{mix}}}=s\}-\mu_{\pi,p}(s)|\le \frac{1}{4^n}.\nonumber 
\end{align}
Select $n=\lceil\log(\mu_{\min}^{-1})/\log 4\rceil$. Then the bound above implies $|\mathbb{P}_{\pi,p,\rho}\{s_{nt_{\text{mix}}}=s\}-\mu_{\pi,p}(s)|\le \mu_{\min}/2$ for any state $s$, which along with $\mu_{\pi,p}(s)\ge \mu_{\min}$ implies that $\mathbb{P}_{\pi,p,\rho}\{s_{nt_{\text{mix}}}=s\}\ge \mu_{\min}/2$. Therefore, we can prove Assumption \ref{assum:dmin} as follows.
\begin{align}
d_{\pi,p}(s)=&(1-\gamma)\sum_{t=0}^{\infty} \gamma^t\mathbb{P}_{\pi,p,\rho} \{s_t=s\} \ge(1-\gamma)\gamma^{nt_{\text{mix}}}\mathbb{P}_{\pi,p,\rho}\{s_{nt_{\text{mix}}}=s\} \ge \frac{\mu_{\min}}{2}\gamma^{nt_{\text{mix}}}(1-\gamma). \nonumber
\end{align}

\section{Proof of Theorem \ref{thm:ToOpt}}\label{sec:proof_thm:ToOpt}
Fix any $\pi_0,\pi_1\in\Pi$. For any $\alpha\in[0,1]$, denote $d_{\alpha}=\alpha d_{\pi_1,p_{\pi_1}}+(1-\alpha)d_{\pi_0,p_{\pi_0}}$, $\pi_{\alpha}(a|s)=\frac{d_{\alpha}(s,a)}{d_{\alpha}(s)}$ where $d_{\alpha}(s)=\sum_{a'}d_{\alpha}(s,a')$, and $p_{\alpha}=p_{\pi_\alpha}$. It can be easily verified that $d_0=d_{\pi_0,p_0}$, $d_1=d_{\pi_1,p_1}$ and $d_{\alpha}=\alpha d_0+(1-\alpha)d_1$. Then we can obtain the following derivatives and their bounds about $\pi_{\alpha}, d_{\alpha}$ in Eqs. (\ref{eq:dpi_alpha})-(\ref{eq:pd_smooth}).
\begin{align}
&\frac{d_{\alpha}(s)[d_1(s,a)-d_0(s,a)]-d_{\alpha}(s,a)[d_1(s)-d_0(s)]}{d_{\alpha}^2(s)}\nonumber\\
=&\frac{[\alpha d_1(s)+(1-\alpha)d_0(s)][d_1(s,a)-d_0(s,a)]-[\alpha d_1(s,a)+(1-\alpha)d_0(s,a)][d_1(s)-d_0(s)]}{d_{\alpha}^2(s)}\nonumber\\
=&\frac{d_0(s)d_1(s,a)-d_0(s,a)d_1(s)}{d_{\alpha}^2(s)}\nonumber\\
=&\frac{d_0(s)d_1(s)[\pi_1(a|s)-\pi_0(a|s)]}{d_{\alpha}^2(s)}.\label{eq:dpi_alpha}
\end{align}
Hence,
\begin{align}
\Big\|\frac{d\pi_{\alpha}}{d\alpha}\Big\|^2=&\sum_{s,a}\Big|\frac{d_0(s)d_1(s)[\pi_1(a|s)-\pi_0(a|s)]}{d_{\alpha}^2(s)}\Big|^2\nonumber\\
\overset{(a)}{\le}&\sum_{s,a}\Big[\frac{\max[d_0(s),d_1(s)]\min[d_0(s),d_1(s)]}{\min^2[d_0(s),d_1(s)]}\Big]^2[\pi_1(a|s)-\pi_0(a|s)]^2\nonumber\\
\overset{(b)}{\le}& D^{-2}\sum_{s,a}[\pi_1(a|s)-\pi_0(a|s)]^2\le D^{-2}\|\pi_1-\pi_0\|^2, \label{eq:dpi_alpha_norm}
\end{align}
where (a) uses $d_{\alpha}(s)=\alpha d_1(s)+(1-\alpha)d_0(s)\ge\min[d_0(s),d_1(s)]$ and (b) uses Assumption \ref{assum:dmin}. Then by taking derivative of Eq. (\ref{eq:dpi_alpha}), we have
\begin{align}
\frac{d^2}{d\alpha^2}\pi_{\alpha}(a|s)=-\frac{2d_0(s)d_1(s)[\pi_1(a|s)-\pi_0(a|s)][d_1(s)-d_0(s)]}{d_{\alpha}^3(s)}.\label{eq:d2pi_alpha}
\end{align}
Hence,
\begin{align}
\Big\|\frac{d^2\pi_{\alpha}}{d\alpha^2}\Big\|^2=&\sum_{s,a}\Big|\frac{2d_0(s)d_1(s)[\pi_1(a|s)-\pi_0(a|s)][d_1(s)-d_0(s)]}{[\alpha d_1(s)+(1-\alpha)d_0(s)]^3}\Big|^2\nonumber\\
\overset{(a)}{\le}&\sum_{s,a}\Big[\frac{2\max[d_0(s),d_1(s)]\min[d_0(s),d_1(s)]\big|d_1(s)-d_0(s)\big|}{D^2\min[d_0(s),d_1(s)]}\Big]^2\![\pi_1(a|s)\!-\!\pi_0(a|s)]^2\nonumber\\
\le& (2D^{-2})^2\max_s\big[|d_1(s)-d_0(s)|^2\big]\sum_{s,a}[\pi_1(a|s)-\pi_0(a|s)]^2\nonumber\\
\le& (2D^{-2})^2\|\pi_1-\pi_0\|^2\Big[\sum_s|d_1(s)-d_0(s)|\Big]^2\nonumber\\
\overset{(b)}{\le}&(2D^{-2})^2\|\pi_1-\pi_0\|^2\Big[\frac{\gamma\sqrt{|\mathcal{A}|}}{1-\gamma}\|\pi_1-\pi_0\|+\frac{\gamma\sqrt{|\mathcal{S}|}}{1-\gamma}\|p_{\pi_1}-p_{\pi_0}\|\Big]^2\nonumber\\
\overset{(c)}{\le}&(2D^{-2})^2\|\pi_1-\pi_0\|^2\Big[\frac{\gamma\sqrt{|\mathcal{A}|}}{1-\gamma}\|\pi_1-\pi_0\|+\frac{\gamma\epsilon_p\sqrt{|\mathcal{S}|}}{1-\gamma}\|\pi_1-\pi_0\|\Big]^2\nonumber\\
\le&(2D^{-2})^2\|\pi_1-\pi_0\|^4\Big[\frac{\gamma(\epsilon_p\sqrt{|\mathcal{S}|}+\sqrt{|\mathcal{A}|})}{1-\gamma}\Big]^2, \label{eq:d2pi_alpha_norm}
\end{align}
where (a) uses $d_{\alpha}(s)=\alpha d_1(s)+(1-\alpha)d_0(s)\ge\min[d_0(s),d_1(s)]\ge D$, (b) uses Lemma \ref{lemma:occup_Lip}, and (c) uses Assumption \ref{assum:sensitive}. 
\begin{align}
&d_0(s)d_1(s)\Big|\frac{d}{d\alpha}\Big[\frac{d_{\alpha}(s,a)}{d_{\alpha}^2(s)}\Big]\Big|\nonumber\\
=&\Big|\frac{d_0(s)d_1(s)}{d_{\alpha}^2(s)}[d_1(s,a)-d_0(s,a)]-\frac{2d_0(s)d_1(s)d_{\alpha}(s,a)}{d_{\alpha}^3(s)}[d_1(s)-d_0(s)]\Big|\nonumber\\
\le&\frac{d_0(s)d_1(s)}{d_{\alpha}^2(s)}\Big[|d_1(s,a)-d_0(s,a)|+\frac{2d_{\alpha}(s,a)}{d_{\alpha}(s)}|d_1(s)-d_0(s)|\Big]\nonumber\\
\le&\frac{\max[d_0(s),d_1(s)]\min[d_0(s),d_1(s)]}{\min^2[d_0(s),d_1(s)]}\big[|d_1(s,a)-d_0(s,a)|+2\pi_{\alpha}(a|s)|d_1(s)-d_0(s)|\big]\nonumber\\
\le&D^{-1}\big[|d_1(s,a)-d_0(s,a)|+2\pi_{\alpha}(a|s)|d_1(s)-d_0(s)|\big].\label{eq:dddd}
\end{align}
\begin{align}
&\frac{d}{d\alpha}[d_{\alpha}(s,a)p_{\alpha}(s'|s,a)]\nonumber\\
=&p_{\alpha}(s'|s,a)[d_1(s,a)-d_0(s,a)]+d_{\alpha}(s,a)\cdot\frac{d}{d\alpha}\pi_{\alpha}(a|s)\cdot\nabla_{\pi}p_{\pi_{\alpha}}(s'|s,a)\nonumber\\
=&p_{\alpha}(s'|s,a)[d_1(s,a)\!-\!d_0(s,a)]\!+\!\frac{d_{\alpha}(s,a)d_0(s)d_1(s)[\pi_1(a|s)\!-\!\pi_0(a|s)]}{d_{\alpha}^2(s)}\cdot\nabla_{\pi}p_{\pi_{\alpha}}(s'|s,a)\label{eq:pd_dalpha}
\end{align}
Then for any $\alpha,\alpha'\in [0,1]$, we have
\begin{align}
&\Big|\frac{d}{d\alpha}[d_{\alpha'}(s,a)p_{\alpha'}(s'|s,a)]-\frac{d}{d\alpha}[d_{\alpha}(s,a)p_{\alpha}(s'|s,a)]\Big|\nonumber\\
\overset{(a)}{\le}&|p_{\alpha'}(s'|s,a)-p_{\alpha}(s'|s,a)|\cdot|d_1(s,a)-d_0(s,a)|+d_0(s)d_1(s)|\pi_1(a|s)-\pi_0(a|s)|\cdot\nonumber\\
&\Big[\Big|\frac{d_{\alpha'}(s,a)}{d_{\alpha'}^2(s)}\Big|\|\nabla_{\pi}p_{\pi_{\alpha'}}(s'|s,a)-\nabla_{\pi}p_{\pi_{\alpha}}(s'|s,a)\|+\Big|\frac{d_{\alpha'}(s,a)}{d_{\alpha'}^2(s)}-\frac{d_{\alpha}(s,a)}{d_{\alpha}^2(s)}\Big|\|\nabla_{\pi}p_{\pi_{\alpha}}(s'|s,a)\|\Big]\nonumber\\
\overset{(b)}{\le}&\epsilon_p\|\pi_{\alpha'}-\pi_{\alpha}\||d_1(s,a)-d_0(s,a)|\nonumber\\
&+\pi_{\alpha'}(a|s)|\pi_1(a|s)-\pi_0(a|s)|\cdot\frac{\max[d_0(s),d_1(s)]\min[d_0(s),d_1(s)]}{\min[d_0(s),d_1(s)]}\cdot S_p\|\pi_{\alpha'}-\pi_{\alpha}\|\nonumber\\
&+D^{-1}\epsilon_p|\pi_1(a|s)-\pi_0(a|s)|\cdot\big[|d_1(s,a)-d_0(s,a)|+2\pi_{\alpha}(a|s)|d_1(s)-d_0(s)|\big]\cdot|\alpha'-\alpha|\nonumber\\
\overset{(c)}{\le}&\epsilon_p D^{-1}\|\pi_1-\pi_0\|\cdot|\alpha'-\alpha|\cdot|d_1(s,a)-d_0(s,a)|\nonumber\\
&+S_p\pi_{\alpha'}(a|s)\cdot|\pi_1(a|s)-\pi_0(a|s)|\cdot[d_0(s)+d_1(s)]\cdot D^{-1}\|\pi_1-\pi_0\|\cdot|\alpha'-\alpha|\nonumber\\
&+D^{-1}\epsilon_p|\pi_1(a|s)-\pi_0(a|s)|\cdot\big[|d_1(s,a)-d_0(s,a)|+2\pi_{\alpha}(a|s)|d_1(s)-d_0(s)|\big]\cdot|\alpha'-\alpha|\nonumber\\
\overset{(d)}{\le}& \ell_{dp}(s,a)|\alpha'-\alpha|,\label{eq:pd_smooth}
\end{align}
where (a) uses Eq. (\ref{eq:pd_dalpha}), (b) uses Assumptions \ref{assum:sensitive}-\ref{assum:smooth_pr}, 
$d_{\alpha'}(s,a)=d_{\alpha'}(s)\pi_{\alpha'}(a|s)$, $d_{\alpha'}(s)=\alpha'd_1(s)+(1-\alpha')d_0(s)\ge\min[d_0(s),d_1(s)]$ and Eq. (\ref{eq:dddd}), (c) uses Assumption \ref{assum:dmin} as well as Eq. (\ref{eq:dpi_alpha_norm}), (d) defines $\ell_{dp}(s,a)$ as the following Eq. (\ref{eq:ldp}) and uses $\pi_{\alpha}(a|s)=\frac{\alpha d_1(s)\pi_1(a|s)+(1-\alpha)d_0(s)\pi_0(a|s)}{\alpha d_1(s)+(1-\alpha)d_0(s)}\le \pi_0(a|s)+\pi_1(a|s)$. 
\begin{align}
\ell_{dp}(s,a)=&2D^{-1}\epsilon_p\|\pi_1-\pi_0\||d_1(s,a)-d_0(s,a)|\nonumber\\
&\!+\!2D^{-1}\epsilon_p[\pi_1(a|s)+\pi_0(a|s)]\cdot|\pi_1(a|s)-\pi_0(a|s)|\cdot|d_1(s)-d_0(s)|\nonumber\\
&\!+\!D^{-1}S_p[\pi_1(a|s)+\pi_0(a|s)]\!\cdot\!|\pi_1(a|s)-\pi_0(a|s)|\!\cdot\!\|\pi_1-\pi_0\|\cdot[d_0(s)+d_1(s)].\label{eq:ldp}
\end{align}

Denote $e_{\alpha}(s)=d_{\pi_{\alpha},p_{\alpha}}(s)-d_{\alpha}(s)$ as the error term due to the policy-dependent transition kernel  $p_{\alpha}=p_{\pi_\alpha}$\footnote{If $p_{\pi_\alpha}\equiv p$ does not depend on the policy $\pi_\alpha$, it can be easily verified that $e_{\alpha}(s)=0$ for all $s\in\mathcal{S}$.}. Note that the occupancy measure (\ref{eq:occup}) satisfies that the Bellman equation (\ref{eq:Bellman_occup}) repeated as follows.
\begin{align}
d_{\pi,p}(s')&=(1-\gamma)\rho(s')+\gamma\sum_{s,a}d_{\pi,p}(s)\pi(a|s)p(s'|s,a),\quad s'\in\mathcal{S}.\label{eq:Bellman_occup2}
\end{align}
Therefore, the error term $e_{\alpha}(s)$ satisfies the following recursion. 
\begin{align}
&e_{\alpha}(s')\nonumber\\
=&d_{\pi_{\alpha},p_{\alpha}}(s')-\alpha d_1(s')-(1-\alpha)d_0(s')\nonumber\\
=&\gamma\sum_{s,a}[d_{\pi_{\alpha},p_{\alpha}}(s)\pi_{\alpha}(a|s)p_{\alpha}(s'|s,a)-\alpha d_{\pi_1,p_1}(s)\pi_1(a|s)p_1(s'|s,a)\nonumber\\
&-(1-\alpha)d_{\pi_0,p_0}(s)\pi_0(a|s)p_0(s'|s,a)]\nonumber\\
=&\gamma\sum_{s,a}[e_{\alpha}(s)\pi_{\alpha}(a|s)p_{\alpha}(s'|s,a)+d_{\alpha}(s,a)p_{\alpha}(s'|s,a)-\alpha d_1(s,a)p_1(s'|s,a)\nonumber\\
&-(1-\alpha)d_0(s,a)p_0(s'|s,a)].\label{eq:e_alpha_recursion}
\end{align}
The above inequality implies that
\begin{align}
&\sum_{s'}|e_{\alpha}(s')|\nonumber\\
\le&\gamma\sum_{s,a,s'}\big[|e_{\alpha}(s)|\pi_{\alpha}(a|s)p_{\alpha}(s'|s,a)\nonumber\\
&+|d_{\alpha}(s,a)p_{\alpha}(s'|s,a)\!-\!\alpha d_1(s,a)p_1(s'|s,a)\!-\!
(1-\alpha)d_0(s,a)p_0(s'|s,a)|\big]\nonumber\\
\overset{(a)}{\le}&\gamma\sum_{s}|e_{\alpha}(s)|+\frac{\gamma\alpha(1-\alpha)}{2}\sum_{s,a,s'}\ell_{dp}(s,a)\nonumber\\
\overset{(b)}{\le}&\gamma\sum_{s}|e_{\alpha}(s)|+\frac{\gamma|\mathcal{S}|\alpha(1-\alpha)}{2}\Big[2D^{-1}\epsilon_p\|\pi_1-\pi_0\|\sum_{s,a}|d_1(s,a)-d_0(s,a)|\nonumber\\
&+4D^{-1}\epsilon_p\|\pi_1-\pi_0\|_{\infty}\sum_{s}|d_1(s)-d_0(s)|+4D^{-1}S_p\|\pi_1-\pi_0\|_{\infty}\cdot\|\pi_1-\pi_0\|\Big]\nonumber\\
\overset{(c)}{\le}&\gamma\sum_{s}|e_{\alpha}(s)|\!+\!\frac{\gamma|\mathcal{S}|\alpha(1\!-\!\alpha)}{2}\Big[6D^{-1}\epsilon_p\|\pi_1\!-\!\pi_0\|\!\cdot\!\frac{1}{1-\gamma}\Big(\sqrt{|\mathcal{A}|}\|\pi_1\!-\!\pi_0\|\!+\!\gamma\sqrt{|\mathcal{S}|}\|p_{\pi_1}\!-\!p_{\pi_0}\|\Big)\nonumber\\
&+4D^{-1}S_p\|\pi_1-\pi_0\|^2\Big]\nonumber\\
\overset{(d)}{\le}&\gamma\sum_{s}|e_{\alpha}(s)|+3D^{-1}\gamma|\mathcal{S}|\alpha(1-\alpha)\|\pi_1-\pi_0\|^2\Big[\frac{\epsilon_p}{1-\gamma}(\sqrt{|\mathcal{A}|}+\gamma\epsilon_p\sqrt{|\mathcal{S}|})+S_p\Big],\nonumber
\end{align}
where (a) uses Eq. (\ref{eq:pd_smooth}) which implies that $d_{\alpha}(s,a)p_{\alpha}(s'|s,a)$ is a Lipschitz smooth function with Lipschitz constant $\ell_{dp}(s,a)$ defined by Eq. (\ref{eq:ldp}), (b) uses Eq. (\ref{eq:ldp}), (c) uses $\|\pi_1-\pi_0\|_{\infty}\le\|\pi_1-\pi_0\|$ and Lemma \ref{lemma:occup_Lip}, and (d) uses Assumption \ref{assum:sensitive}. Rearranging the above inequality, we get
\begin{align}
\sum_{s}|e_{\alpha}(s)|\le \frac{3\gamma|\mathcal{S}|\alpha(1-\alpha)}{D(1-\gamma)^2}\|\pi_1-\pi_0\|^2\big[\epsilon_p\big(\sqrt{|\mathcal{A}|}+\gamma\epsilon_p\sqrt{|\mathcal{S}|}\big)+S_p(1-\gamma)\big].\label{eq:ealpha_norm}
\end{align} 
Therefore, for any reward function $r$, we have
\begin{align}
&J_{\lambda}(\pi_{\alpha},\pi_{\alpha},p_{\alpha},r)-\alpha J_{\lambda}(\pi_1,\pi_1,p_1,r)-(1-\alpha)J_{\lambda}(\pi_0,\pi_0,p_0,r)\nonumber\\
\overset{(a)}{=}&\frac{1}{1-\gamma}\sum_{s,a}\Big[d_{\pi_{\alpha},p_{\alpha}}(s,a)[r(s,a)-\lambda\log\pi_{\alpha}(a|s)]-\alpha d_1(s,a)[r(s,a)-\lambda\log\pi_1(a|s)]\nonumber\\
&-(1-\alpha)d_0(s,a)[r(s,a)-\lambda\log\pi_0(a|s)]\Big]\nonumber\\
=&\frac{1}{1-\gamma}\sum_{s,a}\Big[[d_{\pi_{\alpha},p_{\alpha}}(s,a)-d_{\alpha}(s,a)][r(s,a)-\lambda\log\pi_{\alpha}(a|s)]+d_{\alpha}(s,a)[r(s,a)-\lambda\log\pi_{\alpha}(a|s)]\nonumber\\
&-\alpha d_1(s,a)[r(s,a)-\lambda\log\pi_1(a|s)]-(1-\alpha)d_0(s,a)[r(s,a)-\lambda\log\pi_0(a|s)]\Big]\nonumber\\
\overset{(b)}{=}&\frac{1}{1-\gamma}\sum_{s,a}[d_{\pi_{\alpha},p_{\alpha}}(s)-d_{\alpha}(s)]\pi_{\alpha}(a|s)[r(s,a)-\lambda\log\pi_{\alpha}(a|s)]\nonumber\\
&+\frac{\lambda}{1-\gamma}\sum_{s,a}\Big[\alpha d_1(s,a)\log\frac{\pi_1(a|s)}{\pi_{\alpha}(a|s)}+(1-\alpha)d_0(s,a)\log\frac{\pi_0(a|s)}{\pi_{\alpha}(a|s)}\Big]\nonumber\\
\overset{(c)}{\ge}&-\!\frac{1+\lambda\log|\mathcal{A}|}{1-\gamma}\!\sum_{s}|e_{\alpha}(s)|\nonumber\\
&+\!\frac{\lambda}{1-\gamma}\!\sum_s\!\Big[\alpha d_1(s)\!\sum_a\!\Big(\!\pi_1(a|s)\log\frac{\pi_1(a|s)}{\pi_{\alpha}(a|s)}\Big)\!+\!(1\!-\!\alpha)d_0(s)\sum_a \Big(\!\pi_0(a|s)\log\frac{\pi_0(a|s)}{\pi_{\alpha}(a|s)}\Big)\Big]\nonumber\\
\overset{(d)}{\ge}&-\frac{1+\lambda\log|\mathcal{A}|}{1-\gamma}\frac{3\gamma|\mathcal{S}|\alpha(1-\alpha)}{D(1-\gamma)^2}\|\pi_1-\pi_0\|^2\big[\epsilon_p\big(\sqrt{|\mathcal{A}|}+\gamma\epsilon_p\sqrt{|\mathcal{S}|}\big)+S_p(1-\gamma)\big]\nonumber\\
&+\frac{\lambda}{1-\gamma}\sum_s \Big[\alpha d_1(s){\rm KL}[\pi_1(\cdot|s)\|\pi_{\alpha}(\cdot|s)]+(1-\alpha)d_0(s){\rm KL}[\pi_0(\cdot|s)\|\pi_{\alpha}(\cdot|s)]\Big]\nonumber\\
\overset{(e)}{\ge}&-\frac{3\gamma|\mathcal{S}|\alpha(1-\alpha)(1+\lambda\log|\mathcal{A}|)}{D(1-\gamma)^3}\|\pi_1-\pi_0\|^2\big[\epsilon_p\big(\sqrt{|\mathcal{A}|}+\gamma\epsilon_p\sqrt{|\mathcal{S}|}\big)+S_p(1-\gamma)\big]\nonumber\\
&+\frac{\lambda}{2(1-\gamma)}\sum_s \Big[\alpha d_1(s)\|\pi_1(\cdot|s)-\pi_{\alpha}(\cdot|s)\|_1^2+(1-\alpha)d_0(s)\|\pi_0(\cdot|s)-\pi_{\alpha}(\cdot|s)\|_1^2\Big]\nonumber\\
\overset{(f)}{=}&-\frac{3\gamma|\mathcal{S}|\alpha(1-\alpha)(1+\lambda\log|\mathcal{A}|)}{D(1-\gamma)^3}\|\pi_1-\pi_0\|^2\big[\epsilon_p\big(\sqrt{|\mathcal{A}|}+\gamma\epsilon_p\sqrt{|\mathcal{S}|}\big)+S_p(1-\gamma)\big]\nonumber\\
&+\frac{\lambda}{2(1-\gamma)}\sum_{s} \Big[\alpha d_1(s)\Big\|\frac{(1-\alpha)d_0(s)}{d_{\alpha}(s)}[\pi_1(\cdot|s)-\pi_0(\cdot|s)]\Big\|_1^2\nonumber\\
&+(1-\alpha)d_0(s)\Big\|\frac{\alpha d_1(s)}{d_{\alpha}(s)}[\pi_1(\cdot|s)-\pi_0(\cdot|s)]\Big\|_1^2\Big]\nonumber\\
\overset{(g)}{=}&\frac{\lambda\alpha(1-\alpha)}{2(1-\gamma)}\sum_{s} \frac{d_0(s)d_1(s)}{d_{\alpha}(s)}\|\pi_1(\cdot|s)-\pi_0(\cdot|s)\|_1^2\nonumber\\
&-\frac{3\gamma|\mathcal{S}|\alpha(1-\alpha)(1+\lambda\log|\mathcal{A}|)}{D(1-\gamma)^3}\|\pi_1-\pi_0\|^2\big[\epsilon_p\big(\sqrt{|\mathcal{A}|}+\gamma\epsilon_p\sqrt{|\mathcal{S}|}\big)+S_p(1-\gamma)\big]\nonumber\\
\overset{(h)}{\ge}&\frac{D\lambda\alpha(1-\alpha)}{2(1-\gamma)}\|\pi_1-\pi_0\|^2\nonumber\\
&-\frac{3\gamma|\mathcal{S}|\alpha(1-\alpha)(1+\lambda\log|\mathcal{A}|)}{D(1-\gamma)^3}\|\pi_1-\pi_0\|^2\big[\epsilon_p\big(\sqrt{|\mathcal{A}|}+\gamma\epsilon_p\sqrt{|\mathcal{S}|}\big)+S_p(1-\gamma)\big]\nonumber\\
\overset{(i)}{=}&\frac{\mu_1\alpha(1-\alpha)}{2}\|\pi_1-\pi_0\|^2,\label{eq:Jstrong_fixr}
\end{align}
where (a) uses Eq. (\ref{eq:Jtau}), (b) uses $d_{\pi_{\alpha},p_{\alpha}}(s,a)=d_{\pi_{\alpha},p_{\alpha}}(s)\pi_{\alpha}(a|s)$, $d_{\alpha}(s,a)=d_{\alpha}(s)\pi_{\alpha}(a|s)$ and $d_{\alpha}=\alpha d_1+(1-\alpha)d_0$, (c) uses $r(s,a)\in[0,1]$, $-\sum_a\pi_{\alpha}(a|s)\log\pi_{\alpha}(a|s)\in[0,\log|\mathcal{A}|]$ and $e_{\alpha}(s)=d_{\pi_{\alpha},p_{\alpha}}(s)-d_{\alpha}(s)$, (d) uses Eq. (\ref{eq:2grad_inprods}), (e) uses Pinsker's inequality, (f) uses $\pi_{\alpha}(a|s)=\frac{d_{\alpha}(s,a)}{d_{\alpha}(s)}=\frac{\alpha d_1(s)}{d_{\alpha}(s)}\pi_1(a|s)+\frac{(1-\alpha)d_0(s)}{d_{\alpha}(s)}\pi_0(a|s)$, (g) uses $d_{\alpha}(s)=\alpha d_1(s)+(1-\alpha)d_0(s)$, (h) uses Assumption \ref{assum:dmin} and $d_{\alpha}(s)\le \max[d_0(s),d_1(s)]$, and (i) defines the constant $\mu_1$ below. 
\begin{align}
\mu_1\overset{\rm def}{=}\frac{D\lambda}{1-\gamma}-\frac{6\gamma|\mathcal{S}|(1+\lambda\log|\mathcal{A}|)}{D(1-\gamma)^3}\big[\epsilon_p\big(\sqrt{|\mathcal{A}|}+\gamma\epsilon_p\sqrt{|\mathcal{S}|}\big)+S_p(1-\gamma)\big].\label{eq:mu1}
\end{align} 
Next, we begin to consider the policy-dependent reward $r_{\alpha}=r_{\pi_{\alpha}}$. Define the function $w(\alpha)=\alpha J_{\lambda}(\pi_1,\pi_1,p_1,r_{\alpha})+(1-\alpha) J_{\lambda}(\pi_0,\pi_0,p_0,r_{\alpha})$, which has the following derivative 
\begin{align}
w'(\alpha)=&J_{\lambda}(\pi_1,\pi_1,p_1,r_{\alpha})-J_{\lambda}(\pi_0,\pi_0,p_0,r_{\alpha})\nonumber\\
&+[\alpha\nabla_r J_{\lambda}(\pi_1,\pi_1,p_1,r_{\alpha})+(1-\alpha)\nabla_r J_{\lambda}(\pi_0,\pi_0,p_0,r_{\alpha})](\nabla_{\pi}r_{\pi_{\alpha}})\frac{d\pi_{\alpha}}{d\alpha}\label{eq:dw}
\end{align}
For any $0\le\alpha\le\alpha'\le 1$, we prove the smoothness of $w(\alpha)$ as follows.
\begin{align}
&|w'(\alpha')-w'(\alpha)|\nonumber\\
=&\Big|\int_{\alpha}^{\alpha'}\nabla_r[J_{\lambda}(\pi_1,\pi_1,p_1,r_{\Tilde{\alpha}})-J_{\lambda}(\pi_0,\pi_0,p_0,r_{\Tilde{\alpha}})](\nabla_{\pi} r_{\pi_{\Tilde{\alpha}}})\frac{d\pi_{\Tilde{\alpha}}}{d\Tilde{\alpha}} d\Tilde{\alpha}\nonumber\\
&+[\alpha'\nabla_r J_{\lambda}(\pi_1,\pi_1,p_1,r_{\alpha'})+(1-\alpha')\nabla_r J_{\lambda}(\pi_0,\pi_0,p_0,r_{\alpha'})](\nabla_{\pi}r_{\pi_{\alpha'}})\Big(\frac{d\pi_{\alpha'}}{d\alpha'}-\frac{d\pi_{\alpha}}{d\alpha}\Big)\nonumber\\
&+[\alpha'\nabla_r J_{\lambda}(\pi_1,\pi_1,p_1,r_{\alpha'})+(1-\alpha')\nabla_r J_{\lambda}(\pi_0,\pi_0,p_0,r_{\alpha'})](\nabla_{\pi}r_{\pi_{\alpha'}}-\nabla_{\pi}r_{\pi_{\alpha}})\frac{d\pi_{\alpha}}{d\alpha}\nonumber\\
&+\big\{\alpha'[\nabla_r J_{\lambda}(\pi_1,\pi_1,p_1,r_{\alpha'})-\nabla_r J_{\lambda}(\pi_1,\pi_1,p_1,r_{\alpha})]\nonumber\\
&+(1-\alpha')[\nabla_r J_{\lambda}(\pi_0,\pi_0,p_0,r_{\alpha'})-\nabla_r J_{\lambda}(\pi_0,\pi_0,p_0,r_{\alpha})]\big\}(\nabla_{\pi}r_{\pi_{\alpha}})\frac{d\pi_{\alpha}}{d\alpha}\nonumber\\
&+(\alpha'-\alpha)[\nabla_r J_{\lambda}(\pi_1,\pi_1,p_1,r_{\alpha})-\nabla_r J_{\lambda}(\pi_0,\pi_0,p_0,r_{\alpha})](\nabla_{\pi}r_{\pi_{\alpha}})\frac{d\pi_{\alpha}}{d\alpha}\Big|\nonumber\\
\overset{(a)}{\le}&\int_{\alpha}^{\alpha'} \frac{\epsilon_r\|\pi_1-\pi_0\|}{D(1-\gamma)^2}\big(\max_{s}\|\pi_1(\cdot|s)-\pi_0(\cdot|s)\|_1+\gamma\max_{s,a} \|p_1(\cdot|s,a)-p_0(\cdot|s,a)\|_1\big)d\Tilde{\alpha}\nonumber\\
&+\frac{\epsilon_r}{1-\gamma}\cdot 2D^{-2}\|\pi_1\!-\!\pi_0\|^2\Big[\frac{\gamma(\epsilon_p\sqrt{|\mathcal{S}|}+\sqrt{|\mathcal{A}|})}{1-\gamma}\Big]|\alpha'\!-\!\alpha| +\frac{S_r\|\pi_{\alpha'}-\pi_{\alpha}\|}{1-\gamma}\cdot D^{-1}\|\pi_1\!-\!\pi_0\|\nonumber\\
&+0+|\alpha'-\alpha|\cdot \frac{\epsilon_r\|\pi_1-\pi_0\|}{D(1-\gamma)^2}\big(\max_{s}\|\pi_1(\cdot|s)-\pi_0(\cdot|s)\|_1+\gamma\max_{s,a} \|p_1(\cdot|s,a)-p_0(\cdot|s,a)\|_1\big)\nonumber\\
\overset{(b)}{\le}&2|\alpha'-\alpha|\cdot \frac{\epsilon_r\|\pi_1-\pi_0\|}{D(1-\gamma)^2}\big(\sqrt{|\mathcal{A}|}\|\pi_1-\pi_0\|+\gamma\sqrt{|\mathcal{S}|}\|p_1-p_0\|\big)\nonumber\\
&+\frac{2\epsilon_r\|\pi_1-\pi_0\|^2}{D^2(1-\gamma)}\Big[\frac{\gamma(\epsilon_p\sqrt{|\mathcal{S}|}+\sqrt{|\mathcal{A}|})}{1-\gamma}\Big]|\alpha'-\alpha|+\frac{S_r\|\pi_1-\pi_0\|^2}{D^2(1-\gamma)}|\alpha'-\alpha|\nonumber\\
\overset{(c)}{\le}&\frac{2\epsilon_r\|\pi_1-\pi_0\|}{D(1-\gamma)^2}\big(\sqrt{|\mathcal{A}|}\|\pi_1-\pi_0\|+\gamma\epsilon_p\sqrt{|\mathcal{S}|}\|\pi_1-\pi_0\|\big)|\alpha'-\alpha|\nonumber\\
&+\frac{2\gamma\epsilon_r\|\pi_1-\pi_0\|^2}{D^2(1-\gamma)^2}\big(\sqrt{|\mathcal{A}|}+\epsilon_p\sqrt{|\mathcal{S}|}\big)|\alpha'-\alpha|+\frac{S_r(1-\gamma)\|\pi_1-\pi_0\|^2}{D^2(1-\gamma)^2}|\alpha'-\alpha|\nonumber\\
\overset{(d)}{\le}&\frac{4\epsilon_r(\sqrt{|\mathcal{A}|}+\gamma\epsilon_p\sqrt{|\mathcal{S}|})+S_r(1-\gamma)}{D^2(1-\gamma)^2}\|\pi_1-\pi_0\|^2|\alpha'-\alpha|,\nonumber
\end{align}
where (a) uses Assumptions \ref{assum:sensitive}-\ref{assum:smooth_pr}, $\|\nabla_r J_{\lambda}(\cdot,\cdot,\cdot,\cdot)\|\le \frac{1}{1-\gamma}$ (implied by Eq. (\ref{eq:Lr})) as well as Eqs. (\ref{eq:lr_all}), (\ref{eq:dpi_alpha_norm}) and (\ref{eq:d2pi_alpha_norm}), (b) uses Eq. (\ref{eq:dpi_alpha_norm}) and $\|x\|_1\le \sqrt{d}\|x\|$ for any $x\in\mathbb{R}^d$, (c) uses Assumption \ref{assum:sensitive}, and (d) uses $D,\gamma\in [0,1]$. The inequality above implies that $w(\alpha)$ is $\mu_2\|\pi_1-\pi_0\|^2$-Lipschitz smooth with the constant $\mu_2$ defined as follows. 
\begin{align}
\mu_2=\frac{4\epsilon_r(\sqrt{|\mathcal{A}|}+\epsilon_p\sqrt{|\mathcal{S}|})+S_r(1-\gamma)}{D^2(1-\gamma)^2}\label{eq:mu2}
\end{align}
Therefore, 
\begin{align}
&V_{\lambda,\pi_{\alpha}}^{\pi_{\alpha}}-\alpha V_{\lambda,\pi_1}^{\pi_1}-(1-\alpha)V_{\lambda,\pi_0}^{\pi_0}\nonumber\\
=&J_{\lambda}(\pi_{\alpha},\pi_{\alpha},p_{\alpha},r_{\alpha})-\alpha J_{\lambda}(\pi_1,\pi_1,p_1,r_1)-(1-\alpha)J_{\lambda}(\pi_0,\pi_0,p_0,r_0)\nonumber\\
\overset{(a)}{\ge}&\alpha J_{\lambda}(\pi_1,\pi_1,p_1,r_{\alpha})+(1-\alpha) J_{\lambda}(\pi_0,\pi_0,p_0,r_{\alpha})+\frac{\mu_1\alpha(1-\alpha)}{2}\|\pi_1-\pi_0\|^2\nonumber\\
&-\alpha J_{\lambda}(\pi_1,\pi_1,p_1,r_1)-(1-\alpha)J_{\lambda}(\pi_0,\pi_0,p_0,r_0)\nonumber\\
=&w(\alpha)-\alpha w(1)-(1-\alpha)w(0)+\frac{\mu_1\alpha(1-\alpha)}{2}\|\pi_1-\pi_0\|^2\nonumber\\
\overset{(b)}{\ge}&\frac{(\mu_1-\mu_2)\alpha(1-\alpha)}{2}\|\pi_1-\pi_0\|^2\nonumber\\
\overset{(c)}{=}&\frac{\mu\alpha(1-\alpha)}{2}\|\pi_1-\pi_0\|^2,\label{eq:J_sc_concave}
\end{align}
where (a) uses Eq. (\ref{eq:Jstrong_fixr}) with $r$ replaced by $r_{\alpha}$, (b) uses the fact proved above that $w(\alpha)$ is $\mu_2\|\pi_1-\pi_0\|^2$-Lipschitz smooth, and (c) defines the following constant $\mu$.
\begin{align}
\mu\overset{\rm def}{=}&\mu_1-\mu_2\nonumber\\
\overset{(a)}{=}&\frac{D\lambda}{1-\gamma}-\frac{6\gamma|\mathcal{S}|(1+\lambda\log|\mathcal{A}|)}{D(1-\gamma)^3}\big[\epsilon_p\big(\sqrt{|\mathcal{A}|}+\gamma\epsilon_p\sqrt{|\mathcal{S}|}\big)+S_p(1-\gamma)\big]\nonumber\\
&-\frac{S_r(1-\gamma)+4\epsilon_r(\sqrt{|\mathcal{A}|}+\epsilon_p\sqrt{|\mathcal{S}|})}{D^2(1-\gamma)^2},\label{eq:mu}
\end{align}
where (a) uses Eqs. (\ref{eq:mu1}) and (\ref{eq:mu2}). Rearranging Eq. (\ref{eq:J_sc_concave}), we obtain that
\begin{align}
\frac{V_{\lambda,\pi_{\alpha}}^{\pi_{\alpha}}-V_{\lambda,\pi_0}^{\pi_0}}{\alpha}\ge V_{\lambda,\pi_1}^{\pi_1}-V_{\lambda,\pi_0}^{\pi_0}+\frac{\mu(1-\alpha)}{2}\|\pi_1-\pi_0\|^2.\nonumber
\end{align}
Letting $\alpha\to +0$ above, we can prove the conclusion as follows. 
\begin{align}
&V_{\lambda,\pi_1}^{\pi_1}-V_{\lambda,\pi_0}^{\pi_0}+\frac{\mu}{2}\|\pi_1-\pi_0\|^2\nonumber\\
\le& \Big[\frac{d}{d\alpha}V_{\lambda,\pi_{\alpha}}^{\pi_{\alpha}}\Big]\Big|_{\alpha=0}\nonumber\\
\le& \sum_{s,a}\frac{\partial V_{\lambda,\pi_0}^{\pi_0}}{\partial \pi_0(s,a)}\Big[\frac{d}{d\alpha}\pi_{\alpha}(a|s)\Big]\Big|_{\alpha=0}\nonumber\\
\overset{(a)}{=}&\sum_s \frac{d_1(s)}{d_0(s)}\sum_a\frac{\partial V_{\lambda,\pi_0}^{\pi_0}}{\partial \pi_0(s,a)}[\pi_1(a|s)-\pi_0(a|s)]\nonumber\\
\le& \sum_s \frac{d_1(s)}{d_0(s)}\Big[\max_{a'}\frac{\partial V_{\lambda,\pi_0}^{\pi_0}}{\partial \pi_0(s,a')}-\sum_a\pi_0(a|s)\frac{\partial V_{\lambda,\pi_0}^{\pi_0}}{\partial \pi_0(s,a)}\Big]\nonumber\\
\overset{(b)}{\le}&D^{-1}\sum_{s,a}\frac{\partial V_{\lambda,\pi_0}^{\pi_0}}{\partial \pi_0(s,a)}[\pi_0^*(a|s)-\pi_0(a|s)]\nonumber\\
\le& D^{-1}\max_{\pi\in\Pi}\big\langle \nabla_{\pi_0}V_{\lambda,\pi_0}^{\pi_0}, \pi-\pi_0\big\rangle,\nonumber
\end{align}
where (a) uses Eq. (\ref{eq:dpi_alpha}), and (b) uses Assumption \ref{assum:dmin} as well as the following Eq. (\ref{eq:greedy}) where $\pi_0^*\in\Pi$ is defined as $\pi_0^*(a^*|s)=1$ for a certain $a^*\in{\arg\max}_{a'}\frac{\partial V_{\lambda,\pi_0}^{\pi_0}}{\partial \pi_0(s,a')}$ and $\pi_0^*(a'|s)=0$ for $a'\ne a^*$.
\begin{align}
    \sum_a\pi_0^*(a|s)\frac{\partial V_{\lambda,\pi_0}^{\pi_0}}{\partial \pi_0(s,a)}=\max_{a'}\frac{\partial V_{\lambda,\pi_0}^{\pi_0}}{\partial \pi_0(s,a')}\ge \sum_a\pi_0(a|s)\frac{\partial V_{\lambda,\pi_0}^{\pi_0}}{\partial \pi_0(s,a)}.\label{eq:greedy}
\end{align} 

\section{Proof of Corollary \ref{coro:stat2PO}}
Based on Theorem \ref{thm:ToOpt}, Eq. (\ref{eq:mu}) holds for any $\pi_0,\pi_1\in\Pi$ as repeated below. 
\begin{align}
V_{\lambda,\pi_1}^{\pi_1}\le& V_{\lambda,\pi_0}^{\pi_0}+D^{-1}\max_{\pi\in\Pi}\big\langle \nabla_{\pi_0}V_{\lambda,\pi_0}^{\pi_0},\pi-\pi_0\big\rangle-\frac{\mu}{2}\|\pi_1-\pi_0\|^2,\label{eq:ToOpt2}
\end{align}
In the above inequality, let $\pi_1\in{\arg\max}_{\pi\in\Pi} V_{\lambda,\pi}^{\pi}$ and $\pi_0=\pi$ is any a $D\epsilon$-stationary policy of interest. Then the inequality above becomes
\begin{align}
\max_{\Tilde{\pi}\in\Pi} V_{\lambda,\Tilde{\pi}}^{\Tilde{\pi}} \le& V_{\lambda,\pi}^{\pi}+D^{-1}\cdot D\epsilon-\frac{\mu}{2}\|\pi_1-\pi\|^2 \overset{\rm (a)}{\le}V_{\lambda,\pi}^{\pi}+\epsilon+|\mu||\mathcal{S}|,\nonumber
\end{align}
where (a) uses Lemma \ref{lemma:pi_diameter}. This implies that $\max_{\Tilde{\pi}\in\Pi} V_{\lambda,\Tilde{\pi}}^{\Tilde{\pi}}-V_{\lambda,\pi}^{\pi}\le \epsilon+|\mu||\mathcal{S}|$, that is, the $D\epsilon$-stationary policy $\pi$ is also an $(\epsilon+|\mu||\mathcal{S}|)$-PO policy. 

If $\mu\ge 0$, the inequality above further implies that $\max_{\Tilde{\pi}\in\Pi} V_{\lambda,\Tilde{\pi}}^{\Tilde{\pi}}-V_{\lambda,\pi}^{\pi}\le \epsilon$, that is, the $D\epsilon$-stationary policy $\pi$ is also an $\epsilon$-PO policy. 

Furthermore, suppose $\mu>0$ and there are two PO policies $\pi_0,\pi_1\in\Pi$, which should satisfy
\begin{align}
&V_{\lambda,\pi_1}^{\pi_1}=V_{\lambda,\pi_0}^{\pi_0}=\max_{\pi\in\Pi}V_{\lambda,\pi}^{\pi},\nonumber\\
&\max_{\pi\in\Pi}\big\langle \nabla_{\pi_0}V_{\lambda,\pi_0}^{\pi_0},\pi-\pi_0\big\rangle=0.\nonumber
\end{align}
Substituting the two equalities above into Eq. (\ref{eq:ToOpt}), we obtain that $\frac{\mu}{2}\|\pi_1-\pi_0\|^2\le 0$, which along with $\mu>0$ implies $\pi_1=\pi_0$, that is, the PO policy is unique. 

\section{Proof of Theorem \ref{thm:pi_ge}}\label{sec:proof_thm:pi_ge}
For any $\pi\in\Pi$, $p\in\mathcal{P}$, $r\in\mathcal{R}$, we have
\begin{align}
\frac{\partial J_{\lambda}(\pi,\pi,p,r)}{\partial\pi(a|s)}\overset{(a)}{=}&\frac{d_{\pi,p}(s)[Q_{\lambda}(\pi,\pi,p,r;s,a)-\lambda]}{1-\gamma}\nonumber\\
\overset{(b)}{=}&\frac{d_{\pi,p}(s)}{1-\gamma}\Big[r(s,a)-\lambda-\lambda\log\pi(a|s)+\gamma\sum_{s'}p(s'|s,a)V_{\lambda}(\pi,p,r;s')\Big],\label{eq:pi_grad}
\end{align}
where (a) uses Eqs. (\ref{eq:dJ5}), and (b) uses Eq. (\ref{eq:Qtau}). 

Then we have
\begin{align}
&\nabla_{\pi} J_{\lambda}(\pi,\pi,p,r)^{\top}(\pi'-\pi)\nonumber\\
=&\sum_s\Big[\frac{\partial J_{\lambda}(\pi,\pi,p,r)}{\partial\pi[a_{\max}(s)|s]}\big(\pi'[a_{\max}(s)|s]-\pi[a_{\max}(s)|s]\big) \nonumber\\
&+ \frac{\partial J_{\lambda}(\pi,\pi,p,r)}{\partial\pi[a_{\min}(s)|s]}\big(\pi'[a_{\min}(s)|s]-\pi[a_{\min}(s)|s]\big)\Big]\nonumber\\
=&\sum_s\Big\{\frac{d_{\pi,p}(s)}{1-\gamma}\big(\pi[a_{\max}(s)|s]-\pi[a_{\min}(s)|s]\big)\Big[r[s,a_{\min}(s)]-r[s,a_{\max}(s)]\nonumber\\
&+\lambda\log\frac{\pi[a_{\max}(s)|s]}{\pi[a_{\min}(s)|s]}+\gamma\sum_{s'}[p(s'|s,a_{\min}(s))-p(s'|s,a_{\max}(s))]V_{\lambda}(\pi,p,r;s')\Big]\Big\}\nonumber\\
\overset{(a)}{\ge}&\frac{1}{1-\gamma}\max_s\Big\{\!\big(\pi[a_{\max}(s)|s]\!-\!\pi[a_{\min}(s)|s]\big)\Big[\lambda\log\frac{\pi[a_{\max}(s)|s]}{\pi[a_{\min}(s)|s]}\!-\!1\!-\!\frac{\gamma(1+\lambda\log|\mathcal{A}|)}{1-\gamma}\Big]\!\Big\},\label{eq:inprod_ge}
\end{align}
where (a) uses $\pi[a_{\max}(s)|s]-\pi[a_{\min}(s)|s]\ge0$, $r(a|s)\in [0,1]$, $p(s'|s,a)\in [0,1]$ for any $s,a,s'$ and Lemma \ref{lemma:Jrange}. 

Consider the following two cases.

(Case I) If $\pi[a_{\min}(s)|s]\ge \frac{1}{2}\pi[a_{\max}(s)|s]$, then as $\pi[a_{\max}(s)|s]\ge \frac{1}{|\mathcal{A}|}$, we have $\pi[a_{\min}(s)|s]\ge \frac{1}{2|\mathcal{A}|}$. 

(Case II) $\pi[a_{\min}(s)|s]<\frac{1}{2}\pi[a_{\max}(s)|s]$, then as $\pi[a_{\max}(s)|s]\ge \frac{1}{|\mathcal{A}|}$, Eq. (\ref{eq:inprod_ge}) implies that
\begin{align}
&\nabla_{\pi} J_{\lambda}(\pi,\pi,p,r)^{\top}(\pi'-\pi)\nonumber\\
\ge&\max_s\Big\{\frac{\pi[a_{\max}(s)|s]}{2(1-\gamma)}\Big[\lambda\log\frac{1}{|\mathcal{A}|\pi[a_{\min}(s)|s]}-\frac{1+\gamma\lambda\log|\mathcal{A}|}{1-\gamma}\Big]\Big\}\nonumber\\
\ge&-\frac{1}{2|\mathcal{A}|(1-\gamma)}\Big[\lambda\log\big(|\mathcal{A}|\min_s\pi[a_{\min}(s)|s]\big)+\frac{1+\gamma\lambda\log|\mathcal{A}|}{1-\gamma}\Big], 
\end{align}
which further implies that for any $s\in\mathcal{S}$ and $a\in\mathcal{A}$, we have
\begin{align}
\pi(a|s)\ge& \pi[a_{\min}(s)|s]\nonumber\\
\ge&\frac{1}{|\mathcal{A}|}\exp\Big[-\frac{1/\lambda+\gamma\log|\mathcal{A}|}{1-\gamma}-\frac{2|\mathcal{A}|}{\lambda}(1-\gamma)\nabla_{\pi} J_{\lambda}(\pi,\pi,p,r)^{\top}(\pi'-\pi)\Big]\nonumber\\
\ge&\frac{1}{2|\mathcal{A}|^{1/(1-\gamma)}}\exp\Big[-\frac{1}{\lambda(1-\gamma)}-\frac{2|\mathcal{A}|}{\lambda}(1-\gamma)\nabla_{\pi} J_{\lambda}(\pi,\pi,p,r)^{\top}(\pi'-\pi)\Big],\label{eq:pi_ge3}
\end{align}
Note that in the two cases above, Eq. (\ref{eq:pi_ge3}) always holds. 

Furthermore, if Assumption \ref{assum:sensitive} holds and $p_{\pi}$, $r_{\pi}$ are differentiable functions of $\pi$, then we have
\begin{align}
&\big\|\nabla_{\pi} J_{\lambda}(\pi,\pi,p_{\pi},r_{\pi})-\nabla_{\pi} J_{\lambda}(\pi,\pi,p_{\Tilde{\pi}},r_{\Tilde{\pi}})|_{\Tilde{\pi}=\pi}\big\|\nonumber\\
=&\big\|\nabla_p J_{\lambda}(\pi,\pi,p_{\pi},r_{\pi}) \nabla_{\pi}p_{\pi}+\nabla_r J_{\lambda}(\pi,\pi,p_{\pi},r_{\pi}) \nabla_{\pi}r_{\pi}\big\|\nonumber\\
\le&\big\|\nabla_p J_{\lambda}(\pi,\pi,p_{\pi},r_{\pi}) \big\|\big\|\nabla_{\pi}p_{\pi}\big\|+\big\|\nabla_r J_{\lambda}(\pi,\pi,p_{\pi},r_{\pi}) \big\|\big\|\nabla_{\pi}r_{\pi}\big\|\nonumber\\
\overset{(a)}{\le}&\frac{\epsilon_p\sqrt{|\mathcal{S}|}(1+\lambda\log|\mathcal{A}|)}{(1-\gamma)^2}+\frac{\epsilon_r}{1-\gamma},\label{eq:dJ_diff}
\end{align}
where (a) uses Assumption \ref{assum:sensitive} as well as Eqs. (\ref{eq:Lp}) and (\ref{eq:Lr}). Therefore, 
\begin{align}
&\big[\nabla_{\pi} J_{\lambda}(\pi,\pi,p_{\Tilde{\pi}},r_{\Tilde{\pi}})|_{\Tilde{\pi}=\pi}\big]^{\top}(\pi'-\pi)\nonumber\\
=&\nabla_{\pi} J_{\lambda}(\pi,\pi,p_{\pi},r_{\pi})^{\top}(\pi'-\pi)-\big[\nabla_{\pi} J_{\lambda}(\pi,\pi,p_{\pi},r_{\pi})-\nabla_{\pi} J_{\lambda}(\pi,\pi,p_{\Tilde{\pi}},r_{\Tilde{\pi}})|_{\Tilde{\pi}=\pi}\big]^{\top}(\pi'-\pi)\nonumber\\
\le&\nabla_{\pi} J_{\lambda}(\pi,\pi,p_{\pi},r_{\pi})^{\top}(\pi'-\pi)+\big\|\nabla_{\pi} J_{\lambda}(\pi,\pi,p_{\pi},r_{\pi})-\nabla_{\pi} J_{\lambda}(\pi,\pi,p_{\Tilde{\pi}},r_{\Tilde{\pi}})|_{\Tilde{\pi}=\pi}\big\|\|\pi'-\pi\|\nonumber\\
\overset{(a)}{\le}&\nabla_{\pi} J_{\lambda}(\pi,\pi,p_{\pi},r_{\pi})^{\top}(\pi'-\pi)+\sqrt{2|\mathcal{S}|}\Big(\frac{\epsilon_p\sqrt{|\mathcal{S}|}(1+\lambda\log|\mathcal{A}|)}{(1-\gamma)^2}+\frac{\epsilon_r}{1-\gamma}\Big),\label{eq:inprod_le}
\end{align}
where (a) uses Eq. (\ref{eq:dJ_diff}) and Lemma \ref{lemma:pi_diameter}. Substituting $p=p_{\pi}$, $r=r_{\pi}$ and then Eq. (\ref{eq:inprod_le}) into Eq. (\ref{eq:pi_ge3}), we can prove Eq. (\ref{eq:pi_ge2}) as follows. 
\begin{align}
\pi(a|s)\ge&\frac{1}{2|\mathcal{A}|^{1/(1-\gamma)}}\exp\Big\{-\frac{1}{\lambda(1-\gamma)}-\frac{2|\mathcal{A}|}{\lambda}(1-\gamma)\cdot\nonumber\\
&\Big[\nabla_{\pi} J_{\lambda}(\pi,\pi,p_{\pi},r_{\pi})^{\top}(\pi'-\pi)+\sqrt{2|\mathcal{S}|}\Big(\frac{\epsilon_p\sqrt{|\mathcal{S}|}(1+\lambda\log|\mathcal{A}|)}{(1-\gamma)^2}+\frac{\epsilon_r}{1-\gamma}\Big)\Big]\Big\}\nonumber\\
=&\pi_{\min}\exp\Big[-\frac{2|\mathcal{A}|}{\lambda}(1-\gamma)\langle\nabla_{\pi} V_{\lambda,\pi}^{\pi},\pi'-\pi\rangle\Big],\nonumber
\end{align}
where the $=$ uses $V_{\lambda,\pi}^{\pi}=J_{\lambda}(\pi,\pi,p_{\pi},r_{\pi})$ and $\pi_{\min}$ defined as follows. 
\begin{align}
\pi_{\min}&\overset{\rm def}{=}\frac{1}{2|\mathcal{A}|^{1/(1-\gamma)}}\exp\Big\{-\frac{1}{\lambda(1-\gamma)}-\frac{2|\mathcal{A}|\sqrt{2|\mathcal{S}|}}{\lambda}\Big[\frac{\epsilon_p\sqrt{|\mathcal{S}|}(1+\lambda\log|\mathcal{A}|)}{1-\gamma}+\epsilon_r\Big]\Big\},\label{eq:pi_min}
\end{align}

\section{Proof of Theorem \ref{thm:V_Lip}}\label{sec:proof_thm:V_Lip}
For any policies $\pi,\pi'$, we have
\begin{align}
&|V_{\lambda,\pi'}^{\pi'}-V_{\lambda,\pi}^{\pi}|\nonumber\\
\le&|J_{\lambda}(\pi',p_{\pi'},r_{\pi'})-J_{\lambda}(\pi,p_{\pi},r_{\pi})|\nonumber\\ 
\le& |J_{\lambda}(\pi',p_{\pi'},r_{\pi'})-J_{\lambda}(\pi',p_{\pi'},r_{\pi})|+|J_{\lambda}(\pi',p_{\pi'},r_{\pi})-J_{\lambda}(\pi',p_{\pi},r_{\pi})|\nonumber\\
&+|J_{\lambda}(\pi',p_{\pi},r_{\pi})-J_{\lambda}(\pi,p_{\pi},r_{\pi})|\nonumber\\ 
\overset{(a)}{\le}& \frac{\|r_{\pi'}-r_{\pi}\|}{1-\gamma}+L_p\|p_{\pi'}-p_{\pi}\|+L_{\pi}\max_{s}\|\log\pi'(\cdot|s)-\log\pi(\cdot|s)\|\nonumber\\
\overset{(b)}{\le}& \Big(L_p\epsilon_p+\frac{\epsilon_r}{1-\gamma}\Big)\|\pi'-\pi\|+L_{\pi}\sqrt{\sum_{s}\|\log\pi'(\cdot|s)-\log\pi(\cdot|s)\|^2}\nonumber\\
\overset{(c)}{\le}& \Big(L_p\epsilon_p+\frac{\epsilon_r}{1-\gamma}\Big)\|\log\pi'-\log\pi\|+L_{\pi}\|\log\pi'-\log\pi\|\nonumber\\
\overset{(d)}{=}& L_{\lambda}\|\log\pi'-\log\pi\|,\label{eq:V_Lip_logpi}
\end{align}
where (a) uses Eqs. (\ref{eq:Lpi}), (\ref{eq:Lp}) and (\ref{eq:Lr}), (b) uses Assumption \ref{eq:sensitive}, (c) uses $|\log y-\log x|\le |y-x|$ for any $x,y\in\mathbb{R}$, and (d) defines the following constant. 
\begin{align}
L_{\lambda}=L_p\epsilon_p+\frac{\epsilon_r}{1-\gamma}+L_{\pi}=\frac{\sqrt{|\mathcal{A}|}(2-\gamma+\gamma\lambda\log|\mathcal{A}|)+\epsilon_p\sqrt{|\mathcal{S}|}(1+\lambda\log|\mathcal{A}|)+\epsilon_r(1-\gamma)}{(1-\gamma)^2}.\nonumber
\end{align}
\begin{align}
L_{\lambda}\!\overset{\rm def}{=}&L_p\epsilon_p+\frac{\epsilon_r}{1-\gamma}+L_{\pi}=\frac{\sqrt{|\mathcal{A}|}(2\!-\!\gamma\!+\!\gamma\lambda\log|\mathcal{A}|)\!+\!\epsilon_p\sqrt{|\mathcal{S}|}(1\!+\!\lambda\log|\mathcal{A}|)}{(1-\gamma)^2}+\frac{\epsilon_r}{1-\gamma}\label{eq:Ltau}
\end{align}

Note that for any $u,v\ge \Delta>0$, 
\begin{align}
|\log u-\log v|=&\log\max(u,v)-\log\min(u,v)\nonumber\\
=&\int_{\min(u,v)}^{\max(u,v)}\frac{1}{x}dx\le \frac{1}{\Delta}[\max(u,v)-\min(u,v)]=\frac{|u-v|}{\Delta}.\nonumber
\end{align}
Therefore, for any $\pi,\pi'\in\Pi_{\Delta}\overset{\rm def}{=}\{\pi\in\Pi:\pi(a|s)\ge\Delta\}$, we have
\begin{align}
\|\log\pi'-\log\pi\|^2=&\sum_{s,a}|\log\pi'(a|s)-\log\pi(a|s)|^2\nonumber\\
\le& \Delta^{-2}\sum_{s,a}|\pi'(a|s)-\pi(a|s)|^2=\Delta^{-2}\|\pi'-\pi\|^2.\nonumber
\end{align}
Substituting the above inequality into Eq. (\ref{eq:V_Lip_logpi}) proves the first inequality of Eq. (\ref{eq:V_Lip_logpi}). 

Next, we will prove the second inequality of Eq. (\ref{eq:V_Lip_logpi}) about the Lipschitz continuity of the following performative policy gradient.
\begin{align}
\nabla_{\pi}V_{\lambda,\pi}^{\pi}=&\nabla_{\pi}J_{\lambda}(\pi,\pi,p_{\pi},r_{\pi})\nonumber\\
=&\nabla_{\pi}J_{\lambda}(\pi,\pi,p_{\Tilde{\pi}},r_{\Tilde{\pi}})|_{\Tilde{\pi}=\pi}+(\nabla_{\pi}p_{\pi})\nabla_{p_{\pi}}J_{\lambda}(\pi,\pi,p_{\pi},r_{\pi})+(\nabla_{\pi}r_{\pi})\nabla_{r_{\pi}}J_{\lambda}(\pi,\pi,p_{\pi},r_{\pi}).\nonumber
\end{align}
For any $\pi, \pi'\in\Pi_{\Delta}$, we have 
\begin{align}
&\|\nabla_{\pi'}V_{\lambda,\pi'}^{\pi'}-\nabla_{\pi}V_{\lambda,\pi}^{\pi}\|\nonumber\\
\le&\big\|\nabla_{\pi'}J_{\lambda}(\pi',\pi',p_{\Tilde{\pi}},r_{\Tilde{\pi}})|_{\Tilde{\pi}=\pi'}-\nabla_{\pi}J_{\lambda}(\pi,\pi,p_{\Tilde{\pi}},r_{\Tilde{\pi}})|_{\Tilde{\pi}=\pi}\big\|\nonumber\\
&+\|\nabla_{\pi'}p_{\pi'}\|\cdot\|\nabla_{p_{\pi'}}J_{\lambda}(\pi',\pi',p_{\pi'},r_{\pi'})-\nabla_{p_{\pi}}J_{\lambda}(\pi,\pi,p_{\pi},r_{\pi})\|\nonumber\\
&+\|\nabla_{p_{\pi}}J_{\lambda}(\pi,\pi,p_{\pi},r_{\pi})\|\cdot\|\nabla_{\pi'}p_{\pi'}-\nabla_{\pi}p_{\pi}\|\nonumber\\
&+\|\nabla_{\pi'}r_{\pi'}\|\cdot\|\nabla_{r_{\pi'}}J_{\lambda}(\pi',\pi',p_{\pi'},r_{\pi'})-\nabla_{r_{\pi}}J_{\lambda}(\pi,\pi,p_{\pi},r_{\pi})\|\nonumber\\
&+\|\nabla_{r_{\pi}}J_{\lambda}(\pi,\pi,p_{\pi},r_{\pi})\|\cdot\|\nabla_{\pi'}r_{\pi'}-\nabla_{\pi}r_{\pi}\|\nonumber\\
\overset{(a)}{\le}&\Big(\frac{|\mathcal{A}|(1+2\lambda\log|\mathcal{A}|)}{(1-\gamma)^2}+\gamma L_{\pi}\Big)\max_{s}\|\log\pi'(\cdot|s)-\log\pi(\cdot|s)\| \nonumber\\
&+\Big[\frac{2(1+\lambda\log|\mathcal{A}|)}{(1-\gamma)^2}+\gamma L_p\Big]\sqrt{|\mathcal{S}||\mathcal{A}|}\|p_{\pi'}-p_{\pi}\| +\frac{\sqrt{|\mathcal{A}|}\|r_{\pi'}-r_{\pi}\|_{\infty}}{1-\gamma}\nonumber\\
&+\epsilon_p\Big[\ell_{\pi}\max_{s}\|\log\pi'(\cdot|s)-\log\pi(\cdot|s)\|+\ell_{p}\|p_{\pi'}-p_{\pi}\|+\frac{2-\gamma}{1-\gamma}\sqrt{|\mathcal{S}|}\|r_{\pi'}-r_{\pi}\|_{\infty}\Big]\nonumber\\
&+L_pS_p\|\pi'-\pi\| +\frac{\gamma\epsilon_r}{(1-\gamma)^2}\big(\max_{s}\|\pi'(\cdot|s)-\pi(\cdot|s)\|_1+\max_{s,a} \|p_{\pi'}(\cdot|s,a)-p_{\pi}(\cdot|s,a)\|_1\big)\nonumber\\
&+\frac{S_r}{1-\gamma}\|\pi'-\pi\|\nonumber\\
\overset{(b)}{\le}&\Big(\frac{|\mathcal{A}|(1+2\lambda\log|\mathcal{A}|)}{\Delta(1-\gamma)^2}+\frac{\gamma L_{\pi}}{\Delta}\Big)\|\pi'-\pi\| +\epsilon_p\sqrt{|\mathcal{S}||\mathcal{A}|}\Big[\frac{2(1+\lambda\log|\mathcal{A}|)}{(1-\gamma)^2}+\gamma L_p\Big]\|\pi'-\pi\|\nonumber\\
&+\frac{\epsilon_r\sqrt{|\mathcal{A}|}\|\pi'-\pi\|}{1-\gamma} +\epsilon_p\Big[\frac{\ell_{\pi}}{\Delta}\|\pi'-\pi\|+\ell_{p}\epsilon_p\|\pi'-\pi\|+\frac{2-\gamma}{1-\gamma}\epsilon_r\sqrt{|\mathcal{S}|}\|\pi'-\pi\|\Big]\nonumber\\
&+L_pS_p\|\pi'-\pi\| +\frac{\gamma\epsilon_r}{(1-\gamma)^2}\big(\sqrt{|\mathcal{S}|}\|\pi'-\pi\|+\epsilon_p\sqrt{|\mathcal{S}|}\|\pi'-\pi\|\big) +\frac{S_r}{1-\gamma}\|\pi'-\pi\|\nonumber\\
\overset{(c)}{\le}&\Big(\frac{|\mathcal{A}|(1+2\lambda\log|\mathcal{A}|)}{\Delta(1-\gamma)^2}+\frac{\gamma L_{\pi}}{\Delta}\Big)\|\pi'-\pi\| +\frac{\epsilon_p}{\Delta}\sqrt{\frac{|\mathcal{S}|}{|\mathcal{A}|}}\Big[\frac{2(1+\lambda\log|\mathcal{A}|)}{(1-\gamma)^2}+\gamma L_p\Big]\|\pi'-\pi\|\nonumber\\
&+\frac{\epsilon_r\|\pi'-\pi\|}{\Delta\sqrt{|\mathcal{A}|}(1-\gamma)} +\frac{\epsilon_p}{\Delta}\Big[\ell_{\pi} +\frac{\ell_{p}\epsilon_p}{|\mathcal{A}|}+\frac{2-\gamma}{|\mathcal{A}|(1-\gamma)}\epsilon_r\sqrt{|\mathcal{S}|}\Big]\|\pi'-\pi\|\nonumber\\
& +\frac{\gamma\epsilon_r\sqrt{|\mathcal{S}|}(1+\epsilon_p)}{\Delta|\mathcal{A}|(1-\gamma)^2}\|\pi'-\pi\| +\frac{L_pS_p+S_r/(1-\gamma)}{\Delta|\mathcal{A}|}\|\pi'-\pi\|\nonumber\\
\overset{(d)}{\le}&\Big(\frac{|\mathcal{A}|(1+2\lambda\log|\mathcal{A}|)}{\Delta(1-\gamma)^2}+\frac{\gamma\sqrt{|\mathcal{A}|}(2-\gamma+\gamma\lambda\log|\mathcal{A}|)}{\Delta(1-\gamma)^2}\Big)\|\pi'-\pi\|\nonumber\\
&+\frac{\epsilon_p}{\Delta}\sqrt{\frac{|\mathcal{S}|}{|\mathcal{A}|}}\Big[\frac{2(1+\lambda\log|\mathcal{A}|)}{(1-\gamma)^2}+\frac{\gamma\sqrt{|\mathcal{S}|}(1+\lambda\log|\mathcal{A}|)}{(1-\gamma)^2}\Big]\|\pi'-\pi\|\nonumber\\
&+\frac{\epsilon_p}{\Delta}\Big[\frac{\sqrt{|\mathcal{S}||\mathcal{A}|}(2+3\gamma\lambda\log|\mathcal{A}|)}{(1-\gamma)^3} +\frac{2\epsilon_p\gamma|\mathcal{S}|(1+\lambda\log|\mathcal{A}|)}{|\mathcal{A}|(1-\gamma)^3}+\frac{2-\gamma}{|\mathcal{A}|(1-\gamma)}\epsilon_r\sqrt{|\mathcal{S}|}\Big]\|\pi'-\pi\|\nonumber\\
&+\frac{\epsilon_r\sqrt{|\mathcal{A}|}(1-\gamma)+\gamma\epsilon_r\sqrt{|\mathcal{S}|}(1+\epsilon_p)}{\Delta|\mathcal{A}|(1-\gamma)^2}\|\pi'-\pi\| \nonumber\\
&+\frac{S_p\sqrt{|\mathcal{S}|}(1+\lambda\log|\mathcal{A}|)+S_r(1-\gamma)}{\Delta|\mathcal{A}|(1-\gamma)^2}\|\pi'-\pi\|\nonumber\\
\le&\frac{3|\mathcal{A}|(1+\lambda\log|\mathcal{A}|)}{\Delta(1-\gamma)^2}\|\pi'-\pi\|+\frac{\epsilon_p\sqrt{|\mathcal{S}||\mathcal{A}|}(5+6\lambda\log|\mathcal{A}|)}{\Delta(1-\gamma)^3}\|\pi'-\pi\|\nonumber\\
&+\frac{\epsilon_r\big[\sqrt{|\mathcal{A}|}(1-\gamma)+\sqrt{|\mathcal{S}|}(\gamma+2\epsilon_p)\big]+S_p\sqrt{|\mathcal{S}|}(1+\lambda\log|\mathcal{A}|)+S_r(1-\gamma)}{\Delta|\mathcal{A}|(1-\gamma)^2}\|\pi'-\pi\|\label{eq:dV_Lip_tmp},
\end{align} 
where (a) uses Eqs. (\ref{eq:Lp}), (\ref{eq:Lr}) and (\ref{eq:lp_all})-(\ref{eq:Lip_pi}) as well as Assumptions \ref{assum:sensitive}-\ref{assum:smooth_pr}, and (b) uses the following bounds for any $\pi,\pi'\in\Delta$, (c) uses $\Delta\le |\mathcal{A}|^{-1}$ (since for any $\pi\in\Pi_{\Delta}$, $1=\sum_{a}\pi(a|s)\ge \Delta|\mathcal{A}|$), (d) uses $L_{\pi}:=\frac{\sqrt{|\mathcal{A}|}(2-\gamma+\gamma\lambda\log|\mathcal{A}|)}{(1-\gamma)^2}$, $L_p:=\frac{\sqrt{|\mathcal{S}|}(1+\lambda\log|\mathcal{A}|)}{(1-\gamma)^2}$, $\ell_{\pi}:=\frac{\sqrt{|\mathcal{S}||\mathcal{A}|}(2+3\gamma\lambda\log|\mathcal{A}|)}{(1-\gamma)^3}$ and $\ell_p:=\frac{2\gamma|\mathcal{S}|(1+\lambda\log|\mathcal{A}|)}{(1-\gamma)^3}$ defined in Lemma \ref{lemma:J_lip}, (e) uses $\ell_{\lambda}$ defined by Eq. (\ref{eq:ell_tau}). 
\begin{align}
\max_{s}\|\log\pi'(\cdot|s)-\log\pi(\cdot|s)\|\le&\Delta^{-1}\max_{s}\|\pi'(\cdot|s)-\pi(\cdot|s)\|\le \Delta^{-1}\|\pi'-\pi\|,\nonumber\\
\|p_{\pi'}-p_{\pi}\|\overset{(a)}{\le}&\epsilon_p\|\pi'-\pi\|,\nonumber\\
\|r_{\pi'}-r_{\pi}\|_{\infty}\le&\|r_{\pi'}-r_{\pi}\| \overset{(a)}{\le} \epsilon_r\|\pi'-\pi\|,\nonumber\\
\max_{s}\|\pi'(\cdot|s)-\pi(\cdot|s)\|_1\le&\sqrt{|\mathcal{S}|}\max_{s}\|\pi'(\cdot|s)-\pi(\cdot|s)\|\le \sqrt{|\mathcal{S}|}\|\pi'-\pi\|,\nonumber\\
\max_{s,a} \|p_{\pi'}(\cdot|s,a)-p_{\pi}(\cdot|s,a)\|_1\le& \sqrt{|\mathcal{S}|}\max_{s,a} \|p_{\pi'}(\cdot|s,a)-p_{\pi}(\cdot|s,a)\|\nonumber\\
\le&\sqrt{|\mathcal{S}|}\|p_{\pi'}-p_{\pi}\|\overset{(a)}{\le}\epsilon_p\sqrt{|\mathcal{S}|}\|\pi'-\pi\|.\nonumber
\end{align}
Here, (a) uses Assumption \ref{assum:sensitive}. Finally, define the Lipschitz constant $\ell_{\lambda}$ as follows and thus Eq. (\ref{eq:dV_Lip_tmp}) implies the second inequality of Eq. (\ref{eq:V_Lip_logpi}) that $\|\nabla_{\pi'}V_{\lambda,\pi'}^{\pi'}-\nabla_{\pi}V_{\lambda,\pi}^{\pi}\|\le\frac{\ell_{\lambda}}{\Delta}\|\pi'-\pi\|$.
\begin{align}
\ell_{\lambda}\overset{\rm def}{=}&\frac{3|\mathcal{A}|(1+\lambda\log|\mathcal{A}|)}{(1-\gamma)^2}+\frac{\epsilon_p\sqrt{|\mathcal{S}||\mathcal{A}|}(5+6\lambda\log|\mathcal{A}|)}{(1-\gamma)^3}\nonumber\\
&+\frac{\epsilon_r\big[\sqrt{|\mathcal{A}|}(1-\gamma)+\sqrt{|\mathcal{S}|}(\gamma+2\epsilon_p)\big]}{|\mathcal{A}|(1-\gamma)^2}+\frac{S_p\sqrt{|\mathcal{S}|}(1+\lambda\log|\mathcal{A}|)+S_r(1-\gamma)}{|\mathcal{A}|(1-\gamma)^2}. \label{eq:ell_tau}
\end{align}

\section{Proof of Proposition \ref{prop:grad_err}}\label{sec:proof_prop_grad_err}
We prove the validity of the stochastic gradient (\ref{eq:0ppg}) first. For any $\pi\in\Pi_{\Delta}$, $s\in\mathcal{S}$ and $a\in\mathcal{A}$, we have $\pi(a|s)\ge \Delta$, so $\pi(a|s)\le 1-\Delta$ (since $\sum_{a'}\pi(a'|s)=1$). For any $u_i\in U_1$, we have $|u_i(a|s)|\le 1$. Therefore,
\begin{align}
(\pi\pm\delta u_i)(a|s) \ge \pi(a|s)-\delta|u_i(a|s)|\ge \Delta-\delta>0,
\end{align}
which means $\pi\pm\delta u_i\in\Pi$. Hence, $V_{\lambda,\pi'}^{\pi'}$ is well defined for $\pi'\in\{\pi+\delta u_i, \pi-\delta u_i\}$. 

Then we will prove the estimation error bound (\ref{eq:grad_err}). Based on Lemma \ref{lemma:orthoT}, there exists an orthogonal transformation $\mathcal{T}:\mathbb{R}^{|\mathcal{A}|}\!\to\! \mathcal{Z}_{|\mathcal{A}|-1}\!=\!\{z\!=\![z_1,\ldots,z_{|\mathcal{A}|}]\in\mathbb{R}^{|\mathcal{A}|}:\sum_i z_i\!=\!0\}$. 

Note that any $x\in\mathbb{R}^{|\mathcal{S}|(|\mathcal{A}|-1)}$ can be written as $x=[x_s]_{s\in\mathcal{S}}$, a concatenation of $|\mathcal{S}|$ vectors $x_s\in\mathbb{R}^{|\mathcal{A}|}$. Therefore, we can define the transformation $T: \mathbb{R}^{|\mathcal{S}|(|\mathcal{A}|-1)}\to \mathcal{L}_0\overset{\rm def}{=}\big\{u\in\mathbb{R}^{|\mathcal{S}||\mathcal{A}|}\!:u(\cdot|s)\in\mathcal{Z}_{|\mathcal{A}|-1}, \forall s\in\mathcal{S}\big\}$ as follows
\begin{align}
[T(x)](\cdot|s)=\mathcal{T}(x_s), \forall s\in\mathcal{S}
\end{align}
where $x_s\in\mathbb{R}^{|\mathcal{A}|}$ are extracted from $|\mathcal{A}|$ entries of $x=[x_s]_{s\in\mathcal{S}}$. For any $x=[x_s]_{s\in\mathcal{S}}, y=[y_s]_{s\in\mathcal{S}}\in\mathbb{R}^{|\mathcal{S}|(|\mathcal{A}|-1)}$ and $\alpha,\beta\in\mathbb{R}$, we can prove that $T$ is an orthogonal transformation as follows. 
\begin{align}
&[T(\alpha x+\beta y)](\cdot|s)=\mathcal{T}(\alpha x_s+\beta y_s)=\alpha\mathcal{T}(x_s)+\beta\mathcal{T}(y_s)=\alpha [T(x)](\cdot|s)+\beta [T(x)](\cdot|s)\nonumber\\
\Rightarrow &T(\alpha x+\beta y)=\alpha T(x)+\beta T(y).\nonumber
\end{align}
\begin{align}
\langle T(x),T(y)\rangle=&\sum_s\big\langle [T(x)](\cdot|s),[T(y)](\cdot|s)\big\rangle=\sum_s \langle\mathcal{T}(x_s),\mathcal{T}(y_s)\rangle=\sum_s \langle x_s,y_s\rangle=\langle x,y\rangle.\nonumber
\end{align}
Define the following set. 
\begin{align}
T^{-1}(\Pi_{\Delta}-|\mathcal{A}|^{-1})\overset{\rm def}{=}\{\pi\in\Pi_{\Delta}:T^{-1}(\pi-|\mathcal{A}|^{-1})\},\label{eq:Pi_inv}
\end{align}
where $\pi-|\mathcal{A}|^{-1}\in\mathbb{R}^{|\mathcal{S}||\mathcal{A}|}$ has entries $(\pi-|\mathcal{A}|^{-1})(a|s)=\pi(a|s)-|\mathcal{A}|^{-1}$, so $\pi-|\mathcal{A}|^{-1}\in\mathcal{L}_0$. Furthermore, since $\Pi_{\Delta}$ is a convex and compact set and $T^{-1}$ is an orthogonal transformation, $T^{-1}(\Pi_{\Delta}-|\mathcal{A}|^{-1})$ is a convex and compact subset of $\mathcal{L}_0$. 
 
Then for any $x\in T^{-1}(\Pi_{\Delta}-|\mathcal{A}|^{-1})$, we have $T(x)+|\mathcal{A}|^{-1}\in\Pi_{\Delta}$, so we can define the function $f_{\lambda}(x)\overset{\rm def}{=}V_{\lambda,T(x)+|\mathcal{A}|^{-1}}^{T(x)+|\mathcal{A}|^{-1}}$. 

Note that as $V_{\lambda,\pi}^{\pi}$ is a differentiable function of $\pi$, so for any $\pi'\in\Pi$ and fixed $\pi\in\Pi$ we have 
\begin{align}
\frac{V_{\lambda,\pi'}^{\pi'}-V_{\lambda,\pi}^{\pi}-\langle\nabla_{\pi}V_{\lambda,\pi}^{\pi}, \pi'-\pi\rangle}{\|\pi'-\pi\|}=& \frac{V_{\lambda,\pi'}^{\pi'}-V_{\lambda,\pi}^{\pi}-\langle{\rm proj}_{\mathcal{L}_0}(\nabla_{\pi}V_{\lambda,\pi}^{\pi}), \pi'-\pi\rangle}{\|\pi'-\pi\|}\nonumber\\
\to& 0\quad ({\rm~as~} \pi'\in\Pi {\rm~and~}\pi'\to\pi),\label{eq:pi_diffable}
\end{align}
where the above $=$ uses $\pi'-\pi\in\mathcal{L}_0$. Then, we can prove that $f_{\lambda}$ is differentiable with gradient $\nabla f_{\lambda}(x)=T^{-1}\big({\rm proj}_{\mathcal{L}_0}\nabla_{\pi}V_{\lambda,\pi}^{\pi}\big|_{\pi=T(x)+|\mathcal{A}|^{-1}}\big)$, since for any $x'\in T^{-1}(\Pi_{\Delta}-|\mathcal{A}|^{-1})$ and fixed $x\in T^{-1}(\Pi_{\Delta}-|\mathcal{A}|^{-1})$ we have 
\begin{align}
&\frac{f_{\lambda}(x')-f_{\lambda}(x)-\big\langle T^{-1}\big[{\rm proj}_{\mathcal{L}_0}\big(\nabla_{\pi}V_{\lambda,\pi}^{\pi}\big|_{\pi=T(x)+|\mathcal{A}|^{-1}}\big)\big], x'-x\big\rangle}{\|x'-x\|}\nonumber\\
\overset{(a)}{=}&\frac{1}{\big\|[T(x')+|\mathcal{A}|^{-1}]-[T(x)+|\mathcal{A}|^{-1}]\big\|}\Big[V_{\lambda,T(x')+|\mathcal{A}|^{-1}}^{T(x')+|\mathcal{A}|^{-1}}-V_{\lambda,T(x)+|\mathcal{A}|^{-1}}^{T(x)+|\mathcal{A}|^{-1}}\nonumber\\
&-\big\langle {\rm proj}_{\mathcal{L}_0}\big(\nabla_{\pi}V_{\lambda,\pi}^{\pi}\big|_{\pi=T(x)+|\mathcal{A}|^{-1}}\big), [T(x')+|\mathcal{A}|^{-1}]-[T(x)+|\mathcal{A}|^{-1}]\big\rangle\Big]\nonumber\\
\overset{(b)}{\to}& 0 {\rm~as~} x'\in T^{-1}(\Pi_{\Delta}-|\mathcal{A}|^{-1}) {\rm~and~}x'\to x, \label{eq:f_diffable}
\end{align}
where (a) uses the property of the orthogonal transformation $T$, and (b) uses Eq. (\ref{eq:pi_diffable}) and the fact that $x'\to x$ means $\big\|[T(x')+|\mathcal{A}|^{-1}]-[T(x)+|\mathcal{A}|^{-1}]\big\|=\|x'-x\|\to 0$. 

Furthermore, we will show that $f_{\lambda}(x)$ is a Lipscthiz continuous and Lipschitz smooth function of $x\in \Pi_{\Delta}$. For any $x, x'\in T^{-1}(\Pi_{\Delta}-|\mathcal{A}|^{-1})$, we have
\begin{align}
|f_{\lambda}(x')-f_{\lambda}(x)|=&\big|V_{\lambda,T(x')+|\mathcal{A}|^{-1}}^{T(x')+|\mathcal{A}|^{-1}}-V_{\lambda,T(x)+|\mathcal{A}|^{-1}}^{T(x)+|\mathcal{A}|^{-1}}\big|
\overset{(a)}{\le}\frac{L_{\lambda}}{\Delta}\|T(x')-T(x)\|\overset{(b)}{=}\frac{L_{\lambda}}{\Delta}\|x'-x\|,\nonumber
\end{align}
\begin{align}
\|\nabla f_{\lambda}(x')-\nabla f_{\lambda}(x)\|=&\big\|T^{-1}\big[{\rm proj}_{\mathcal{L}_0}\big(\nabla_{\pi}V_{\lambda,\pi}^{\pi}\big|_{\pi=T(x')}\big)\big]-T^{-1}\big[{\rm proj}_{\mathcal{L}_0}\big(\nabla_{\pi}V_{\lambda,\pi}^{\pi}\big|_{\pi=T(x)}\big)\big]\big\|\nonumber\\
\overset{(b)}{=}&\big\|{\rm proj}_{\mathcal{L}_0}\big(\nabla_{\pi}V_{\lambda,\pi}^{\pi}\big|_{\pi=T(x')+|\mathcal{A}|^{-1}}\big)-{\rm proj}_{\mathcal{L}_0}\big(\nabla_{\pi}V_{\lambda,\pi}^{\pi}\big|_{\pi=T(x)+|\mathcal{A}|^{-1}}\big)\big\|\nonumber\\
\le&\big\|\big(\nabla_{\pi}V_{\lambda,\pi}^{\pi}\big|_{\pi=T(x')+|\mathcal{A}|^{-1}}\big)-\big(\nabla_{\pi}V_{\lambda,\pi}^{\pi}\big|_{\pi=T(x)+|\mathcal{A}|^{-1}}\big)\big\|\nonumber\\
\overset{(a)}{\le}&\frac{\ell_{\lambda}}{\Delta}\|T(x')-T(x)\|\overset{(b)}{=}\frac{\ell_{\lambda}}{\Delta}\|x'-x\|,\nonumber
\end{align}
In both the inequalities above, (a) applies Theorem \ref{thm:V_Lip} to $T(x)+|\mathcal{A}|^{-1}, T(x')+|\mathcal{A}|^{-1}\in\Pi_{\Delta}$ and (b) uses the property of the orthogonal transformation $T$. The two inequalities above implies that $f_{\lambda}$ is an $\frac{L_{\lambda}}{\Delta}$-Lipschitz continuous and $\frac{\ell_{\lambda}}{\Delta}$-Lipschitz smooth function on $T^{-1}(\Pi_{\Delta}-|\mathcal{A}|^{-1})$. 

Denote 
\begin{align}
g_{\lambda,\delta}(\pi)\!=\!\frac{|\mathcal{S}|(|\mathcal{A}|\!-\!1)}{2N\delta}\sum_{i=1}^N \big({V}_{\lambda,\pi+\delta u_i}^{\pi+\delta u_i}\!-\!{V}_{\lambda,\pi-\delta u_i}^{\pi-\delta u_i}\big)u_i,\label{eq:gV}
\end{align}
which replaces $\hat{V}_{\lambda,\pi'}^{\pi'}$ with $V_{\lambda,\pi'}^{\pi'}$ in Eq. (\ref{eq:0ppg}). The estimation error of the performative policy gradient estimator above can be rewritten as follows for any $\pi\in\Pi_{\Delta}$. 
\begin{align}
&g_{\lambda,\delta}(\pi)-{\rm proj}_{\mathcal{L}_0}(\nabla_{\pi} V_{\lambda,\pi}^{\pi})\nonumber\\
\overset{(a)}{=}&\Big(\frac{|\mathcal{S}|(|\mathcal{A}|\!-\!1)}{2N\delta}\sum_{i=1}^N \big(V_{\lambda,\pi+\delta u_i}^{\pi+\delta u_i}\!-\!V_{\lambda,\pi-\delta u_i}^{\pi-\delta u_i}\big)u_i\Big)-{\rm proj}_{\mathcal{L}_0}(\nabla_{\pi} V_{\lambda,\pi}^{\pi})\nonumber\\
\overset{(b)}{=}&\Big(\frac{|\mathcal{S}|(|\mathcal{A}|\!-\!1)}{2N\delta}\sum_{i=1}^N \big(f_{\lambda}\big[T^{-1}(\pi-|\mathcal{A}|^{-1})+\delta T^{-1}(u_i)\big]\!-\!f_{\lambda}\big[T^{-1}(\pi-|\mathcal{A}|^{-1}])-\delta T^{-1}(u_i)\big]\big)\cdot\nonumber\\
&T^{-1}(u_i)\Big)-T^{-1}[{\rm proj}_{\mathcal{L}_0}(\nabla_{\pi} V_{\lambda,\pi}^{\pi})]\nonumber\\
\overset{(c)}{=}&\Big(\frac{|\mathcal{S}|(|\mathcal{A}|\!-\!1)}{2N\delta}\sum_{i=1}^N \big(f_{\lambda}\big[T^{-1}(\pi-|\mathcal{A}|^{-1})+\delta T^{-1}(u_i)\big]\!-\!f_{\lambda}\big[T^{-1}(\pi-|\mathcal{A}|^{-1}])-\delta T^{-1}(u_i)\big]\big)\cdot\nonumber\\
&T^{-1}(u_i)\Big)-\nabla f_{\lambda}[T^{-1}(\pi-|\mathcal{A}|^{-1})],\label{eq:gerr_mid}
\end{align} 
where (a) uses Eq. (\ref{eq:0ppg}), (b) uses $f_{\lambda}(x)\overset{\rm def}{=}V_{\lambda,T(x)+|\mathcal{A}|^{-1}}^{T(x)+|\mathcal{A}|^{-1}}$ and the property of the orthogonal transformation $T^{-1}$, (c) uses $\nabla f_{\lambda}(x)=T^{-1}\big({\rm proj}_{\mathcal{L}_0}\nabla_{\pi}V_{\lambda,\pi}^{\pi}\big|_{\pi=T(x)+|\mathcal{A}|^{-1}}\big)$. Note that in the above Eq. (\ref{eq:gerr_mid}), $\pi\in\Pi_{\Delta}$ and $u_i$ is uniformly distributed on the sphere $U_1\cap\mathcal{L}_0$ with $U_1\overset{\rm def}{=}\{u\in\mathbb{R}^{|\mathcal{S}||\mathcal{A}|}\!:\|u\|\!=\! 1\}$. 

Hence, $\pi\pm\delta u_i\in\Pi_{\Delta-\delta}$ which implies $T^{-1}(\pi-|\mathcal{A}|^{-1})\pm\delta T^{-1}(u_i)=T^{-1}(\pi\pm\delta u_i-|\mathcal{A}|^{-1})\in T^{-1}(\Pi_{\Delta-\delta}-|\mathcal{A}|^{-1})$. Also, $T^{-1}(u_i)$ is uniformly distributed on the sphere $T^{-1}(U_{1,0})=\mathbb{S}_{|\mathcal{S}|(|\mathcal{A}|-1)}=\{u\in\mathbb{R}^{|\mathcal{S}|(|\mathcal{A}|-1)}:\|u\|=1\}$. Therefore, we can apply Lemma \ref{lemma:gerr} to the above Eq. (\ref{eq:gerr_mid}) where the function $f_{\lambda}$ is an $\frac{L_{\lambda}}{\Delta-\delta}$-Lipschitz continuous and $\frac{\ell_{\lambda}}{\Delta-\delta}$-Lipschitz smooth function on $T^{-1}(\Pi_{\Delta-\delta}-|\mathcal{A}|^{-1})$, and obtain the following bound which holds with probability at least $1-\eta$. 
\begin{align}
&\|g_{\lambda,\delta}(\pi)-{\rm proj}_{\mathcal{L}_0}(\nabla_{\pi} V_{\lambda,\pi}^{\pi})\|\nonumber\\
\le&\frac{4L_{\lambda}|\mathcal{S}|(|\mathcal{A}|\!-\!1)}{3N(\Delta-\delta)}\log\Big(\frac{|\mathcal{S}|(|\mathcal{A}|\!-\!1)+1}{\eta}\Big)\!+\!\frac{L_{\lambda}|\mathcal{S}|(|\mathcal{A}|\!-\!1)}{\Delta-\delta}\sqrt{\frac{2}{N}\log\Big(\frac{|\mathcal{S}|(|\mathcal{A}|\!-\!1)\!+\!1}{\eta}\Big)}\!+\!\frac{\delta\ell_{\lambda}}{\Delta-\delta}\nonumber\\
\le&\frac{4L_{\lambda}|\mathcal{S}||\mathcal{A}|}{3N(\Delta-\delta)}\log\Big(\frac{|\mathcal{S}||\mathcal{A}|}{\eta}\Big)+\frac{L_{\lambda}|\mathcal{S}||\mathcal{A}|}{\Delta-\delta}\sqrt{\frac{2}{N}\log\Big(\frac{|\mathcal{S}||\mathcal{A}|}{\eta}\Big)}+\frac{\delta\ell_{\lambda}}{\Delta-\delta}.\label{eq:gerr_exactV}
\end{align}
Note that $|\hat{V}_{\lambda,\pi}^{\pi}-V_{\lambda,\pi}^{\pi}|\le\epsilon_V$ holds for any a certain policy $\pi$ with probability at least $1-\eta$. Therefore, with probability at least $1-2N\eta$, we have
\begin{align}
|\hat{V}_{\lambda,\pi'}^{\pi'}-V_{\lambda,\pi'}^{\pi'}|\le\epsilon_V, \forall \pi'\in\{\pi\pm\delta u_i\}_{i=1}^N\label{eq:eps_Vi}
\end{align}

Therefore, with probability at least $1-(2N+1)\eta$, Eqs. (\ref{eq:gerr_exactV}) and (\ref{eq:eps_Vi}) hold and thus we have
\begin{align}
&\|\hat{g}_{\lambda,\delta}(\pi)-{\rm proj}_{\mathcal{L}_0}(\nabla_{\pi} V_{\lambda,\pi}^{\pi})\|\nonumber\\
\le&\|\hat{g}_{\lambda,\delta}(\pi)-g_{\lambda,\delta}(\pi)\|+\|g_{\lambda,\delta}(\pi)-{\rm proj}_{\mathcal{L}_0}(\nabla_{\pi} V_{\lambda,\pi}^{\pi})\|\nonumber\\
\overset{(a)}{\le}& \Big\|\frac{|\mathcal{S}|(|\mathcal{A}|\!-\!1)}{2N\delta}\sum_{i=1}^N \big(\hat{V}_{\lambda,\pi+\delta u_i}^{\pi+\delta u_i}-V_{\lambda,\pi+\delta u_i}^{\pi+\delta u_i}\!-\!\hat{V}_{\lambda,\pi-\delta u_i}^{\pi-\delta u_i}+V_{\lambda,\pi-\delta u_i}^{\pi-\delta u_i}\big)u_i\Big\|\nonumber\\
&+\frac{4L_{\lambda}|\mathcal{S}||\mathcal{A}|}{3N(\Delta-\delta)}\log\Big(\frac{|\mathcal{S}||\mathcal{A}|}{\eta}\Big)+\frac{L_{\lambda}|\mathcal{S}||\mathcal{A}|}{\Delta-\delta}\sqrt{\frac{2}{N}\log\Big(\frac{|\mathcal{S}||\mathcal{A}|}{\eta}\Big)}+\frac{\delta\ell_{\lambda}}{\Delta-\delta}\nonumber\\
\overset{(b)}{\le}& \frac{|\mathcal{S}||\mathcal{A}|}{N\delta}\sum_{i=1}^N \big\|\big(\hat{V}_{\lambda,\pi+\delta u_i}^{\pi+\delta u_i}-V_{\lambda,\pi+\delta u_i}^{\pi+\delta u_i}\!-\!\hat{V}_{\lambda,\pi-\delta u_i}^{\pi-\delta u_i}+V_{\lambda,\pi-\delta u_i}^{\pi-\delta u_i}\big)u_i\big\|\nonumber\\
&+\frac{4L_{\lambda}|\mathcal{S}||\mathcal{A}|}{3N(\Delta-\delta)}\log\Big(\frac{|\mathcal{S}||\mathcal{A}|}{\eta}\Big)+\frac{L_{\lambda}|\mathcal{S}||\mathcal{A}|}{\Delta-\delta}\sqrt{\frac{2}{N}\log\Big(\frac{|\mathcal{S}||\mathcal{A}|}{\eta}\Big)}+\frac{\delta\ell_{\lambda}}{\Delta-\delta}\nonumber\\
\le& \frac{|\mathcal{S}||\mathcal{A}|}{N\delta}\sum_{i=1}^N \big(|\hat{V}_{\lambda,\pi+\delta u_i}^{\pi+\delta u_i}-V_{\lambda,\pi+\delta u_i}^{\pi+\delta u_i}|\!+\!|\hat{V}_{\lambda,\pi-\delta u_i}^{\pi-\delta u_i}+V_{\lambda,\pi-\delta u_i}^{\pi-\delta u_i}|\big)\nonumber\\
&+\frac{4L_{\lambda}|\mathcal{S}||\mathcal{A}|}{3N(\Delta-\delta)}\log\Big(\frac{|\mathcal{S}||\mathcal{A}|}{\eta}\Big)+\frac{L_{\lambda}|\mathcal{S}||\mathcal{A}|}{\Delta-\delta}\sqrt{\frac{2}{N}\log\Big(\frac{|\mathcal{S}||\mathcal{A}|}{\eta}\Big)}+\frac{\delta\ell_{\lambda}}{\Delta-\delta}\nonumber\\
\overset{(c)}{\le}&\frac{2|\mathcal{S}||\mathcal{A}|\epsilon_V}{\delta}+\frac{4L_{\lambda}|\mathcal{S}||\mathcal{A}|}{3N(\Delta-\delta)}\log\Big(\frac{|\mathcal{S}||\mathcal{A}|}{\eta}\Big)+\frac{L_{\lambda}|\mathcal{S}||\mathcal{A}|}{\Delta-\delta}\sqrt{\frac{2}{N}\log\Big(\frac{|\mathcal{S}||\mathcal{A}|}{\eta}\Big)}+\frac{\delta\ell_{\lambda}}{\Delta-\delta},\nonumber
\end{align}
where (a) uses Eqs. (\ref{eq:0ppg}), (\ref{eq:gdelta}) and (\ref{eq:gerr_exactV}), (b) uses Jensen's inequality that $\|\frac{1}{N}\sum_{i=1}^N x_i\|^2\le \frac{1}{N}\sum_{i=1}^N \|x_i\|^2$ for any vectors $\{x_i\}_{i=1}^N$ of the same dimensionality, (c) uses $|\hat{V}_{\lambda,\pi}^{\pi'}-V_{\lambda,\pi'}^{\pi'}|\le\epsilon_V$ for any policy $\pi'$. By replacing $\eta$ with $\frac{\eta}{3N}$ in the inequality above, we prove the error bound (\ref{eq:grad_err}) as follows which holds with probability at least $1-\eta$.
\begin{align}
&\|\hat{g}_{\lambda,\delta}(\pi)-{\rm proj}_{\mathcal{L}_0}(\nabla_{\pi} V_{\lambda,\pi}^{\pi})\|\nonumber\\
\le&\frac{2|\mathcal{S}||\mathcal{A}|\epsilon_V}{\delta}\!+\!\frac{4L_{\lambda}|\mathcal{S}||\mathcal{A}|}{3N(\Delta-\delta)}\log\Big(\frac{3N|\mathcal{S}||\mathcal{A}|}{\eta}\Big)\!+\!\frac{L_{\lambda}|\mathcal{S}||\mathcal{A}|}{\Delta-\delta}\sqrt{\frac{2}{N}\log\Big(\frac{3N|\mathcal{S}||\mathcal{A}|}{\eta}\Big)}\!+\!\frac{\delta\ell_{\lambda}}{\Delta-\delta} \label{eq:conclude_gerr}\\
=&\mathcal{O}\big(\frac{\epsilon_V}{\delta}+\frac{\log(N/\eta)}{\sqrt{N}}+\delta\big)\nonumber
\end{align}

\section{Proof of Proposition \ref{prop:2grad_inprods}}
For any $\pi\in\Pi_{\Delta}$, it is easily seen that the corresponding $\pi'$ defined by Eq. (\ref{eq:pi_pie}) also belongs to $\Pi_{\Delta}$. Therefore,
\begin{align}
\langle\nabla_{\pi} V_{\lambda,\pi}^{\pi},\pi'-\pi\rangle \le \max_{\Tilde{\pi}\in\Pi_{\Delta}}\langle\nabla_{\pi} V_{\lambda,\pi}^{\pi},\Tilde{\pi}-\pi\rangle\le \frac{D\lambda}{5|\mathcal{A}|(1-\gamma)}.\nonumber
\end{align}
Substituting the above inequality into Eq. (\ref{eq:pi_ge2}), we obtain that 
\begin{align}
\pi(a|s)\ge&\pi_{\min}\exp\Big[-\frac{2|\mathcal{A}|}{D\lambda}(1-\gamma)\langle\nabla_{\pi} V_{\lambda,\pi}^{\pi},\pi'-\pi\rangle\Big]\ge \frac{2\pi_{\min}}{3}\ge 2\Delta.\nonumber
\end{align}
Therefore, for any $\pi_2\in\Pi$, we can prove that $\frac{\pi_2+\pi}{2}\in\Pi_{\Delta}$ as follows. 
\begin{align}
\frac{\pi_2(a|s)+\pi(a|s)}{2}\ge \frac{0+2\Delta}{2}=\Delta.\nonumber
\end{align}
Therefore, we can prove 
Eq. (\ref{eq:2grad_inprods}) as follows.
\begin{align}
\max_{\pi_2\in\Pi}\langle\nabla_{\pi} V_{\lambda,\pi}^{\pi},\pi_2-\pi\rangle
=&2\max_{\pi_2\in\Pi}\Big\langle\nabla_{\pi} V_{\lambda,\pi}^{\pi},\frac{\pi_2+\pi}{2}-\pi\Big\rangle \overset{(a)}{\le}2\max_{\Tilde{\pi}\in\Pi_{\Delta}}\langle\nabla_{\pi} V_{\lambda,\pi}^{\pi},\Tilde{\pi}-\pi\rangle.\nonumber
\end{align}
where (a) uses $\frac{\pi_2+\pi}{2}\in\Pi_{\Delta}$. 

\section{Proof of Theorem \ref{thm:0ppg_rate}}\label{sec:proof_0ppg_rate}
If $\pi_t\in\Pi_{\Delta}$, then $\pi_{t+1}\in\Pi_{\Delta}$, since $\Pi_{\Delta}$ is a convex set and $\pi_{t+1}$ obtained by Eq. (\ref{eq:pi_update}) is a convex combination of $\pi_t, \Tilde{\pi}_t\in\Pi_{\Delta}$. Since $\pi_0\in\Pi_{\Delta}$, we have $\pi_t\in\Pi_{\Delta}$ for all $t$ by induction. Therefore, Proposition \ref{prop:grad_err} implies that the following bound holds simultaneously for all $\{\pi_t\}_{t=1}^T\subseteq\Pi_{\Delta}$ with probability at least $1-\eta$. 
\begin{align}
&\|\hat{g}_{\lambda,\delta}(\pi_t)-{\rm proj}_{\mathcal{L}_0}(\nabla_{\pi} V_{\lambda,\pi_t}^{\pi_t})\|\nonumber\\
\le& \frac{2|\mathcal{S}||\mathcal{A}|\epsilon_V}{\delta}+\frac{4L_{\lambda}|\mathcal{S}||\mathcal{A}|}{3TN(\Delta-\delta)}\log\Big(\frac{3TN|\mathcal{S}||\mathcal{A}|}{\eta}\Big)+\frac{L_{\lambda}|\mathcal{S}||\mathcal{A}|}{\Delta-\delta}\sqrt{\frac{2}{N}\log\Big(\frac{3TN|\mathcal{S}||\mathcal{A}|}{\eta}\Big)}+\frac{\delta\ell_{\lambda}}{\Delta-\delta}. \label{eq:grad_err_pit}
\end{align}
The bound above further implies that for any $\pi\in\Pi$, we have
\begin{align}
&\big|\big\langle \hat{g}_{\lambda,\delta}(\pi_t)-\nabla_{\pi} V_{\lambda,\pi_t}^{\pi_t}, \pi-\pi_t\big\rangle\big| \nonumber\\
\overset{(a)}{=}&\big|\big\langle \hat{g}_{\lambda,\delta}(\pi_t)-{\rm proj}_{\mathcal{L}_0}(\nabla_{\pi} V_{\lambda,\pi_t}^{\pi_t}), \pi-\pi_t\big\rangle\big|\nonumber\\
\le&\|\hat{g}_{\lambda,\delta}(\pi_t)-{\rm proj}_{\mathcal{L}_0}(\nabla_{\pi} V_{\lambda,\pi_t}^{\pi_t})\|\cdot\|\pi-\pi_t\|\nonumber\\
\overset{(b)}{\le}&\sqrt{2|\mathcal{S}|}\Big[\frac{2|\mathcal{S}||\mathcal{A}|\epsilon_V}{\delta}+\frac{4L_{\lambda}|\mathcal{S}||\mathcal{A}|}{3TN(\Delta-\delta)}\log\Big(\frac{3TN|\mathcal{S}||\mathcal{A}|}{\eta}\Big)\nonumber\\
&+\frac{L_{\lambda}|\mathcal{S}||\mathcal{A}|}{\Delta-\delta}\sqrt{\frac{2}{N}\log\Big(\frac{3TN|\mathcal{S}||\mathcal{A}|}{\eta}\Big)}+\frac{\delta\ell_{\lambda}}{\Delta-\delta}\Big], \label{eq:dire_err_pit}
\end{align}
where (a) uses $\Tilde{\pi}_t-\pi_t, \Tilde{\pi}-\pi_t\in\mathcal{L}_0$ for $\Tilde{\pi}_t, \Tilde{\pi}\in\Pi_{\Delta}$, and (b) uses Eq. (\ref{eq:grad_err_pit}) and Lemma \ref{lemma:pi_diameter}. 

Under the conditions above, we have
\begin{align}
&V_{\lambda,\pi_{t+1}}^{\pi_{t+1}}\nonumber\\
\overset{(a)}{\ge}&V_{\lambda,\pi_t}^{\pi_t}+\langle\nabla_{\pi}V_{\lambda,\pi_t}^{\pi_t}, \pi_{t+1}-\pi_t\rangle-\frac{\ell_{\lambda}}{2\Delta}\|\pi_{t+1}-\pi_t\|^2\nonumber\\
\overset{(b)}{=}&V_{\lambda,\pi_t}^{\pi_t}+\beta\langle\nabla_{\pi}V_{\lambda,\pi_t}^{\pi_t}, \Tilde{\pi}_t-\pi_t\rangle-\frac{\ell_{\lambda}\beta^2}{2\Delta}\|\Tilde{\pi}_t-\pi_t\|^2\nonumber\\
=&V_{\lambda,\pi_t}^{\pi_t}+\beta\langle\hat{g}_{\lambda,\delta}(\pi_t), \Tilde{\pi}_t-\pi_t\rangle+\beta\langle\nabla_{\pi}V_{\lambda,\pi_t}^{\pi_t}-\hat{g}_{\lambda,\delta}(\pi_t), \Tilde{\pi}_t-\pi_t\rangle-\frac{\ell_{\lambda}\beta^2}{2\Delta}\|\Tilde{\pi}_t-\pi_t\|^2\nonumber\\
\overset{(c)}{\ge}&V_{\lambda,\pi_t}^{\pi_t}+\beta\langle\hat{g}_{\lambda,\delta}(\pi_t), \Tilde{\pi}_t-\pi_t\rangle -\frac{\ell_{\lambda}|\mathcal{S}|\beta^2}{\Delta}-\beta\sqrt{2|\mathcal{S}|}\Big[\frac{2|\mathcal{S}||\mathcal{A}|\epsilon_V}{\delta}\nonumber\\
&+\frac{4L_{\lambda}|\mathcal{S}||\mathcal{A}|}{3TN(\Delta-\delta)}\log\Big(\frac{3TN|\mathcal{S}||\mathcal{A}|}{\eta}\Big)+\frac{L_{\lambda}|\mathcal{S}||\mathcal{A}|}{\Delta-\delta}\sqrt{\frac{2}{N}\log\Big(\frac{3TN|\mathcal{S}||\mathcal{A}|}{\eta}\Big)}+\frac{\delta\ell_{\lambda}}{\Delta-\delta}\Big],\label{eq:Verr_iter}
\end{align}  
where (a) uses the $\frac{\ell_{\lambda}}{\Delta}$-Lipschitz smoothness of $V_{\lambda,\pi}^{\pi}$ on $\Pi_{\Delta}$, (b) uses Eq. (\ref{eq:pi_update}), (c) uses Eq. (\ref{eq:dire_err_pit}) and Lemma \ref{lemma:pi_diameter}.

Rearranging and averaging Eq. (\ref{eq:Verr_iter}) over $t=0,1,\ldots,T-1$, we obtain that
\begin{align}
&\max_{\Tilde{\pi}\in\Pi_{\Delta}}\langle\hat{g}_{\lambda,\delta}(\pi_{\widetilde{T}}), \Tilde{\pi}-\pi_{\widetilde{T}}\rangle\nonumber\\
\overset{(a)}{=}&\langle\hat{g}_{\lambda,\delta}(\pi_{\widetilde{T}}), \Tilde{\pi}_{\widetilde{T}}-\pi_{\widetilde{T}}\rangle\nonumber\\
\overset{(b)}{\le}&\frac{1}{T}\sum_{t=0}^{T-1} \langle\hat{g}_{\lambda,\delta}(\pi_t), \Tilde{\pi}_t-\pi_t\rangle\nonumber\\
\le&\frac{V_{\lambda,\pi_T}^{\pi_T}-V_{\lambda,\pi_0}^{\pi_0}}{T\beta}+\frac{\ell_{\lambda}|\mathcal{S}|\beta}{\Delta} +\sqrt{2|\mathcal{S}|}\Big[\frac{2|\mathcal{S}||\mathcal{A}|\epsilon_V}{\delta}\nonumber\\
&+\frac{4L_{\lambda}|\mathcal{S}||\mathcal{A}|}{3TN(\Delta-\delta)}\log\Big(\frac{3TN|\mathcal{S}||\mathcal{A}|}{\eta}\Big)+\frac{L_{\lambda}|\mathcal{S}||\mathcal{A}|}{\Delta-\delta}\sqrt{\frac{2}{N}\log\Big(\frac{3TN|\mathcal{S}||\mathcal{A}|}{\eta}\Big)}+\frac{\delta\ell_{\lambda}}{\Delta-\delta}\Big]\nonumber\\
\le&\frac{1+\lambda\log|\mathcal{A}|}{T\beta(1-\gamma)}+\frac{\ell_{\lambda}|\mathcal{S}|\beta}{\Delta} +\sqrt{2|\mathcal{S}|}\Big[\frac{2|\mathcal{S}||\mathcal{A}|\epsilon_V}{\delta}\nonumber\\
&+\frac{4L_{\lambda}|\mathcal{S}||\mathcal{A}|}{3TN(\Delta-\delta)}\log\Big(\frac{3TN|\mathcal{S}||\mathcal{A}|}{\eta}\Big)+\frac{L_{\lambda}|\mathcal{S}||\mathcal{A}|}{\Delta-\delta}\sqrt{\frac{2}{N}\log\Big(\frac{3TN|\mathcal{S}||\mathcal{A}|}{\eta}\Big)}+\frac{\delta\ell_{\lambda}}{\Delta-\delta}\Big],\label{eq:PiDelta_rate}
\end{align}
where (a) uses Lemma \ref{lemma:wolfe} which means $\Tilde{\pi}_t$ satisfies Eq. (\ref{eq:pi_wolfe}) and (b) uses the output rule of Algorithm \ref{alg:0ppg} that $\widetilde{T}\in\mathop{\arg\min}_{0\le t\le T-1}\langle\hat{g}_{\lambda,\delta}(\pi_t), \Tilde{\pi}_t-\pi_t\rangle$. Therefore, 
\begin{align}
&\max_{\Tilde{\pi}\in\Pi_{\Delta}}\big\langle\nabla_{\pi}V_{\lambda,\pi_{\widetilde{T}}}^{\pi_{\widetilde{T}}}, \Tilde{\pi}-\pi_{\widetilde{T}}\big\rangle \nonumber\\
=&\max_{\Tilde{\pi}\in\Pi_{\Delta}}\big[\big\langle\nabla_{\pi}V_{\lambda,\pi_{\widetilde{T}}}^{\pi_{\widetilde{T}}}-\hat{g}_{\lambda,\delta}(\pi_{\pi_{\widetilde{T}}}), \Tilde{\pi}-\pi_{\widetilde{T}}\big\rangle + \big\langle\hat{g}_{\lambda,\delta}(\pi_{\pi_{\widetilde{T}}}), \Tilde{\pi}-\pi_{\widetilde{T}}\big\rangle\big]\nonumber\\
\overset{(a)}{\le}&\frac{1+\lambda\log|\mathcal{A}|}{T\beta(1-\gamma)}+\frac{\ell_{\lambda}|\mathcal{S}|\beta}{\Delta} +2\sqrt{2|\mathcal{S}|}\Big[\frac{2|\mathcal{S}||\mathcal{A}|\epsilon_V}{\delta}\nonumber\\
&+\frac{4L_{\lambda}|\mathcal{S}||\mathcal{A}|}{3TN(\Delta-\delta)}\log\Big(\frac{3TN|\mathcal{S}||\mathcal{A}|}{\eta}\Big)+\frac{L_{\lambda}|\mathcal{S}||\mathcal{A}|}{\Delta-\delta}\sqrt{\frac{2}{N}\log\Big(\frac{3TN|\mathcal{S}||\mathcal{A}|}{\eta}\Big)}+\frac{\delta\ell_{\lambda}}{\Delta-\delta}\Big],\label{eq:rate_mid}
\end{align}
where (a) uses Eqs. (\ref{eq:dire_err_pit}) and (\ref{eq:PiDelta_rate}). 

Use the following hyperparameter choices for Algorithm \ref{alg:0ppg}.
\begin{align}
\Delta=&\frac{\pi_{\min}}{3},\label{eq:Delta}\\
\beta=&\frac{D\Delta\epsilon}{12\ell_{\lambda}|\mathcal{S}|}=\frac{D\pi_{\min}\epsilon}{36\ell_{\lambda}|\mathcal{S}|}=\mathcal{O}(\epsilon),\label{eq:beta}\\
T=&\frac{12(1+\lambda\log|\mathcal{A}|)}{D\epsilon\beta(1-\gamma)}=\frac{432\ell_{\lambda}|\mathcal{S}|(1+\lambda\log|\mathcal{A}|)}{\pi_{\min}D^2(1-\gamma)\epsilon^2}=\mathcal{O}(\epsilon^{-2})\label{eq:T}\\
\delta=&\frac{D\Delta\epsilon}{48\sqrt{2|\mathcal{S}|}\ell_{\lambda}}=\frac{D\pi_{\min}\epsilon}{144\sqrt{2|\mathcal{S}|}\ell_{\lambda}}=\mathcal{O}(\epsilon)\overset{(a)}{\le}\frac{\Delta}{2},\label{eq:delta}\\
\epsilon_V=&\frac{D\delta\epsilon}{48|\mathcal{S}||\mathcal{A}|\sqrt{2|\mathcal{S}|}}=\frac{\pi_{\min}D^2\epsilon^2}{13824\ell_{\lambda}|\mathcal{S}|^2|\mathcal{A}|}=\mathcal{O}(\epsilon^2)\label{eq:epsV}\\
N=&\frac{663552L_{\lambda}^2|\mathcal{S}|^3|\mathcal{A}|^2}{D^2\pi_{\min}^2\epsilon^2}\log\max\Big(\frac{165888L_{\lambda}^2|\mathcal{S}|^3|\mathcal{A}|^2}{D^2\pi_{\min}^2\epsilon^2}, \frac{1296\ell_{\lambda}|\mathcal{S}|^2|\mathcal{A}|(1+\lambda\log|\mathcal{A}|)}{D^2\eta\pi_{\min}(1-\gamma)\epsilon^2}\Big)\nonumber\\
&+2\log\Big(\frac{3|\mathcal{S}||\mathcal{A}|}{\eta}\Big)+3\nonumber\\
=&\mathcal{O}[\epsilon^{-2}\log(\eta^{-1}\epsilon^{-1})]\label{eq:N}
\end{align}
where (a) uses $\epsilon\le 24\sqrt{2|\mathcal{S}|}\ell_{\lambda}/D$. With the hyperparameter choices above, we obtain the following inequalities (\ref{eq:term1})-(\ref{eq:term3}).
\begin{align}
&2\sqrt{2|\mathcal{S}|}\cdot\frac{L_{\lambda}|\mathcal{S}||\mathcal{A}|}{\Delta-\delta}\sqrt{\frac{2}{N}\log\Big(\frac{3TN|\mathcal{S}||\mathcal{A}|}{\eta}\Big)}\nonumber\\
\overset{(a)}{\le}&\frac{24L_{\lambda}|\mathcal{S}|^{1.5}|\mathcal{A}|}{\pi_{\min}}\sqrt{\frac{\log N}{N}+\frac{1}{N}\log\Big(\frac{1296\ell_{\lambda}|\mathcal{S}|^2|\mathcal{A}|(1+\lambda\log|\mathcal{A}|)}{\eta\pi_{\min}D^2(1-\gamma)\epsilon^2}\Big)}\nonumber\\
\overset{(b)}{\le}&\frac{24L_{\lambda}|\mathcal{S}|^{1.5}|\mathcal{A}|}{\pi_{\min}}\sqrt{\Tilde{\epsilon}+\frac{\Tilde{\epsilon}}{4}}\nonumber\\
=&\frac{12\sqrt{5}L_{\lambda}|\mathcal{S}|^{1.5}|\mathcal{A}|}{\pi_{\min}}\cdot\frac{D\pi_{\min}\epsilon}{\sqrt{165888}L_{\lambda}|\mathcal{S}|^{1.5}|\mathcal{A}|}\le\frac{D\epsilon}{12},\label{eq:term1}
\end{align}
where (a) uses Eq. (\ref{eq:T}) and $\delta\le\Delta/2=\pi_{\min}/6$ implied by Eqs. (\ref{eq:Delta}) and (\ref{eq:delta}), (b) uses Eq. (\ref{eq:N}) and its implication that $N\ge 4\Tilde{\epsilon}^{-1}\log(\Tilde{\epsilon}^{-1})$ with $\Tilde{\epsilon}=\frac{\pi_{\min}^2\epsilon^2}{165888D^2L_{\lambda}^2|\mathcal{S}|^3|\mathcal{A}|^2}\le 0.5$ (since $\epsilon\le \frac{288DL_{\lambda}|\mathcal{S}|^{1.5}|\mathcal{A}|}{\pi_{\min}}$), which implies $\frac{\log N}{N}\le \Tilde{\epsilon}$ based on Lemma \ref{lemma:logx_by_x}. 
\begin{align}
\frac{1}{TN}\log\Big(\frac{3TN|\mathcal{S}||\mathcal{A}|}{\eta}\Big)=&\frac{\log(TN)}{TN}+\frac{1}{TN}\log\Big(\frac{3|\mathcal{S}||\mathcal{A}|}{\eta}\Big)\overset{(a)}{\le}\frac{1}{2}+\frac{1}{2}=1,\label{eq:term2}
\end{align}
where (a) uses $NT\ge N\ge \max\Big[3,2\log\Big(\frac{3|\mathcal{S}||\mathcal{A}|}{\eta}\Big)\Big]$ and Lemma \ref{lemma:logx_by_x}. 
\begin{align}
2\sqrt{2|\mathcal{S}|}\cdot\frac{4L_{\lambda}|\mathcal{S}||\mathcal{A}|}{3TN(\Delta-\delta)}\log\Big(\frac{3TN|\mathcal{S}||\mathcal{A}|}{\eta}\Big)\overset{(a)}{\le}& 2\sqrt{2|\mathcal{S}|}\cdot\frac{\sqrt{2}L_{\lambda}|\mathcal{S}||\mathcal{A}|}{\Delta-\delta}\sqrt{\frac{1}{TN}\log\Big(\frac{3TN|\mathcal{S}||\mathcal{A}|}{\eta}\Big)} \nonumber\\
\overset{(b)}{\le}& \frac{D\epsilon}{12}\label{eq:term3}
\end{align}
where (a) uses $\frac{4}{3}<\sqrt{2}$ and $y\le \sqrt{y}$ for $y=\frac{1}{TN}\log\Big(\frac{3TN|\mathcal{S}||\mathcal{A}|}{\eta}\Big)\le 1$ (Eq. (\ref{eq:term2})), and (b) uses $T\ge 1$ and Eq. (\ref{eq:term1}). By substituting the hyperparameter choices (\ref{eq:Delta})-(\ref{eq:N}) as well as Eqs. (\ref{eq:term1}) and (\ref{eq:term3}) into Eq. (\ref{eq:rate_mid}), we have 
\begin{align}
&\max_{\Tilde{\pi}\in\Pi_{\Delta}}\big\langle\nabla_{\pi}V_{\lambda,\pi_{\widetilde{T}}}^{\pi_{\widetilde{T}}}, \Tilde{\pi}-\pi_{\widetilde{T}}\big\rangle \nonumber\\
\le&\frac{1+\lambda\log|\mathcal{A}|}{T\beta(1-\gamma)}+\frac{\ell_{\lambda}|\mathcal{S}|\beta}{\Delta} +2\sqrt{2|\mathcal{S}|}\Big[\frac{2|\mathcal{S}||\mathcal{A}|\epsilon_V}{\delta}\nonumber\\
&+\frac{4L_{\lambda}|\mathcal{S}||\mathcal{A}|}{3TN(\Delta-\delta)}\log\Big(\frac{3TN|\mathcal{S}||\mathcal{A}|}{\eta}\Big)+\frac{L_{\lambda}|\mathcal{S}||\mathcal{A}|}{\Delta-\delta}\sqrt{\frac{2}{N}\log\Big(\frac{3TN|\mathcal{S}||\mathcal{A}|}{\eta}\Big)}+\frac{\delta\ell_{\lambda}}{\Delta-\delta}\Big]\nonumber\\
\le&\frac{1+\lambda\log|\mathcal{A}|}{\beta(1-\gamma)}\frac{\epsilon\beta(1-\gamma)}{12D(1+\lambda\log|\mathcal{A}|)} + \frac{\ell_{\lambda}|\mathcal{S}|}{\Delta}\cdot \frac{\Delta\epsilon}{12D\ell_{\lambda}|\mathcal{S}|}\nonumber\\
&+\frac{4\sqrt{2|\mathcal{S}|}|\mathcal{S}||\mathcal{A}|}{\delta}\cdot \frac{\delta\epsilon}{48D|\mathcal{S}||\mathcal{A}|\sqrt{2|\mathcal{S}|}} +\frac{\epsilon}{12D}+\frac{\epsilon}{12D}+\frac{2\sqrt{2|\mathcal{S}|}\ell_{\lambda}}{\Delta/2}\cdot\frac{\Delta\epsilon}{48\sqrt{2|\mathcal{S}|}D\ell_{\lambda}}\nonumber\\
=&\frac{D\epsilon}{2}\overset{(a)}{\le} \frac{D\lambda}{5|\mathcal{A}|(1-\gamma)},\nonumber
\end{align}
where (a) uses $\epsilon\le \frac{2\lambda D^2}{5|\mathcal{A}|(1-\gamma)}$. Then based on Proposition \ref{prop:2grad_inprods}, the inequality above implies that 
\begin{align}
\max_{\Tilde{\pi}\in\Pi}\big\langle\nabla_{\pi}V_{\lambda,\pi_{\widetilde{T}}}^{\pi_{\widetilde{T}}}, \Tilde{\pi}-\pi_{\widetilde{T}}\big\rangle\le D\epsilon,\nonumber
\end{align}
which means $\pi_{\widetilde{T}}$ is a $D\epsilon$-stationary policy. Then if $\mu\ge 0$, Corollary \ref{coro:stat2PO} implies that $\pi_{\widetilde{T}}$ is also an $\epsilon$-PO policy.

\section{Adjusting Our Results to the Existing Quadratic Regularizer}\label{sec:QuadReg}
In Section \ref{sec:alg_all}, we have proposed a 0-FW algorithm and obtain its finite-time convergence result to the desired PO policy for our entropy-regularized value function (\ref{eq:Vfunc}). We will briefly show that 0-FW algorithm can also converge to PO for the existing performative reinforcement learning defined by the value function (\ref{eq:Vfunc_general}) with quadratic regularizer $\mathcal{H}_{\pi'}(\pi)=\frac{1}{2}\|d_{\pi,p_{\pi'}}\|^2$ \citep{mandal2023performative,rank2024performative,pollatos2025corruption}. The \textit{performative value function} can be rewritten as the following $\lambda$-strongly concave function of $d_{\pi,p_{\pi}}$. 
\begin{align}
V_{\lambda,\pi}^{\pi}=\langle d_{\pi,p_{\pi}}, r_{\pi}\rangle-\lambda \|d_{\pi,p_{\pi}}\|^2. \label{eq:Vfunc_quad}
\end{align}
We can prove the \textit{performative value function} above also satisfies Theorem \ref{thm:ToOpt} (gradient dominance) with a different $\mu$, following the same proof logic, since both regularizers $\mathcal{H}_{\pi}(\pi)$ are strongly convex functions of $d_{\pi,p_{\pi}}$ which implies that $V_{\lambda,\pi_{\alpha}}^{\pi_{\alpha}}$ is a $\mu$-strongly concave function of $\alpha$ as shown in the proof of Theorem \ref{thm:ToOpt} in Appendix \ref{sec:proof_thm:ToOpt}. By direct calculation, we can also show that $V_{\lambda,\pi}^{\pi}$ above is a Lipschitz continuous and Lipschitz smooth function of $\pi\in\Pi$. With these two properties, we can follow the proof logic of Theorem \ref{thm:0ppg_rate} to show that the 0-FW algorithm (with the same procedure as that of Algorithm \ref{alg:0ppg} except the different values of $V_{\lambda,\pi_{\alpha}}^{\pi_{\alpha}}$ in the policy evaluation step) converges to a stationary policy of the \textit{performative value function} (\ref{eq:Vfunc_quad}), which by gradient dominance is a PO policy when the new value of $\mu$ satisfies $\mu\ge 0$. 


\end{document}